\documentclass{article} 
\usepackage{iclr2027_conference,times}

\usepackage{amsmath,amsfonts,bm}

\def\eqref#1{equation~\ref{#1}}

\def\1{\bm{1}}

\DeclareMathAlphabet{\mathsfit}{\encodingdefault}{\sfdefault}{m}{sl}
\SetMathAlphabet{\mathsfit}{bold}{\encodingdefault}{\sfdefault}{bx}{n}

\usepackage{hyperref}
\usepackage{url}
\usepackage{graphicx}
\usepackage{booktabs}
\usepackage[export]{adjustbox}
\usepackage{multirow}
\usepackage{makecell}
\usepackage[table]{xcolor}   
\definecolor{ourrow}{HTML}{FFF3C4}   

\usepackage{amsmath,amssymb,amsthm}

\newtheorem{theorem}{Theorem}
\newtheorem{proposition}{Proposition}

\newtheorem{remark}{Remark}

\usepackage{tabularx}
\usepackage{amssymb}
\usepackage{subcaption}
\usepackage[T1]{fontenc}

\usepackage[most]{tcolorbox}
\usepackage{amsthm}
\usepackage{arydshln}

\definecolor{promptgray}{RGB}{248,248,248}

\newtcblisting{promptbox}[1][]{
    listing only,
    breakable,
    enhanced,
    colback=promptgray,
    colframe=black!15,
    boxrule=0.4pt,
    arc=2pt,
    left=6pt,
    right=6pt,
    top=6pt,
    bottom=6pt,
    listing options={
        basicstyle=\ttfamily\small,
        breaklines=true,
        columns=fullflexible,
        keepspaces=true,
        showstringspaces=false,
        tabsize=2
    },
    #1
}

\title{Can MLLMs Reason About Visual Persuasion? \\ Evaluating the Efficacy and Faithfulness of Reasoning}

\iclrfinalcopy

\author{Naeun Lee \quad Hyunjong Kim \quad Sunghwan Choi \quad Injin Kong \quad Yohan Jo$^\dagger$ \\
Seoul National University\\
\texttt{naeun.lee@snu.ac.kr} \\
}

\begin{document}
\raggedbottom

\maketitle

\begin{abstract}

Persuasive visuals play a central role in advertising, public communication, and online media, making it increasingly important to understand whether an image is persuasive and why. However, current Multimodal Large Language Models (MLLMs) have limited ability to reason about visual persuasion. Our analysis reveals that models often rely on a shortcut---selectively citing easily recognizable visual elements and treating their mere presence as evidence of persuasiveness---rather than reasoning over the diverse visual cues in an image to determine how and why they contribute to persuasiveness. To address this limitation, we propose (1) a fine-tuning approach that trains models on rationales reflecting diverse perspectives on visual persuasiveness, and (2) an evaluation framework that measures the faithfulness of models' rationales through three complementary metrics. Fine-tuning on multi-perspective rationales improves persuasiveness prediction and reasoning effectiveness. However, our evaluation framework reveals a discrepancy between prediction performance and rationale faithfulness, showing that higher prediction performance does not necessarily correspond to more faithful reasoning. These findings highlight the need to evaluate effectiveness and faithfulness separately and inform future directions for improving both the training and evaluation of visual persuasion reasoning.

\end{abstract}

\def\thefootnote{\fnsymbol{footnote}}
\footnotetext[2]{Corresponding author.}
\def\thefootnote{\arabic{footnote}}

\section{Introduction}

Visual persuasion, the use of images to influence cognition, beliefs, and behavioral intentions, plays a central role in modern communication \citep{chandler2011dictionary}.
Across public health campaigns, political messaging, and digital advertising, persuasive images serve as a medium for delivering messages that change the viewer's attitudes and behaviors.
Advances in generative AI further increase the importance of this problem, as text-to-image models enable automated generation of tailored persuasive content at lower cost~\citep{pvp, cap}.
As such content is generated and deployed at increasing scale, evaluating whether and why an image is persuasive becomes a problem of growing practical and societal significance.

Figure~\ref{fig:pvp_examples} (left) shows a target message, two images, and their persuasiveness labels.
Despite the strong performance of MLLMs on many multimodal tasks, visual persuasiveness evaluation remains challenging for two main reasons.
First, unlike factual image understanding, it requires interpretive reasoning over diverse factors, such as affective, symbolic, and compositional cues~\citep{blair2011possibility, kjeldsen2015study, mcquarrie1999visual, phillips2004beyond}, which can interact in complex ways to shape persuasiveness.
Second, this reasoning is inherently non-unique: the persuasiveness of an image can vary with viewers' personality traits and values~\citep{pvp}, allowing for multiple valid rationales for the same image and message and making it difficult to establish a consistent evaluation criterion.

\begin{figure}[t]
    \centering
    \includegraphics[width=1.0\linewidth]{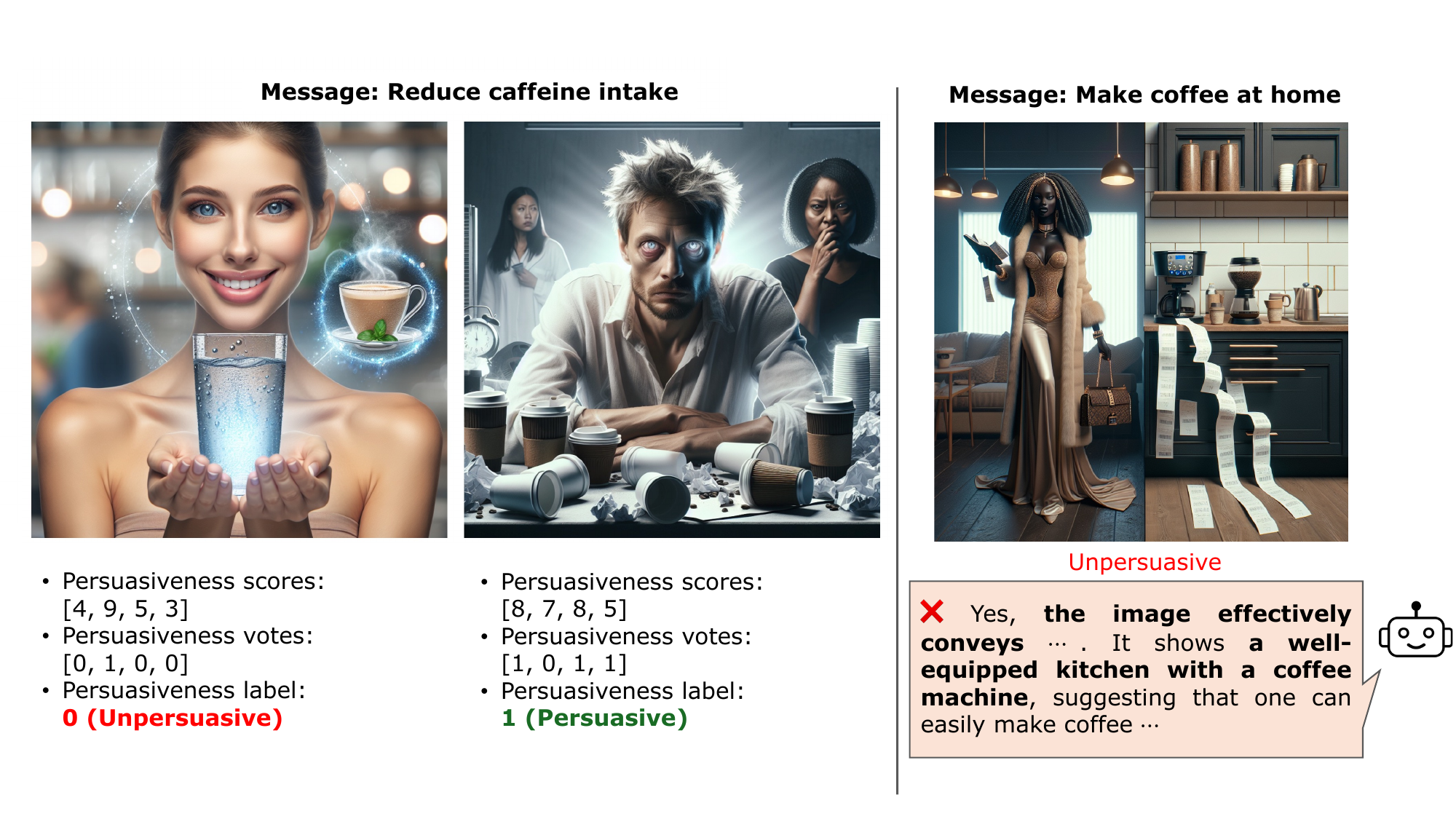}
    \caption{\textbf{The PVP dataset and a failure mode of MLLMs.}
    \textbf{(Left)} Example image--message pairs from the PVP dataset. \emph{Persuasiveness scores} are the 0--10 ratings assigned by individual annotators, while \emph{persuasiveness votes} are binary votes converted based on each annotator's score distribution. The \emph{persuasiveness label} is determined by aggregating these binary votes.
    \textbf{(Right)} A \emph{message-relevance shortcut}, where the model treats the presence of message-related visual elements as evidence of persuasiveness.
    }
    \label{fig:pvp_examples}
\end{figure}

Can current MLLMs reason about visual persuasion?
Our analysis suggests they cannot (Sections \ref{section:training} and \ref{sec:rationale-evaluation}), revealing two main problems. 
First, instructing models to generate rationales before predicting the persuasiveness of an image in relation to a target message does not reliably improve, and can even degrade prediction accuracy.
Moreover, their reasoning often exhibits a message-relevance shortcut: identifying visual elements related to the target message and treating the mere presence of such elements as sufficient evidence of persuasiveness.
For example, as shown in Figure~\ref{fig:pvp_examples} (right), a model treats the presence of a coffee machine as evidence that an image is persuasive for the message ``Make coffee at home'', despite the image being labeled unpersuasive by humans.
These findings highlight two challenges: (1) improving the reasoning capabilities of MLLMs for visual persuasion, and (2) evaluating the faithfulness of generated rationales to both the input images and the model's final decisions.


To address the first gap, we construct a rationale dataset from large vision-language teacher models. We improve rationale diversity by introducing dual-axis prompts that vary along \emph{evidence polarity} (\emph{including} vs.\ \emph{excluding} opposing evidence for persuasiveness) and \emph{visual granularity} (\emph{global} vs.\ \emph{local}). This approach also mitigates the teacher-model's biases in reasoning patterns.
Fine-tuning on the resulting multi-perspective rationales improves persuasiveness prediction over training-free baselines, achieves classification performance comparable to label-only SFT, and produces more visually grounded and prediction-consistent rationales.

For the second gap, we introduce a novel evaluation framework for reasoning faithfulness in visual persuasion, assessing rationales along three dimensions: \emph{rationale-to-image groundedness}, whether the claims in the rationale are supported by visual evidence in the image; \emph{rationale-to-decision consistency}, whether the model's final decision is consistent with the rationale's overall assessment; and \emph{rationale-to-decision sensitivity}, whether the model's decision is sensitive to changes in the visual evidence identified by the rationale.
This framework reveals that fine-tuning on the multi-perspective rationales improves rationale-to-image groundedness and rationale-to-decision consistency as the training rationales specify diverse visual components that support or counter image persuasiveness, while showing a trade-off in rationale-to-decision sensitivity. 
Overall, our work lays a foundation for training approaches that jointly improve prediction performance and reasoning faithfulness in visual persuasiveness prediction.


\section{Related Work}

\subsection{Visual Persuasion and Visual Rhetoric}
Computational work on visual persuasion has progressed from inferring communicative intent from images of political figures~\citep{joo2014visual}, to mining image persuasiveness in multimodal social media posts~\citep{liu2022imagearg}, and collecting large-scale persuasiveness ratings~\citep{pvp}.
These efforts treat persuasiveness as a prediction target without examining reasoning behind those judgments.
Theoretical work in visual rhetoric establishes, however, that persuasive images operate through affect, symbolism, and argument structure~\citep{blair2011possibility,kjeldsen2015study,mcquarrie1999visual,phillips2004beyond}, with meaning constructed across multiple levels of visual perception~\citep{navon1977forest,oliva2006building}.
Correct prediction alone cannot reveal whether a model's judgment is grounded in a coherent visual argument, which our work aims to measure.

\subsection{Rationale Supervision}
A growing body of work has shown that fine-tuning language models on chain-of-thought rationales generated by larger teacher models, rather than output labels alone, improves task performance and generalization~\citep{magister2023teaching,ho2023large,shridhar2023distilling}.
This paradigm has been extended to MLLMs, where supervising intermediate reasoning over visual inputs further improves performance on complex multimodal tasks~\citep{zhang2023multimodal,shao2024visual}.
These methods, however, largely treat label correctness as sufficient supervision, leaving the quality of the reasoning process itself unmeasured.
In visual persuasion, where persuasive judgments admit multiple valid interpretations, correct label prediction alone does not guarantee that the model's rationale faithfully reflects the basis of its judgment.

\subsection{Rationale Faithfulness Evaluation}
In NLP systems, faithfulness refers to how accurately an explanation reflects a model's reasoning process~\citep{jacovi2020towards}.
Prior work has studied faithfulness along dimensions, including the sufficiency and comprehensiveness of extractive rationales~\citep{deyoung2020eraser}, prediction sensitivity to reasoning-step perturbations~\citep{lanham2023measuring}, and the causal influence of intermediate reasoning steps on final outcomes~\citep{paul2024making}.
These approaches, however, assume settings where explicit evidence, structured reasoning steps, or well-defined decision spaces provide anchors for evaluating reasoning faithfulness.
Extending these criteria to visual persuasion is nontrivial, as rationales may rely on subjective, symbolic, and affective interpretations of visual cues, need not follow a fixed sequence of reasoning steps, and admit multiple plausible reasoning paths for the same image and message.
To the best of our knowledge, we introduce the first rationale evaluation protocol for this setting, assessing reasoning quality along rationale-to-decision consistency, rationale-to-image groundedness, and rationale-to-decision sensitivity.

\section{Dual-Axis Rationale Training for Visual Persuasiveness Prediction}
\label{section:training}

\subsection{Motivation}
\label{sec:motivation}
As a preliminary analysis, we examined whether MLLMs' reasoning improves their persuasiveness prediction. Table~\ref{tab:classification_results} compares instructing MLLMs to predict directly (Base) with instructing them to generate a rationale before predicting (Base-Reasoning). Base MLLMs achieve only modest performance under direct prediction, and rationale-first prediction does not consistently improve over this baseline. To understand this limitation, we analyzed rationales generated by the base models and found a message-relevance shortcut: MLLMs often treat message-related visual elements as sufficient evidence of persuasiveness. Figure~\ref{fig:pvp_examples} (right) illustrates this: for the message ``Make coffee at home'', the model cites the presence of a coffee machine to justify its prediction of `persuasive', even though the ground-truth label is `unpersuasive'.

This limitation is thus not merely a visual recognition failure, but an \emph{evidence-use failure}: MLLMs rely on message-relevant cues while underweighting evidence against persuasiveness, such as distracting objects that draw attention from the intended message, or cues too weak to make the target action compelling. Their rationales therefore justify a positive prediction with selectively supportive cues, rather than weighing supporting and opposing evidence for an image-level judgment. We further quantify this shortcut and find that Reasoning-SFT reduces shortcut-based reasoning on three of the four backbones. Detailed results are provided in Appendix~\ref{app:message_shortcut}.

This also motivates rationale supervision beyond the model's own trajectories. Since online trajectories are sampled from the model's current reasoning policy, they expose the model to a limited set of rationale patterns. If the base model already equates message relevance with persuasiveness, optimizing over such trajectories may reinforce this shortcut rather than teach reasoning patterns outside its current capability.

\subsection{Dual-Axis Rationale Design: Evidence Polarity and Visual Granularity}
\label{sec:prompt-design}

To diversify reasoning perspectives used for supervision, we design dual-axis prompts that vary the type and level of visual evidence considered in the rationale. The first axis, \emph{evidence polarity}, controls whether the rationale considers only supporting evidence for persuasiveness or both supporting and opposing evidence. The second axis, \emph{visual granularity}, controls whether the rationale focuses on the overall scene or on specific visual elements. Figure~\ref{fig:prompt_design} summarizes the resulting prompt design.

\begin{figure}
    \centering
    \includegraphics[width=1.0\linewidth]{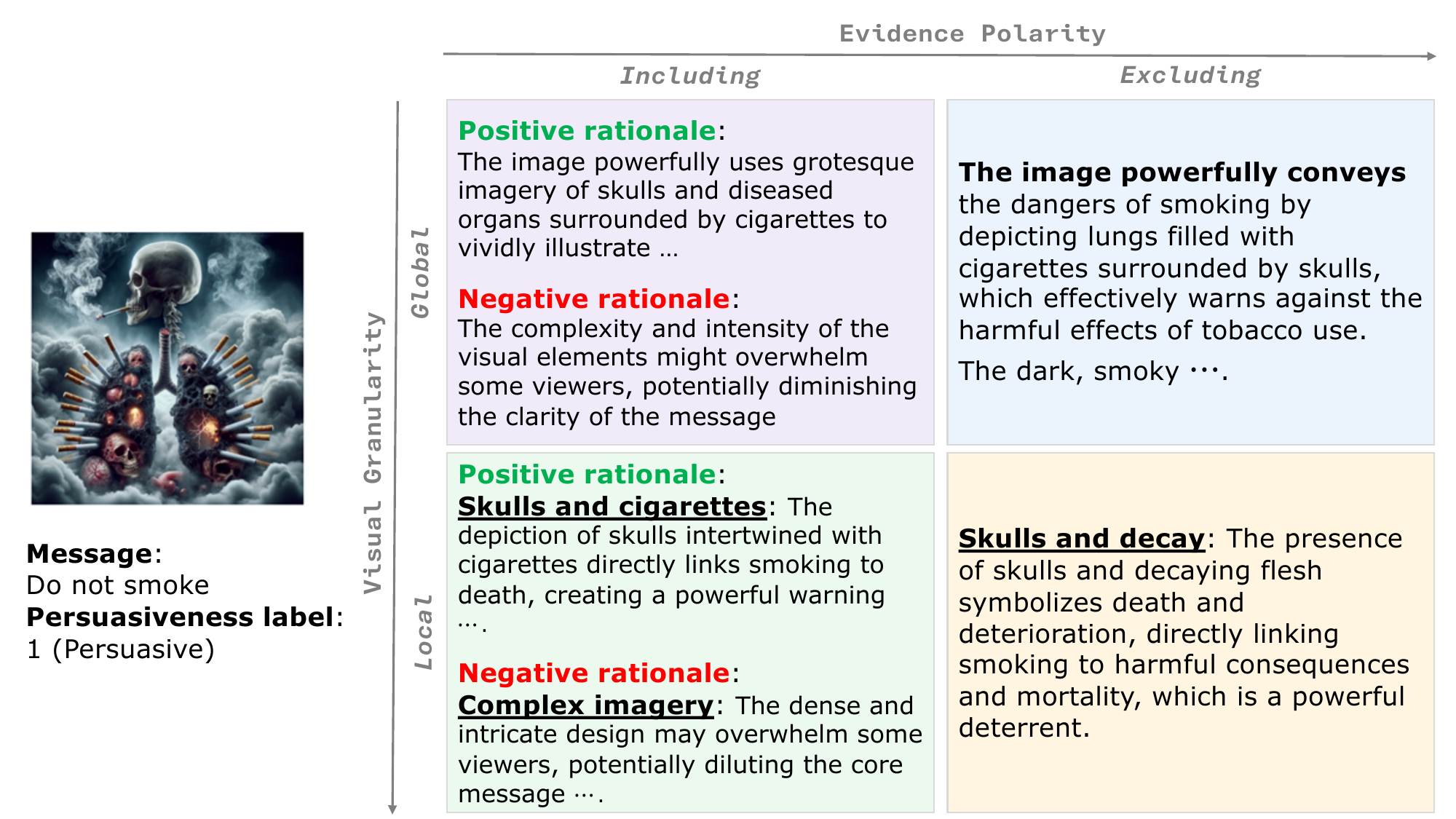}
    \caption{\textbf{Dual-axis rationale design.} Four rationale types defined by two axes: evidence polarity (including vs. excluding opposing evidence) and visual granularity (global vs. local). Given an image, its message ("Do not smoke"), and its persuasiveness label, each prompt type elicits a rationale with different characteristics.}
    \label{fig:prompt_design}
\end{figure}

Evidence polarity specifies whether a rationale uses only evidence supporting the persuasiveness decision or also considers evidence that weakens it. This axis is motivated by cognitive psychology work showing that robust judgment requires considering not only belief-consistent evidence, but also conflicting cues and alternative interpretations~\citep{baron2005rationality, stanovich1997reasoning,nickerson1998confirmation, stanovich2013myside}.
Hence, we define two prompt types for guiding the model to \emph{include} or \emph{exclude} opposing evidence for persuasiveness.

Visual granularity specifies whether a rationale evaluates the image as a whole or focuses on specific visual evidence. This axis is motivated by work in visual cognition showing that perception operates across global and local levels, from scene structure to individual elements~\citep{navon1977forest, oliva2006building}. We define two prompt types: \emph{global} and \emph{local}. Global prompts guide the model to evaluate the scene, composition, and communicative context, whereas local prompts guide it to examine specific objects, actions, text, and visual details relevant to the persuasiveness decision.
\subsection{Rationale Dataset Construction}
\label{sec:dataset-reconstruction}

To construct rationales for training MLLMs, we use the PVP dataset~\citep{pvp}, which consists of both AI-generated and web-sourced images.
We focus on the AI-generated subset for licensing and reproducibility reasons, and many web-sourced images are misaligned with the target messages.
The AI-generated subset of PVP contains 14{,}161 images paired with 596 messages across 20 topics, with each image--message pair rated for persuasiveness on a 0--10 scale by four annotators. We formulate persuasiveness evaluation as a binary classification task of persuasive vs. unpersuasive~\citep{Tan_2016, liu2022imagearg, liu2023imagearg2023}. Although persuasiveness is viewer-dependent, as an initial study of visual persuasiveness reasoning, we focus on images where annotators exhibit sufficient agreement. This provides more stable labels that are less sensitive to idiosyncratic viewer preferences and allows us to investigate models' reasoning about visual components without introducing more complex viewer factors.

To construct the binary labels, we first account for annotator-specific rating scales by converting each annotator's raw score into a binary vote using the first and third quartiles of their score distribution (Figure~\ref{fig:pvp_examples} left). Based on these binary votes, we retain image--message pairs as {persuasive} or {unpersuasive} only when at least 75\% of votes agree, discarding the remaining pairs. The resulting dataset has a mean vote agreement of
0.875.
Finally, for each of the five random seeds, we split the retained pairs by message
while maintaining label balance. This ensures that no message overlaps between
the training and test sets, so that evaluation measures generalization to unseen
messages. Averaged over the 5 seeds, the training split contains 823 examples and the
held-out test split contains 206 examples. Further details on dataset reconstruction are provided in Appendix~\ref{app:dataset-reconstruction}.

For each image--message--label triple in the training split, we apply all seven prompt templates from Appendix~\ref{app:prompt-reasoning-data}, covering different prompting configurations along the two axes of evidence polarity and visual granularity. These templates are designed to elicit rationales from different perspectives by varying whether opposing evidence is considered and whether the rationale focuses on global or local visual evidence.
We construct four teacher-generated rationale sets: \emph{Qwen}, generated by Qwen2.5-VL-72B-Instruct~\citep{qwen2.5-VL}; \emph{Phi}, generated by Phi-4-reasoning-vision-15B~\citep{phi4vr14b2026}; \emph{Gemma}, generated by Gemma-3-27B-IT~\citep{team2025gemma}; and \emph{LLaVA}, generated by LLaVA-1.6-34B~\citep{liu2024llavanext}.
After filtering out rationales that do not follow the required output format, each rationale set contains roughly 5{,}700 instances. 
We further examined the quality of label-conditioned teacher rationales through two complementary analyses. Label-conditioned rationales showed high semantic similarity to independently generated rationales (SBERT: 91.6\%). Detailed results are provided in Appendix~\ref{app:teacher_rationale_quality}.

\subsection{Experiments on Dual-Axis Rationale Supervision}
\label{sec:empirical-validation}

\subsubsection{Experimental Setup}
\label{sec:experimental-setup}
We fine-tune student models—Qwen2.5-VL-7B-Instruct, Phi-3.5-Vision-Instruct, Gemma-3-4B-IT, and LLaVA-1.6-7B—on the \emph{Qwen}, \emph{Phi}, \emph{Gemma}, and \emph{LLaVA} rationale sets, respectively. Each student is paired with a teacher from the same model family to mitigate feature-space misalignment~\citep{hao2023oneforall, boizard2024towards}. We compare our \emph{Reasoning-SFT} against five baselines grouped by training objective. \emph{Training-Free} baselines include Base, which directly predicts the label, and Base-Reasoning, which generates a rationale before prediction. \emph{Answer-Only Objective} methods include SFT, trained with supervised next-token prediction on labels, and GRPO-Ans~\citep{shao2024deepseekmath}, which uses a final-answer reward while applying the loss only to answer tokens. \emph{Answer + Rationale Objective} methods include GRPO, which uses the same final-answer reward but applies the loss to the full rationale--answer sequence, and our Reasoning-SFT, which uses supervised next-token prediction on teacher-generated rationales and labels. The prompt templates used for each setup and implementation details are provided in Appendix~\ref{app:prompts} and Appendix~\ref{app:implementation}, respectively.

\subsubsection{Results}

\begin{table*}[t]
    \centering
    \small
    \setlength{\tabcolsep}{4pt}
    \caption{
    Persuasiveness classification performance (\%) of student models on the PVP test split.
    All numbers are averaged over five random seeds.
    \emph{Our method} (Reasoning-SFT) fine-tunes the student using supervision on
    teacher-generated rationales from diverse reasoning perspectives.
    BA: balanced accuracy; P: precision; R: recall.
    High recall with low precision indicates over-prediction of the `persuasive' class.
    \textbf{Bold} indicates the best BA and F1 within each column.
    Detailed results with standard deviations across the five random seeds are provided in Appendix~\ref{app:classification_seeds}.
    }
    \label{tab:classification_results}

    \begin{adjustbox}{max width=\textwidth}
    \begin{tabular}{@{}l
        cccc
        cccc
        cccc
        cccc@{}}
    \\
    \toprule
     & \multicolumn{4}{c}{\textbf{Qwen2.5-VL-7B}}
     & \multicolumn{4}{c}{\textbf{Phi-3.5-Vision}}
     & \multicolumn{4}{c}{\textbf{Gemma-3-4B}}
     & \multicolumn{4}{c}{\textbf{LLaVA-1.6-7B}} \\
    \cmidrule(lr){2-5}
    \cmidrule(lr){6-9}
    \cmidrule(lr){10-13}
    \cmidrule(lr){14-17}

    \textbf{Setup}
    & \textbf{BA} & \textbf{F1} & \textbf{P} & \textbf{R}
    & \textbf{BA} & \textbf{F1} & \textbf{P} & \textbf{R}
    & \textbf{BA} & \textbf{F1} & \textbf{P} & \textbf{R}
    & \textbf{BA} & \textbf{F1} & \textbf{P} & \textbf{R} \\
    \midrule

    \multicolumn{17}{@{}l}{\emph{Training-Free:}} \\
    \quad Base
    & 63.2 & 68.7 & 59.6 & 81.3
    & 65.2 & 69.9 & 61.4 & 81.1
    & 61.7 & 69.9 & 57.4 & 89.6
    & 50.6 & 66.7 & 50.1 & 99.8 \\

    \quad Base-Reasoning
    & 61.9 & 68.9 & 57.9 & 85.0
    & 56.9 & 64.6 & 54.7 & 77.1
    & 55.2 & 63.9 & 53.4 & 79.6
    & 52.0 & 66.1 & 50.9 & 94.3 \\

    \midrule
    \multicolumn{17}{@{}l}{\emph{Answer-Only Objective:}} \\
    \quad SFT
    & 74.6 & 75.0 & 75.0 & 75.3
    & 73.7 & \textbf{75.6} & 70.5 & 81.7
    & 72.1 & 71.8 & 72.2 & 71.7
    & \textbf{72.9} & \textbf{74.6} & 70.3 & 79.9 \\

    \quad GRPO-Ans
    & 70.8 & 71.3 & 70.0 & 72.8
    & 67.5 & 66.6 & 74.3 & 60.5
    & 71.5 & 73.4 & 68.5 & 79.4
    & 53.2 & 67.2 & 51.7 & 96.5 \\

    \midrule
    \multicolumn{17}{@{}l}{\emph{Answer + Rationale Objective:}} \\
    \quad GRPO
    & 69.2 & 73.2 & 64.4 & 85.0
    & 55.3 & 67.6 & 52.7 & 94.2
    & 66.7 & 72.7 & 61.5 & 89.0
    & 56.1 & 68.5 & 53.2 & 96.3 \\

    \rowcolor{ourrow}
    \quad \emph{Reasoning-SFT}
    & \textbf{75.1} & \textbf{75.1} & 75.7 & 74.7
    & \textbf{74.8} & 74.2 & 76.3 & 72.2
    & \textbf{75.7} & \textbf{75.3} & 76.7 & 74.3
    & 72.5 & 72.5 & 72.7 & 72.6 \\

    \bottomrule
    \end{tabular}
    \end{adjustbox}
\end{table*}

Table~\ref{tab:classification_results} reports persuasiveness classification performance on the PVP test split over five random seeds.
Detailed statistical analyses of the performance differences are provided in Appendix~\ref{app:rationale_statistical_analysis}.
Reasoning-SFT achieves classification performance comparable to answer-only SFT across all four student models, while enabling explicit rationale generation. 
As an additional analysis, we examine whether prompting label-only SFT models to generate rationales can reproduce the benefits of Reasoning-SFT. Across five random seeds, prompting SFT models to generate rationales before prediction reduces balanced accuracy by 5.5--9.4 percentage points across all four backbones compared with direct prediction. Furthermore, Qwen and LLaVA fail to generate the requested rationales in 13.4\% and 86.4\% of cases, respectively, while Gemma exhibits relatively low rationale--prediction consistency (81.2 $\pm$ 9.74\%). On Qwen, the model predicts the persuasive class for 68.1\% of test examples despite approximately balanced labels, reducing balanced accuracy from 74.6\% to 69.1\%. These results indicate that label-only supervision does not readily transfer to effective rationale generation.

In contrast, Reasoning-SFT is explicitly trained to reason before prediction and can generate rationales while maintaining classification performance comparable to or better than answer-only training. One possible explanation is that dual-axis supervision encourages models to consider both supporting and opposing evidence at different levels of visual granularity, helping them distinguish mere message relevance from actual persuasiveness. This may mitigate overprediction of the persuasive class. To examine whether the improvements arise from combining diverse reasoning perspectives rather than increased training data, we compare single-axis and dual-axis supervision under matched training budgets (Appendix~\ref{app:matched_budget}). Using the original teacher-generated rationales, dual-axis supervision achieves the highest balanced accuracy across all four student models, outperforming the best single-axis setting by 0.4–1.7 percentage points.



Generating rationales without supervising them, by contrast, provides no reliable benefit. Prompting the base model to reason before answering (Base-Reasoning) even degrades balanced accuracy on three of the four student models, by up to 8.3 pp, suggesting that base MLLMs cannot effectively reason about persuasiveness. Both GRPO variants also fall below label-only SFT on all four models. GRPO---which applies reward over the full rationale-and-answer sequence---underperforms GRPO-Ans on three of the four models and shows signs of prediction collapse on most of them.
We hypothesize that this collapse stems from the base models' prior bias toward the `persuasive' class combined with outcome-based rewards: they achieve persuasive-class recall of 81.1--99.8\% but precision of only 50.1--61.4\%, indicating that they overpredict the `persuasive' class.
Because the reward depends only on final-answer correctness, it may reinforce rationale--answer trajectories that follow this bias whenever they yield a correct prediction, regardless of rationale quality. This can further encourage overprediction of `persuasive', resulting in lower balanced accuracy.
In Section~\ref{sec:rationale-evaluation}, we further investigate this behavior by evaluating the faithfulness of the generated rationales.

\section{Faithfulness Evaluation for Persuasiveness Rationales}
\label{sec:rationale-evaluation}
Reasoning-SFT with dual-axis rationales achieves classification performance comparable to label-only SFT while providing supervision for visual persuasion reasoning.
However, the generated rationales may cite relevant visual evidence without influencing the model's final decision. This concern is particularly relevant to visual persuasion, where models must not only identify message-relevant visual elements but also explain how they contribute to whether the image is persuasive.

Besides, assessing whether rationales faithfully provide such explanations is challenging because there is no single faithful rationale that can serve as the ground truth---the same image, target message, and label admit multiple valid rationales. Existing faithfulness evaluation methods are also not well suited to our setting, since visual persuasion often requires interpreting visual evidence rather than simply identifying it. We therefore introduce a new framework for evaluating the faithfulness of generated rationales in visual persuasiveness along three complementary dimensions: whether they are grounded in the image, whether they support the predicted label, and whether changing the cited visual evidence affects the model's prediction.

\begin{figure}[t]
    \centering
    \includegraphics[width=0.9\linewidth]{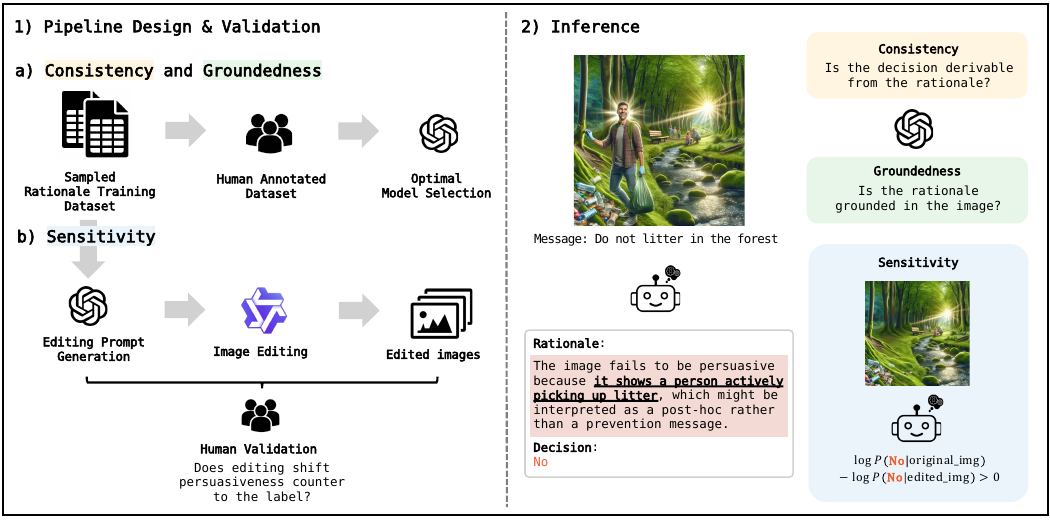}
    
    \caption{Overview of faithfulness evaluation pipelines.
    (Left) We validate judge-based pipelines for rationale-to-decision consistency and rationale-to-image groundedness against human annotations, and the editing pipeline for rationale-to-decision sensitivity.
    (Right) At inference, GPT-5 judges consistency and GPT-6-astra judges groundedness, while sensitivity is measured by the log-probability gap between original and edited images for the predicted label.}

    \label{fig:eval-pipeline}
\end{figure}

\subsection{Evaluation Metrics}
\label{sec:faithfulness-metrics}
Figure~\ref{fig:eval-pipeline} summarizes the evaluation pipelines for the three faithfulness metrics.
All metrics are computed as $\{0, 1\}$ scores, where $1$ indicates a faithful rationale along the corresponding dimension.

\subsubsection{Rationale-to-Image Groundedness}
Rationale-to-image groundedness measures whether the rationale is grounded in visual evidence from the image.
This metric identifies rationales that describe or infer visual evidence that is not supported by the image. Evaluating groundedness is challenging in visual persuasion because the relevant evidence is not limited to objects or regions. Rationales often refer to emotional tone, symbolism, composition, and image--message alignment, which are difficult to capture with existing object-centric grounding metrics~\citep{park2018multimodal, selvaraju2019taking, salewski2020clevr}.

Accordingly, we initially compared various candidate methods, including CLIP-based similarity, GPT-5-based atomic-fact verification, and direct GPT-5 prompting (Appendix~\ref{app:groundedness_candidates}). To assess their alignment with human judgments, we sampled 10\% of the training data and obtained majority-vote annotations from three annotators, with substantial inter-annotator agreement (Fleiss' $\kappa = 0.680$).

Although direct GPT-5 prompting was initially selected, its agreement with human judgments was limited (Cohen's $\kappa = 0.145$). Following error analysis, we refined the evaluation guidelines, introduced five few-shot examples excluded from validation, and adopted GPT-6-astra as the groundedness judge. The improved judge achieved 80.1\% balanced accuracy and Cohen's $\kappa = 0.418$ against the human majority vote. We use this improved judge to evaluate rationale groundedness in our experiments. The evaluation protocol and prompt are provided in Appendix~\ref{app:detail-rationale}. We assign a groundedness score of 1 when GPT-6-astra judges the rationale as grounded in the corresponding image, and 0 otherwise.

To further examine the remaining disagreements, we conducted an additional evaluation on a separate set of 300 image--rationale pairs, with no overlap with the original validation set. Among the disagreements, 78.6\% involved the judge rejecting rationales accepted by the human majority. Furthermore, 96.4\% of the rationales accepted by the judge were also accepted by humans. These findings suggest that the judge tends to apply a more conservative acceptance criterion, supporting its potential use as a screening tool for grounded rationales. Detailed analysis and cost comparisons are provided in Appendix~\ref{app:groundedness_candidates}.

\subsubsection{Rationale-to-Decision Consistency}
Rationale-to-decision consistency measures whether the predicted persuasiveness decision is derivable from the generated rationale.
This metric tests whether the rationale provides sufficient evidence for the model's predicted label.
To assess how well consistency judgments align with human judgments, we sampled 10\% of the training data and asked three annotators to judge whether each rationale is sufficient to derive the decision. We obtained the human labels by majority vote, with near-perfect inter-annotator agreement (Fleiss' $\kappa = 0.996$). We then evaluated GPT-5, Qwen3-32B, Qwen3-8B, and an NLI-based DeBERTa model as candidate judges against these human labels. GPT-5 achieved the highest performance against the human labels, attaining 99.6\% for both balanced accuracy and F1 (Appendix~\ref{app:consistency_candidates}). We therefore adopt GPT-5 as the final judge and assign a score of $1$ when it determines that the decision is derivable from the rationale, and $0$ otherwise.

\subsubsection{Rationale-to-Decision Sensitivity}
\label{sec:decision-sensitivity}
Rationale-to-decision sensitivity measures whether the model's decision is sensitive to visual evidence cited in the rationale. Whereas consistency and groundedness assess properties of the rationale itself, sensitivity tests whether intervening on rationale-cited evidence changes the model's decision behavior. This metric follows perturbation-based faithfulness evaluation, which tests whether removing or altering the evidence identified in an explanation changes the model's prediction~\citep{deyoung2020eraser, hooker2019benchmark, atanasova2023faithfulness}. Because visual persuasiveness rationales describe semantic evidence rather than pixel-level masks, we implement this idea through rationale-conditioned counterfactual image editing.

Given an original image $I$, target message $m$, rationale $r$, and predicted label $\hat{y}$, we proceed in three steps. First, GPT-5 converts $r$ into an editing instruction that weakens the rationale-cited evidence supporting $\hat{y}$, and Qwen-Image-Edit~\citep{wu2025qwenimagetechnicalreport} generates the edited image $I'$. Second, we feed $I'$ and $m$ back into the model using the same prompt. Third, we compute $\Delta_{\mathrm{sens}} = \log P(\hat{y} \mid I, m) - \log P(\hat{y} \mid I', m)$ and assign a binary score of $1$ if $\Delta_{\mathrm{sens}} > 0$ and $0$ otherwise, so that a score of $1$ indicates that confidence in the original prediction drops after the edit.

Since sensitivity depends on the quality of the image editing, three annotators validated the editing pipeline, showing substantial agreement (Fleiss' $\kappa = 0.693$) and judging $74.9\%$ of edited images as successfully shifting persuasiveness against the predicted label (Appendix~\ref{app:sensitivity_human_validation}). Appendix~\ref{app:image-edit-controls} reports control experiments with random and persuasiveness-preserving edits, verifying that sensitivity reflects the influence of rationale-cited evidence rather than a general response to editing.

\subsection{Rationale Faithfulness Results}
\label{sec:faithfulness-results}
We evaluate rationale faithfulness for three setups that generate rationales: Base-Reasoning (no training), GRPO-Ans (answer-only reward), and Reasoning-SFT (rationale supervision). We use GRPO-Ans as the RL baseline because it achieves higher balanced accuracy than GRPO on average (65.8\% vs.\ 61.8\%; Table~\ref{tab:classification_results}). Figure~\ref{fig:faithfulness_result} presents the quantitative faithfulness results across the three metrics. Qualitative analyses are also provided in Appendix~\ref{app:qualitative_faithfulness}.

\begin{figure}
    \centering
    \includegraphics[width=1.0\linewidth]{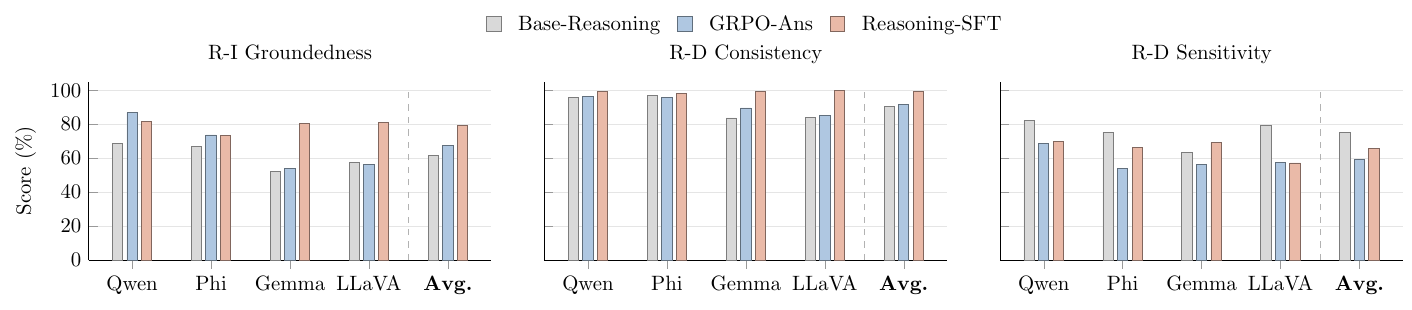}
    \caption{
            Rationale faithfulness (\%) of each training setup across four student
            models and their average, over five random seeds.
            Reasoning-SFT yields the highest R-I groundedness and R-D consistency,
            while Base-Reasoning retains the highest R-D sensitivity.
            Exact mean $\pm$ standard deviation values are reported in
            Table~\ref{tab:faithfulness_results}
            (Appendix~\ref{app:quantitative_faithfulness}).
    }
    \label{fig:faithfulness_result}
\end{figure}

\subsubsection{Rationale-to-Image Groundedness}

For rationale-to-image groundedness, Base-Reasoning achieves the lowest average score, while GRPO-Ans improves by 6.1 percentage points. Because GRPO-Ans receives reward only for the final answer, rationales that lead to correct predictions can be reinforced regardless of whether they are grounded in the image. Reasoning-SFT achieves the highest average groundedness (79.3\%), outperforming Base-Reasoning (61.5\%) and GRPO-Ans (67.6\%) by 17.8 and 11.7 percentage points, respectively. Although GRPO-Ans performs better on Qwen and comparably on Phi, Reasoning-SFT achieves higher groundedness on Gemma and LLaVA. Since teacher rationales specify various visual components from diverse reasoning perspectives, Reasoning-SFT may encourage the student to attend more closely to image-specific information. Such supervision may help the student incorporate these visual details into its own rationales, resulting in stronger grounding in the image.

To investigate the suboptimal groundedness scores across all three setups (60--80\%), we randomly sampled from each of the four Reasoning-SFT training splits the same number of examples as in the test set and evaluated their consistency and groundedness. The sampled rationales show near-perfect consistency (98.1--100), whereas groundedness is substantially lower (79.9--85.6). This gap suggests that image-grounded rationale generation is inherently more challenging than achieving rationale–decision consistency and that using more grounded teachers may be able to improve student models' groundedness.



\subsubsection{Rationale-to-Decision Consistency}

For rationale-to-decision consistency, Reasoning-SFT achieves the highest average score, outperforming Base-Reasoning by approximately 9 pp and GRPO-Ans by 7 pp. In the Reasoning-SFT training data, each teacher rationale is generated based on the ground-truth label, resulting in high consistency between the rationale and the label. 
Therefore, training the student model to generate such a rationale and then the final answer provides direct supervision for producing answers that are consistent with the model’s own reasoning. In contrast, Base-Reasoning receives no training signal that explicitly enforces consistency between the rationale and the final answer. GRPO-Ans likewise provides reward only for the final answer, so a sample can receive the same reward even when the rationale is inconsistent with the answer, as long as the answer itself is correct. As a result, GRPO-Ans lacks an explicit signal for penalizing rationale–answer inconsistency.







\subsubsection{Rationale-to-Decision Sensitivity}
For rationale-to-decision sensitivity, Reasoning-SFT generally retains sensitivity to the visual evidence cited in the rationale, although LLaVA shows a notable decrease. One possible explanation is that rationale supervision has a dual effect. 
When teachers are instructed to specify diverse visual evidence in their rationales, this naturally introduces less critical visual components along with more significant ones. 
As a result, training the student on these rationales may encourage the model to cite not only decision-sensitive evidence but also less decision-sensitive evidence. Consequently, such teacher rationales may introduce a trade-off between image groundedness and decision sensitivity, and improving both properties in rationales would be a fruitful direction.
Additional analyses show that sensitivity scores are not substantially affected by near-zero prediction changes, while the relative ranking of methods is preserved across two image-editing models on Qwen and Gemma, despite variations in absolute scores (Appendix~\ref{app:sensitivity_analysis}).

\section{Future Directions}

Our findings suggest two directions for future work. First, groundedness and sensitivity could be further improved through higher-quality rationale data and targeted training objectives. As discussed above, the training rationales themselves exhibit relatively low groundedness, and the groundedness achieved by Reasoning-SFT appears to be partly constrained by that of its training data. Constructing more strongly grounded rationales may therefore provide more effective supervision for improving rationale groundedness. Sensitivity is not directly supervised during training, and both GRPO-Ans and Reasoning-SFT show lower sensitivity than Base-Reasoning. This motivates intervention-based objectives that explicitly encourage sensitivity to rationale-cited visual evidence.

Second, persuasiveness classification performance and rationale faithfulness should be evaluated separately. The discrepancy between the classification results in Table~\ref{tab:classification_results} and the faithfulness results in Figure~\ref{fig:faithfulness_result} indicates that higher classification performance does not necessarily correspond to greater rationale faithfulness. Our framework also has limitations: automated groundedness judgments remain imperfect, consistency may exhibit ceiling effects, and sensitivity can be influenced by non-targeted image changes. Moreover, because multiple valid reasoning paths may support the same prediction, these metrics cannot uniquely identify the reasoning path underlying a decision or establish causal reliance on the stated rationale. Future work should develop more reliable evaluation methods that account for alternative valid rationales and further examine the generality of our findings across additional model backbones and diverse image domains.

\section{Conclusion}
MLLMs often cite easily recognizable visual elements and treat their mere presence as evidence of persuasiveness, which can lead to incorrect predictions and raises questions about whether their rationales faithfully reflect the visual evidence. In this paper, we investigated both the effectiveness and faithfulness of MLLM reasoning for visual persuasion. For effectiveness, we proposed dual-axis rationale supervision, varying prompts along evidence polarity and visual granularity to provide diverse reasoning signals for fine-tuning.
This approach achieves classification performance comparable to label-only SFT while producing more visually grounded and prediction-consistent rationales.
For faithfulness, we further introduced, to our knowledge, the first rationale faithfulness evaluation framework for visual persuasion. Applying this framework reveals a discrepancy between persuasiveness classification performance and rationale faithfulness, highlighting the need to evaluate them separately. Based on these findings, we outline future directions for improving both the training and evaluation of effective and faithful visual persuasion reasoning.


\subsection*{AI use statement}

In this work, we used generative AI tools to generate synthetic rationale data and to assist with translation. We have not used generative AI tools to develop theoretical models or conceptual frameworks, propose or refine hypotheses, design research methodology or experiments, implement methods, clean or reformat datasets, or interpret results. Formulating mathematical claims, assisting with proofs, and qualitative or thematic data analysis are not applicable to this work. Additionally, we used generative AI tools to edit the paper to improve readability. We have reviewed all AI-assisted work. The synthetic rationales were manually examined by the authors for quality before being used in our experiments, and all translated and edited text was reviewed by the authors to ensure that its meaning and technical content were preserved. We take responsibility for the final content of this work, including text, claims, and artifacts produced with the aid of generative AI.

\subsection*{Ethics statement}
This work involves three human evaluation studies for validating the
faithfulness metrics. The groundedness and consistency evaluations were conducted by three of the authors, while the sensitivity evaluation involved volunteers who assessed whether the counterfactual image edits were successfully performed.
These volunteers provided informed consent and were compensated at the applicable minimum wage. Participants in the sensitivity evaluation were informed of the persuasive nature of the stimuli, and no personally identifiable information was collected. This study is exempt from IRB review under 45 CFR 46.104(d)(2)(i). We use the
PVP dataset in accordance with its CC-BY-4.0 license and use all models and libraries under their respective licenses.

Visual persuasiveness evaluation may support beneficial applications but could also be misused to facilitate misleading or manipulative content. Our framework is intended for analyzing and auditing persuasive reasoning rather than optimizing persuasive content.

\subsection*{Reproducibility statement}
The experimental settings and implementation details required to reproduce the main results of this work are described in the paper. Specifically, the experimental setup is provided in Section~\ref{sec:dataset-reconstruction}, with additional implementation details in Appendix~\ref{app:implementation}. We also provide the code and dataset used in our experiments through anonymous links to facilitate reproduction and verification of the results. Together, these materials provide the information needed to reproduce the experimental procedures underlying the main results and conclusions of this work.



\bibliography{iclr2027_conference}
\bibliographystyle{iclr2027_conference}

\appendix
\section{Dataset Reconstruction Details}
\label{app:dataset-reconstruction}

We describe the full annotator-aware binary voting procedure used to
reconstruct the PVP dataset as a binary classification task.

\subsection{Annotator-aware Binary Voting.}
The per-annotator mean persuasiveness scores range from 0.02 to 10,
with 22.2\% of annotators having a mean score below 3 and 12.1\%
having a mean score above 7. This indicates substantial variation in
individual scoring tendencies, with some annotators assigning
systematically lower scores and others assigning systematically higher
scores. To account for these annotator-specific scoring scales, we
convert each annotator's raw score into a binary vote relative to that
annotator's own score distribution, rather than averaging raw scores
across annotators.

For each annotator $a$, we compute the first and third quartiles of
that annotator's 0--10 scores. Let $r_i^{(a)}$ be annotator $a$'s
score for image--message pair $i$, and let $Q_1^{(a)}$ and
$Q_3^{(a)}$ denote the first and third quartiles of annotator $a$'s
scores. We define the annotator-level vote as
\[
v_i^{(a)} =
\begin{cases}
1 & \text{if } r_i^{(a)} > Q_3^{(a)},\\
0 & \text{if } r_i^{(a)} < Q_1^{(a)}.
\end{cases}
\]
Scores falling between $Q_1^{(a)}$ and $Q_3^{(a)}$ are treated as
ambiguous and do not contribute a vote.

\subsection{Consensus-based Label Aggregation.}
Image--message pairs for which at least 75\% of the collected binary
votes are persuasive receive a persuasive label (1), while pairs for
which at least 75\% are unpersuasive receive an unpersuasive label
(0). Pairs that do not meet either threshold are discarded. This
procedure retains examples with stronger annotator consensus while
allowing for individual differences in the use of the original
0--10 rating scale.

\subsection{Inter-Annotator Agreement.}
We examine inter-annotator agreement at each stage of the
reconstruction procedure. Krippendorff's $\alpha$ is 0.151 on the
original 0--10 ratings, reflecting the substantial variation in
persuasiveness judgments also reported in the original PVP study.
After converting the ratings into annotator-specific binary votes,
Krippendorff's $\alpha$ increases to 0.292. On the final set of pairs
retained by the 75\% consensus criterion, Krippendorff's $\alpha$,
computed over the annotator-level binary votes, further increases to
0.500, with a mean vote agreement of 0.875. This progression indicates
that the reconstruction procedure retains examples with stronger
agreement while preserving the annotator-aware formulation of the
original ratings.

\subsection{Dataset Split.}
We split the retained pairs by message, ensuring that no message
appears in both splits, while maintaining persuasiveness label balance
between splits. This message-disjoint split enables evaluation of
out-of-message generalization. The final training split contains 823
image--message pairs covering 309 messages, and the test split
contains 206 image--message pairs covering 79 messages.

\section{Prompt Templates for Reasoning Data Collection}
\label{app:prompt-reasoning-data}

This appendix provides the prompt templates used to collect visual persuasion reasoning.
Each prompt is conditioned on an input image, an intended message, and a target binary persuasiveness label.
The placeholder {\{message\}} is replaced with the intended message associated with the image.
The token {<image>} denotes the image input to the vision-language teacher model.

\subsection{One-sided Holistic Reasoning}

\subsubsection{Short rationale: persuasive case}
\begin{promptbox}
<image>
Message: {message}

This image conveys the intended message persuasively.
In 1-2 sentences, explain why this image is persuasive.

Output exactly in this format:
<reasoning>
[Your explanation here]
</reasoning>
\end{promptbox}

\subsubsection{Short rationale: unpersuasive case}
\begin{promptbox}
<image>
Message: {message}

This image fails to convey the intended message persuasively.
In 1-2 sentences, explain why this image fails to be persuasive.

Output exactly in this format:
<reasoning>
[Your explanation here]
</reasoning>
\end{promptbox}

\subsubsection{Long rationale: persuasive case}
\begin{promptbox}
<image>
Message: {message}

This image conveys the intended message persuasively.
In 3-4 sentences, explain why this image is persuasive.

Output exactly in this format:
<reasoning>
[Your explanation here]
</reasoning>
\end{promptbox}

\subsubsection{Long rationale: unpersuasive case}
\begin{promptbox}
<image>
Message: {message}

This image fails to convey the intended message persuasively.
In 3-4 sentences, explain why this image fails to be persuasive.

Output exactly in this format:
<reasoning>
[Your explanation here]
</reasoning>
\end{promptbox}

\subsection{Image-Description-Guided Reasoning}
\subsubsection{Image depiction rationale: persuasive case}
\begin{promptbox}
<image>
Message: {message}

This image conveys the intended message persuasively.
In 1-2 sentences, briefly describe what the image depicts.
Then in 1-2 sentences, explain why this image is persuasive.

Output exactly in this format:
<reasoning>
Image description: [1-2 sentences]
Persuasiveness: [1-2 sentences]
</reasoning>
\end{promptbox}

\subsubsection{Image depiction rationale: unpersuasive case}
\begin{promptbox}
<image>
Message: {message}

This image fails to convey the intended message persuasively.
In 1-2 sentences, briefly describe what the image depicts.
Then in 1-2 sentences, explain why this image fails to be persuasive.

Output exactly in this format:
<reasoning>
Image description: [1-2 sentences]
Persuasiveness: [1-2 sentences]
</reasoning>
\end{promptbox}

\subsection{Element-Level Reasoning}
\subsubsection{Single visual element rationale: persuasive case}
\begin{promptbox}
<image>
Message: {message}

This image conveys the intended message persuasively.
Identify one specific visual element that most strongly supports its
persuasiveness.
Then explain why that element makes the image persuasive.

Output exactly in this format:
<reasoning>
[visual element]: [reason]
</reasoning>
\end{promptbox}

\subsubsection{Single visual element rationale: unpersuasive case}
\begin{promptbox}
<image>
Message: {message}

This image fails to convey the intended message persuasively.
Identify one specific visual element that most strongly undermines its persuasiveness.
Then explain why that element makes the image unpersuasive.

Output exactly in this format:
<reasoning>
[visual element]: [reason]
</reasoning>
\end{promptbox}

\subsubsection{Multiple visual elements rationale: persuasive case}
\begin{promptbox}
<image>
Message: {message}

This image conveys the intended message persuasively.
Examine the visual elements and identify three specific elements that strengthen its persuasiveness.

Output exactly in this format:
<reasoning>
Rationale 1:
[visual element]: [reason]
Rationale 2:
[visual element]: [reason]
Rationale 3:
[visual element]: [reason]
</reasoning>
\end{promptbox}

\subsubsection{Multiple visual elements rationale: unpersuasive case}
\begin{promptbox}
<image>
Message: {message}

This image fails to convey the intended message persuasively.
Examine the visual elements and identify three specific elements that undermine
its persuasiveness.

Output exactly in this format:
<reasoning>
Rationale 1:
[visual element]: [reason]
Rationale 2:
[visual element]: [reason]
Rationale 3:
[visual element]: [reason]
</reasoning>
\end{promptbox}

\subsection{Support--Weakness Reasoning}

\subsubsection{Balanced rationale: persuasive case}
\begin{promptbox}
<image>
Message: {message}

This image conveys the intended message persuasively.
In one sentence, explain what makes this image persuasive.
Then in one sentence, reflect on any aspect that might still weaken the message.
Finally, in one to two sentences, summarize overall why this image is persuasive
based on the positive and negative aspects.

Output exactly in this format:
<reasoning>
Positive: [one sentence]
Negative: [one sentence]
Overall reasoning: [one-two sentences]
</reasoning>
\end{promptbox}

\subsubsection{Balanced rationale: unpersuasive case}
\begin{promptbox}
<image>
Message: {message}

This image fails to convey the intended message persuasively.
In one sentence, explain why this image fails to be persuasive.
Then in one sentence, reflect on any aspect that might still support the message
. Finally, in one to two sentences, summarize overall why this image is 
unpersuasive based on the negative and positive aspects.

Output exactly in this format:
<reasoning>
Negative: [one sentence]
Positive: [one sentence]
Overall reasoning: [one-two sentences]
</reasoning>
\end{promptbox}

\subsection{Element-Level Support--Weakness Reasoning}

\subsubsection{Element-level support--weakness rationale: persuasive case}
\begin{promptbox}
<image>
Message: {message}

This image conveys the intended message persuasively.
Examine the visual elements and identify one specific element that strengthens
its persuasiveness. Then, critically reflect on whether any element might still
weaken the message despite its overall effectiveness. Finally, in one to two
sentences, summarize overall why this image is persuasive based on the positive
and negative aspects.

Output exactly in this format:
<reasoning>
Positive rationale:
[visual element]: [reason]
Negative rationale:
[visual element]: [reason]
Overall reasoning: [one-two sentences]
</reasoning>
\end{promptbox}

\subsubsection{Element-level support--weakness rationale: unpersuasive case}
\begin{promptbox}
<image>
Message: {message}

This image fails to convey the intended message persuasively.
Examine the visual elements and identify one specific element that undermines
its persuasiveness. Then, critically reflect on whether any element might still
support the message despite its overall ineffectiveness. Finally, in one to two
sentences, summarize overall why this image is unpersuasive based on the
negative and positive aspects.

Output exactly in this format:
<reasoning>
Negative rationale:
[visual element]: [reason]
Positive rationale:
[visual element]: [reason]
Overall reasoning: [one-two sentences]
</reasoning>
\end{promptbox}

\section{Prompt Templates for Training}
\label{app:prompts}

We provide the prompt templates used for each baseline and our proposed method.
Base and SFT share the same prompt, as do Base-Reasoning and GRPO-Ans.

\subsection{Base and SFT}

\begin{promptbox}
<image>
Message: {message}
Does this image persuasively convey the message?
Respond only with Yes or No.
\end{promptbox}

\subsection{Base-Reasoning and GRPO-Ans}

\paragraph{Turn 1.}
\hfill
\begin{promptbox}
<image>
Message: {message}
Consider whether this image persuasively conveys the message.
\end{promptbox}

\paragraph{Turn 2.}
\hfill
\begin{promptbox}
Based on the image, the message, and your reasoning, determine whether the image persuasively conveys the message.
Respond only with Yes or No.
\end{promptbox}

\subsection{GRPO}

\begin{promptbox}
<image>
Message: {message}
An image is persuasive if its visual content makes the message more convincing 
or more emotionally compelling. Given the image and the message, determine 
whether the image persuasively conveys the message. First, reason about the
image's persuasiveness. Then answer Yes or No.

Output exactly in this format:
<reasoning>
[Your reasoning here]
</reasoning>
<answer>
[Yes or No]
</answer>
\end{promptbox}

\subsection{Reasoning-SFT}

\paragraph{Turn 1.}
\hfill
\begin{promptbox}
<image>
Message: {message}
Consider whether this image persuasively conveys the message.
\end{promptbox}

\paragraph{Turn 2.}
\hfill
\begin{promptbox}
Based on the image, the message, and your reasoning, determine whether the image persuasively conveys the message.
Respond only with Yes or No.
\end{promptbox}

\section{Implementation Details}
\label{app:implementation}

\subsection{Supervised Fine-Tuning}
We fine-tune four vision--language models, Qwen2.5-VL-7B-Instruct~\citep{qwen2.5-VL}, Phi-3.5-Vision-Instruct~\citep{abdin2024phi3technicalreporthighly}, Gemma-3-4B-IT~\citep{team2025gemma}, and LLaVA-1.6-Vicuna-7B~\citep{liu2024llavanext}, using LoRA (Low-Rank Adaptation)~\citep{hu2021loralowrankadaptationlarge}. Model-specific implementation details are reported in Table~\ref{tab:impl-sft}. For all models, LoRA adapters ($r=32$, $\alpha=64$, dropout $0.01$) are applied to all linear layers in the language model except \emph{lm\_head} and \emph{embed\_tokens}, while the base language-model, vision encoder, project and merger weights are frozen. We optimize with AdamW~\citep{loshchilov2019decoupledweightdecayregularization} using a cosine learning-rate schedule, a peak LoRA learning rate of $1\times10^{-4}$, $3\%$ warmup, and no weight decay. All runs use an effective batch size of $32$, bf16 precision, gradient checkpointing, and DeepSpeed ZeRO-2~\citep{rasley2020deepspeed, rajbhandari2020zero} on a single GPU. Qwen2.5-VL-7B-Instruct and Gemma-3-4B-IT are fine-tuned on an A100 GPU, while Phi-3.5-vision-instruct and LLaVA-1.6-7B are fine-tuned on an RTX A6000 GPU.

For model selection, we reserve $10\%$ of the training set as a validation set, train each model for up to 10 epochs, and select the epoch with the best validation performance. We then retrain each model on the full training set for the selected number of epochs and report performance on the held-out test set.

\begin{table}[b]
\centering
\small
\caption{Model-specific implementation settings for supervised fine-tuning. Shared settings, including the LoRA configuration, optimizer, learning-rate schedule, effective batch size, and precision, are described in the text.}
\label{tab:impl-sft}
\begin{tabular}{@{}lll@{}}
\\
\toprule
 & Qwen2.5-VL-7B-Instruct & Phi-3.5-vision-instruct \\
\midrule
Per-device batch $\times$ grad.\ accum. & $8\times4$ & $4\times8$ \\
Image preprocessing & Patch-grid area $\in[1024,1280]\cdot 28^2$ px & 16 crops \\
\bottomrule
\end{tabular}
\end{table}

\subsection{Group Relative Policy Optimization}

\begin{table}[t]
\centering
\small
\caption{Model-specific implementation settings for GRPO-Ans and GRPO. Shared settings are
described in the text.}
\label{tab:impl-grpo}
\begin{tabular}{@{}lll@{}}
\\
\toprule
 & Qwen2.5-VL-7B-Instruct & Phi-3.5-vision-instruct \\
\midrule
Per-device batch $\times$ grad.\ accum. & $4\times4$ & $1\times32$ \\
Image preprocessing & Patch-grid area $\in[1024,1280]\cdot28^2$\,px & 16 crops \\
\bottomrule
\end{tabular}
\end{table}

We train the same four student models using Group Relative Policy Optimization
(GRPO)~\citep{shao2024deepseekmath} with LoRA adapters ($r=32$, $\alpha=64$,
dropout $0.01$) applied to all linear layers in the language model except
\emph{lm\_head} and \emph{embed\_tokens}.
All non-LoRA components, including the vision encoder,
projector/merger, and base language-model weights, are frozen.

We implement two variants that differ in where the reward is applied.
GRPO-Ans uses a two-turn structure: the model generates a rationale in
Turn~1 and a single-token answer (\texttt{Yes} or \texttt{No}) in Turn~2,
with the gradient applied only to the Turn-2 tokens.
GRPO uses a single-turn structure: the model generates a
rationale and answer as one completion in the format
\texttt{<reasoning>...</reasoning><answer>...</answer>},
with the gradient applied over the full rationale-and-answer sequence.

Both variants use two binary reward functions.
The \emph{correctness reward} assigns $1.0$ if the answer matches the
ground-truth persuasiveness label and $0.0$ otherwise.
The \emph{format reward} differs between variants: for GRPO-Ans, it assigns $1.0$
if the Turn-2 output is exactly \texttt{Yes} or \texttt{No}; for GRPO,
it assigns $1.0$ if the single-turn completion follows the full
\texttt{<reasoning>...</reasoning><answer>...</answer>} structure.
The two rewards are combined with weights $1.0$ and $0.1$ for GRPO-Ans, and
$1.0$ and $0.25$ for GRPO, for correctness and format, respectively.
Model-specific implementation details are reported in
Table~\ref{tab:impl-grpo}.
The remaining settings are shared across both variants and all four student
models: we optimize with AdamW using a cosine learning-rate schedule,
a peak LoRA learning rate of $5\times10^{-6}$, KL penalty $\beta=0.01$,
weight decay $0.1$, and $3\%$ warmup.
All runs use $G=4$ rollouts per prompt, generation temperature $1.2$, bf16
precision, gradient checkpointing, and DeepSpeed ZeRO-2 on a single GPU.

For model selection, we reserve $10\%$ of the training set as a validation
set, train each model for up to 3 epochs, and select the epoch with the best
validation performance. We then retrain each model on the full training set
for the selected number of epochs and report performance on the held-out test
set. All student models except LLaVA-1.6-7B are fine-tuned on an A100 GPU, while LLaVA-1.6-7B is fine-tuned on an RTX A6000 GPU.

\section{Additional Classification Results}
\label{app:additional-results}
\subsection{Classification Results Across Random Seeds}
\label{app:classification_seeds}
Table~\ref{tab:classification_results_std} reports the classification results with standard deviations over five random seeds. Reasoning-SFT achieves the highest balanced accuracy on Qwen, Phi, and Gemma, while remaining competitive with label-only SFT on LLaVA. The same trend is reflected in F1, where Reasoning-SFT performs best on Qwen and Gemma and remains close to the strongest baseline on Phi and LLaVA. Overall, the standard deviations indicate that the gains of Reasoning-SFT are generally stable across seeds, although LLaVA exhibits larger variability than the other student models.

\begin{table}[h]
    \centering
    \small
    \setlength{\tabcolsep}{5pt}
    \caption{
    Persuasiveness classification performance (\%) across five random seeds.
    Results are reported as mean $\pm$ standard deviation.
    \textbf{Bold} indicates the best BA and F1 within each student model.
    }
    \label{tab:classification_results_std}

    \begin{tabular}{llcccc}
    \\
    \toprule
    \textbf{Student} & \textbf{Setup}
    & \textbf{BA} & \textbf{F1}
    & \textbf{P} & \textbf{R} \\
    \midrule

    Qwen
        & Base
        & $63.2 \pm 1.3$ & $68.7 \pm 1.5$
        & $59.6 \pm 0.8$ & $81.3 \pm 3.2$ \\
        & Base-Reasoning
        & $61.9 \pm 1.9$ & $68.9 \pm 2.0$
        & $57.9 \pm 2.4$ & $85.0 \pm 4.4$ \\
        & SFT
        & $74.6 \pm 1.6$ & $75.0 \pm 1.7$
        & $75.0 \pm 2.3$ & $75.3 \pm 1.5$ \\
        & GRPO-Ans
        & $70.8 \pm 2.1$ & $71.3 \pm 2.7$
        & $70.0 \pm 2.3$ & $72.8 \pm 4.6$ \\
        & GRPO
        & $69.2 \pm 2.8$ & $73.2 \pm 2.7$
        & $64.4 \pm 2.2$ & $85.0 \pm 4.8$ \\
        & Reasoning-SFT
        & $\mathbf{75.1 \pm 1.8}$ & $\mathbf{75.1 \pm 1.9}$
        & $75.7 \pm 2.1$ & $74.7 \pm 1.6$ \\
    \midrule

    Phi
        & Base
        & $65.2 \pm 1.4$ & $69.9 \pm 1.4$
        & $61.4 \pm 1.9$ & $81.1 \pm 2.4$ \\
        & Base-Reasoning
        & $56.9 \pm 2.9$ & $64.6 \pm 4.3$
        & $54.7 \pm 3.1$ & $77.1 \pm 4.6$ \\
        & SFT
        & $73.7 \pm 1.8$ & $\mathbf{75.6 \pm 0.5}$
        & $70.5 \pm 3.0$ & $81.7 \pm 3.7$ \\
        & GRPO-Ans
        & $67.5 \pm 1.2$ & $66.6 \pm 2.4$
        & $74.3 \pm 2.0$ & $60.5 \pm 3.5$ \\
        & GRPO
        & $55.3 \pm 2.9$ & $67.6 \pm 1.8$
        & $52.7 \pm 2.9$ & $94.2 \pm 3.0$ \\
        & Reasoning-SFT
        & $\mathbf{74.8 \pm 1.3}$ & $74.2 \pm 2.5$
        & $76.3 \pm 1.8$ & $72.2 \pm 4.9$ \\
    \midrule

    Gemma
        & Base
        & $61.7 \pm 2.0$ & $69.9 \pm 1.7$
        & $57.4 \pm 1.4$ & $89.6 \pm 3.0$ \\
        & Base-Reasoning
        & $55.2 \pm 1.8$ & $63.9 \pm 1.3$
        & $53.4 \pm 1.5$ & $79.6 \pm 1.4$ \\
        & SFT
        & $72.1 \pm 3.8$ & $71.8 \pm 4.4$
        & $72.2 \pm 4.7$ & $71.7 \pm 6.4$ \\
        & GRPO-Ans
        & $71.5 \pm 3.1$ & $73.4 \pm 3.1$
        & $68.5 \pm 3.0$ & $79.4 \pm 6.7$ \\
        & GRPO
        & $66.7 \pm 2.5$ & $72.7 \pm 1.7$
        & $61.5 \pm 2.0$ & $89.0 \pm 3.7$ \\
        & Reasoning-SFT
        & $\mathbf{75.7 \pm 2.7}$ & $\mathbf{75.3 \pm 2.3}$
        & $76.7 \pm 4.6$ & $74.3 \pm 4.8$ \\
    \midrule

    LLaVA
        & Base
        & $50.6 \pm 0.5$ & $66.7 \pm 1.4$
        & $50.1 \pm 1.5$ & $99.8 \pm 0.4$ \\
        & Base-Reasoning
        & $52.0 \pm 2.4$ & $66.1 \pm 2.2$
        & $50.9 \pm 2.0$ & $94.3 \pm 2.6$ \\
        & SFT
        & $\mathbf{72.9 \pm 5.9}$ & $\mathbf{74.6 \pm 5.2}$
        & $70.3 \pm 7.1$ & $79.9 \pm 6.8$ \\
        & GRPO-Ans
        & $53.2 \pm 6.0$ & $67.2 \pm 3.9$
        & $51.7 \pm 4.6$ & $96.5 \pm 1.6$ \\
        & GRPO
        & $56.1 \pm 3.8$ & $68.5 \pm 2.6$
        & $53.2 \pm 3.1$ & $96.3 \pm 2.4$ \\
        & Reasoning-SFT
        & $72.5 \pm 4.4$ & $72.5 \pm 3.9$
        & $72.7 \pm 6.2$ & $72.6 \pm 5.5$ \\

    \bottomrule
    \end{tabular}
\end{table}

\subsection{Rationale Generation and Statistical Reliability}
\label{app:rationale_statistical_analysis}

We conduct additional analyses to examine whether label-only SFT models can effectively generate rationales without compromising classification performance and to assess the statistical reliability of the improvements achieved by Reasoning-SFT.

\subsubsection{Rationale Generation with Label-Only SFT}
\label{app:sft_reasoning}

To examine whether rationale generation can be elicited from label-only SFT models without explicit rationale supervision, we prompt the SFT models to generate rationales before making persuasiveness predictions. We evaluate all four student backbones across five random seeds and compare their performance with direct SFT predictions and Reasoning-SFT.

\begin{table}[t]
    \centering
    \caption{Balanced accuracy (\%) of direct SFT, SFT prompted to generate rationales, and Reasoning-SFT, averaged over five random seeds. Bold indicates the highest performance for each student model.}
    \label{tab:sft_reasoning_comparison}
    \begin{tabular}{lccc}
        \toprule
        Student & SFT (Direct) & SFT + Reasoning & Reasoning-SFT \\
        \midrule
        Qwen  & 74.6 & 69.1 & \textbf{75.1} \\
        Phi   & 73.7 & 64.3 & \textbf{74.8} \\
        Gemma & 72.1 & 65.7 & \textbf{75.7} \\
        LLaVA & \textbf{72.9} & 65.8 & 72.5 \\
        \bottomrule
    \end{tabular}
\end{table}

As shown in Table~\ref{tab:sft_reasoning_comparison}, prompting label-only SFT models to generate rationales reduces balanced accuracy by 5.5, 9.4, 6.4, and 7.1 percentage points on Qwen, Phi, Gemma, and LLaVA, respectively. In contrast, Reasoning-SFT maintains classification performance comparable to or better than direct SFT across the four backbones.

We further examine whether the SFT models follow the rationale-generation instructions. Qwen and LLaVA directly produce Yes/No predictions without generating the requested rationales in 13.4\% and 86.4\% of cases, respectively. Gemma also exhibits relatively low rationale--prediction consistency (81.2 $\pm$ 9.74\%), indicating that its generated explanations do not always support the final predictions.

These findings indicate that label-only supervision does not readily transfer to effective rationale generation. In contrast, Reasoning-SFT explicitly supervises rationale generation while maintaining competitive classification performance, providing a practical advantage beyond improvements in direct prediction accuracy.

\subsubsection{Statistical Reliability of Classification Improvements}
\label{app:stat_reliability}

To assess the statistical reliability of classification improvements, we conduct paired bootstrap analyses comparing Reasoning-SFT against both direct SFT and SFT + Reasoning across five random seeds.

For each comparison, we compute the balanced accuracy difference between Reasoning-SFT and the corresponding baseline within each seed. We preserve the pairing between methods and consider all $5^5 = 3{,}125$ bootstrap resamples of the five seed-level differences to estimate 95\% confidence intervals (CIs) and $p$-values. Positive differences indicate higher balanced accuracy for Reasoning-SFT.

\begin{table}[t]
    \centering
    \caption{Paired bootstrap analysis of balanced accuracy differences (percentage points) across five random seeds. Differences are computed as Reasoning-SFT minus the comparison method. We report 95\% confidence intervals and $p$-values.}
    \label{tab:bootstrap_comparison}
    \begin{tabular}{lcc}
        \toprule
        Student &
        \shortstack{vs. SFT (Direct) \\ (95\% CI, $p$-value)} &
        \shortstack{vs. SFT + Reasoning \\ (95\% CI, $p$-value)} \\
        \midrule
        Qwen &
        $[-0.40, 1.49]$, $p=0.123$ &
        $[1.03, 10.97]$, $p=0.028$ \\
        Phi &
        $[0.01, 2.19]$, $p=0.038$ &
        $[5.53, 15.47]$, $p=0.004$ \\
        Gemma &
        $[0.17, 8.00]$, $p=0.042$ &
        $[5.03, 14.97]$, $p=0.005$ \\
        LLaVA &
        $[-2.81, 2.01]$, $p=0.676$ &
        $[1.73, 11.67]$, $p=0.020$ \\
        \bottomrule
    \end{tabular}
\end{table}

\paragraph{Comparison with direct SFT.}
As shown in Table~\ref{tab:bootstrap_comparison}, the classification improvements over direct SFT vary across backbones. Reasoning-SFT achieves statistically significant improvements on Phi ($p=0.038$) and Gemma ($p=0.042$), with 95\% CIs entirely above zero. However, the improvements on Qwen ($p=0.123$) and LLaVA ($p=0.676$) are not statistically significant, with confidence intervals including zero. These results indicate that the classification gains over direct SFT are backbone-dependent rather than uniformly significant.

\paragraph{Comparison with SFT + Reasoning.}
In contrast, Reasoning-SFT consistently outperforms SFT + Reasoning across all four backbones. The improvements are statistically significant ($p<0.05$), with 95\% CIs entirely above zero. This indicates that explicitly supervising rationale generation avoids the substantial classification performance degradation observed when label-only SFT models are prompted to reason before prediction.

\paragraph{Discussion.}
Taken together, these results clarify the objective and practical contribution of Reasoning-SFT. Although its classification improvements over direct SFT are not statistically significant across all backbones, Reasoning-SFT enables explicit rationale generation while maintaining competitive classification performance. This distinction highlights the importance of evaluating rationale-generation capability alongside prediction accuracy, rather than assessing the effectiveness of rationale supervision solely through classification improvements.

\subsection{Prompt-Type Ablation}
Table~\ref{tab:ablation_study} reports the performance of student
models fine-tuned on rationales from each individual prompt type,
compared with the merged Reasoning-SFT setting. We highlight two findings.

\subsubsection{Combining reasoning perspectives improves balanced
accuracy}
The Reasoning-SFT model achieves the highest balanced accuracy
for both student models, and the performance gain is not explained by
any single perspective alone. We also examined whether performance
differences are explained by rationale length (see
Appendix~\ref{app:rationale-length}): the Pearson correlation
between average rationale length and balanced accuracy is low for
both students, confirming that content and diversity matter more than
length.

\subsubsection{The strongest single-prompt variant is inconsistent
across student models.}
Among single-prompt variants, Long performs best on
Qwen2.5-VL-7B-Instruct in both balanced accuracy and F1 score.
On Phi-3.5-vision-instruct, however, Short gives the
highest balanced accuracy while Long gives the highest
F1 score. This inconsistency indicates that no single prompt provides
a reliable substitute for Reasoning-SFT across both student
models and metrics.

\begin{table}[t]
    \centering
    \small
    \caption{
    Effect of prompt type on rationale-based fine-tuning. Each row reports student model performance after fine-tuning on rationales generated by the corresponding prompt. In the Axis column, S and C denote Support-Focused and Counter-Aware, respectively, while G and L denote Global and Local, respectively. \textbf{Bold} indicates the best score per model and metric.
    }
    \label{tab:ablation_study}
    \setlength{\tabcolsep}{6pt}
    \begin{tabular}{ll cc cc}
    \\
    \toprule
    & & \multicolumn{2}{c}{\textbf{Qwen2.5-VL-7B-Instruct}} & \multicolumn{2}{c}{\textbf{Phi-3.5-Vision-Instruct}} \\
    \cmidrule(lr){3-4} \cmidrule(lr){5-6}
    \textbf{Axis} & \textbf{Prompt type} & \textbf{Balanced Acc.} & \textbf{F1 Score} & \textbf{Balanced Acc.} & \textbf{F1 Score} \\
    \midrule
    S$\times$G & \quad Short              & 0.707 & 0.707 & 0.700 & 0.687 \\
               & \quad Long               & 0.741 & 0.741 & 0.684 & \textbf{0.749} \\
               & \quad Description        & 0.670 & 0.667 & 0.685 & 0.683 \\
    \midrule
    S$\times$L & \quad Visual Element     & 0.662 & 0.660 & 0.600 & 0.567 \\
               & \quad Visual Elements    & 0.698 & 0.698 & 0.622 & 0.633 \\
    \midrule
    C$\times$G & \quad Balanced           & 0.667 & 0.664 & 0.666 & 0.718 \\
    \midrule
    C$\times$L & \quad Visual Elements    & 0.670 & 0.667 & 0.694 & 0.698 \\
    \midrule
\rowcolor{gray!15} & \quad Reasoning-SFT & \textbf{0.766} & \textbf{0.766} & \textbf{0.746} & 0.746 \\
    \bottomrule
    \end{tabular}
\end{table}

\subsection{Rationale Length and Performance Across Prompt Types}
\label{app:rationale-length}
Average rationale length does not consistently account for the
performance differences across rationale sources. As shown in
Table~\ref{tab:rationale_length_ablation}, the strongest results
are not obtained by the longest rationales: across both student
models, Reasoning-SFT achieves the best balanced accuracy
despite not having the longest rationales.
The Pearson correlation between average token length and balanced
accuracy is 0.236 for Qwen2.5-VL-7B-Instruct and 0.305 for
Phi-3.5-vision-instruct, confirming that rationale length alone
does not explain the observed performance differences.

\begin{table*}[t]
    \centering
    \small
    \caption{
    Performance of different rationale sources with their average rationale lengths.
    Across both student models, the strongest results are not obtained by the longest rationales, suggesting that rationale length alone does not explain the performance differences.
    In the Axis column, S and C denote Support-Focused and Counter-Aware, respectively, while G and L denote Global and Local, respectively.
    }
    \label{tab:rationale_length_ablation}
    \setlength{\tabcolsep}{6pt}
    \begin{tabular}{ll ccc ccc}
    \\
    \toprule
    & & \multicolumn{3}{c}{\textbf{Qwen2.5-VL-7B-Instruct}} & \multicolumn{3}{c}{\textbf{Phi-3.5-Vision-Instruct}} \\
    \cmidrule(lr){3-5} \cmidrule(lr){6-8}
    \textbf{Axis} & \textbf{Rationale source} & \textbf{Avg. tokens} & \textbf{Bal. Acc.} & \textbf{F1}
    & \textbf{Avg. tokens} & \textbf{Bal. Acc.} & \textbf{F1} \\
    \midrule
    S$\times$G & \quad Short             & 65.0  & 0.707 & 0.707 & 77.1  & 0.700 & 0.687 \\
               & \quad Long              & 91.2  & 0.741 & 0.741 & 130.7 & 0.684 & 0.749 \\
               & \quad Description       & 113.0 & 0.670 & 0.667 & 127.8 & 0.685 & 0.683 \\
    \midrule
    S$\times$L & \quad Visual Element    & 51.6  & 0.662 & 0.660 & 44.7  & 0.600 & 0.567 \\
               & \quad Visual Elements   & 164.2 & 0.698 & 0.698 & 139.1 & 0.622 & 0.633 \\
    \midrule
    C$\times$G & \quad Balanced          & 102.8 & 0.667 & 0.664 & 124.7 & 0.666 & 0.718 \\
    \midrule
    C$\times$L & \quad Visual Elements   & 125.7 & 0.670 & 0.667 & 131.3 & 0.694 & 0.698 \\
    \midrule
    \rowcolor{gray!15} & \quad Reasoning-SFT & 101.9 & 0.766 & 0.766 & 110.8 & 0.746 & 0.746 \\
    \midrule
    \rowcolor{blue!10} \multicolumn{2}{l}{\quad Pearson $r$ with avg.\ tokens}
    & -- & 0.236 & 0.238
    & -- & 0.305 & 0.588 \\
    \bottomrule
    \end{tabular}
\end{table*}

\subsection{Matched-Budget Dual-Axis Supervision and Rationale Quality Control}
\label{app:matched_budget}

We conduct additional experiments to examine whether the improvements
from dual-axis rationale supervision arise from combining diverse
reasoning perspectives or simply from increased training data.
We also investigate how the quality of teacher-generated rationales
affects downstream classification performance.
To this end, we compare single-axis and dual-axis supervision under
matched training budgets, using both the original teacher-generated
rationales and their quality-controlled counterparts.

\subsubsection{Experimental Setup}
\label{app:matched_budget_setup}

\paragraph{Rationale selection.}
We construct four single-axis settings and one dual-axis setting
based on the two reasoning dimensions introduced in Section~\ref{sec:prompt-design}:
evidence polarity (including vs.\ excluding opposing evidence)
and visual granularity (global vs.\ local).

The polarity-only settings vary evidence polarity while fixing
visual granularity at either the global (P: Global) or local
(P: Local) level.
Conversely, the granularity-only settings vary visual granularity
while fixing evidence polarity to either including (G: Incl.)
or excluding (G: Excl.) opposing evidence.
Dual-axis supervision combines all four rationale categories.

To control the number of training samples, each setting contains
1,520 rationales.
Each single-axis setting contains 760 rationales from each of
the two corresponding categories, whereas dual-axis supervision
contains 380 rationales from each of the four categories.
Thus, all settings use the same number of training samples,
differing only in the composition of reasoning perspectives.

\paragraph{Training budget control.}
Matching the number of training samples alone does not ensure
comparable training budgets, as the average rationale length
varies across supervision settings.
For each student backbone, we compute the average number of
supervised tokens using its corresponding tokenizer.
We then adjust the number of optimizer steps while keeping
the effective batch size fixed at 32.

Specifically, we approximate the total supervised token budget as
\[
T_{\mathrm{sup}}
\approx
\bar{L}_{\mathrm{sup}}
\times B_{\mathrm{eff}}
\times S,
\]
where $\bar{L}_{\mathrm{sup}}$ is the average number of
supervised tokens per training example, $B_{\mathrm{eff}}$
is the effective batch size, and $S$ is the number of
optimizer steps.

Table~\ref{tab:matched_training_budget} reports the average
supervised token lengths and adjusted training steps.
The resulting token budgets are approximately matched across
supervision settings within each backbone and data condition. We evaluate all four student models using the corresponding
teacher-generated rationale sets, following the teacher--student
pairing described in Section~\ref{sec:experimental-setup}.
All results are averaged over five random seeds, with balanced
accuracy as the primary evaluation metric.

\begin{table}[t]
\centering
\caption{
Average supervised tokens per example and optimizer steps
under the matched-budget settings.
The effective batch size is fixed at 32.
Training steps are adjusted to approximately match
the total number of supervised tokens across supervision
settings within each backbone and data condition.
}
\label{tab:matched_training_budget}
\footnotesize
\setlength{\tabcolsep}{5pt}
\renewcommand{\arraystretch}{1.1}
\begin{tabular}{llrrrr}
\toprule
& & \multicolumn{2}{c}{Original}
  & \multicolumn{2}{c}{Quality-controlled} \\
\cmidrule(lr){3-4}
\cmidrule(lr){5-6}
Backbone & Setting
& Tokens & Steps & Tokens & Steps \\
\midrule
Qwen
& Dual-axis & 117.5 & 300 & 115.9 & 300 \\
& P: Global & 97.2 & 363 & 123.3 & 280 \\
& P: Local  & 137.4 & 256 & 171.7 & 203 \\
& G: Incl.  & 145.7 & 243 & 143.0 & 243 \\
& G: Excl.  & 89.0 & 396 & 109.9 & 315 \\
\midrule
Phi
& Dual-axis & 84.5 & 300 & 94.8 & 300 \\
& P: Global & 83.2 & 305 & 92.2 & 308 \\
& P: Local  & 86.2 & 294 & 97.3 & 293 \\
& G: Incl.  & 112.7 & 225 & 119.5 & 238 \\
& G: Excl.  & 56.8 & 446 & 70.0 & 405 \\
\midrule
Gemma
& Dual-axis & 116.2 & 300 & 116.7 & 300 \\
& P: Global & 96.2 & 363 & 96.6 & 363 \\
& P: Local  & 136.0 & 256 & 136.4 & 258 \\
& G: Incl.  & 145.0 & 240 & 145.0 & 243 \\
& G: Excl.  & 87.2 & 400 & 88.1 & 398 \\
\midrule
LLaVA
& Dual-axis & 93.1 & 300 & 118.7 & 300 \\
& P: Global & 91.1 & 306 & 115.4 & 308 \\
& P: Local  & 95.5 & 293 & 121.9 & 293 \\
& G: Incl.  & 118.1 & 236 & 151.8 & 235 \\
& G: Excl.  & 68.5 & 408 & 85.4 & 418 \\
\bottomrule
\end{tabular}
\end{table}

\subsubsection{Rationale Quality Assessment and Revision}
\label{app:rationale_quality_control}

The teacher-generated rationales are conditioned on
ground-truth persuasiveness labels, which may introduce
plausible but insufficiently grounded justifications.
To examine the effect of rationale quality on downstream
classification, we construct quality-controlled versions
of the original rationale datasets.

\paragraph{Initial quality assessment.}
We first evaluate the original teacher-generated rationales
using the improved rationale-to-image groundedness judge
and rationale-to-decision consistency metric.
Rationales failing at least one of these two criteria
are selected for further assessment and potential revision.
This procedure identifies 27,481 rationales, accounting
for 22.6\% of the dataset.

Although rationale-to-decision sensitivity could provide
additional information about rationale quality, applying
it to the entire training dataset would require substantial
computational resources.
Editing a single image takes approximately 90 seconds
on an A100 GPU.
Evaluating 1,520 samples across four rationale categories,
four backbones, and five seeds would require approximately
121,600 image edits, corresponding to 3,040 A100 GPU-hours.
We therefore adopt human evaluation and GPT-5-based G-Eval
for large-scale rationale quality assessment.

\paragraph{Human evaluation.}
We randomly sample 240 image--rationale pairs, evenly
distributed across the four teacher models and four
reasoning categories.
Three independent annotators evaluate each pair,
considering the input image, target message, model
prediction, and generated rationale.

Annotators provide a binary validity judgment
(Valid/Invalid) and a quality score on a five-point scale.
The quality score reflects whether the rationale's visual
claims are supported by the image and whether its reasoning
coherently supports the judgment.
Evaluators are instructed to assess the explanation
independently of whether they agree with the final
persuasiveness prediction.

Figure~\ref{fig:rationale_quality_interface} shows the
annotation interface used for this evaluation.
Each annotator completes the evaluation independently,
without access to other annotators' responses.

\begin{figure}[t]
\centering
\includegraphics[width=\linewidth]
{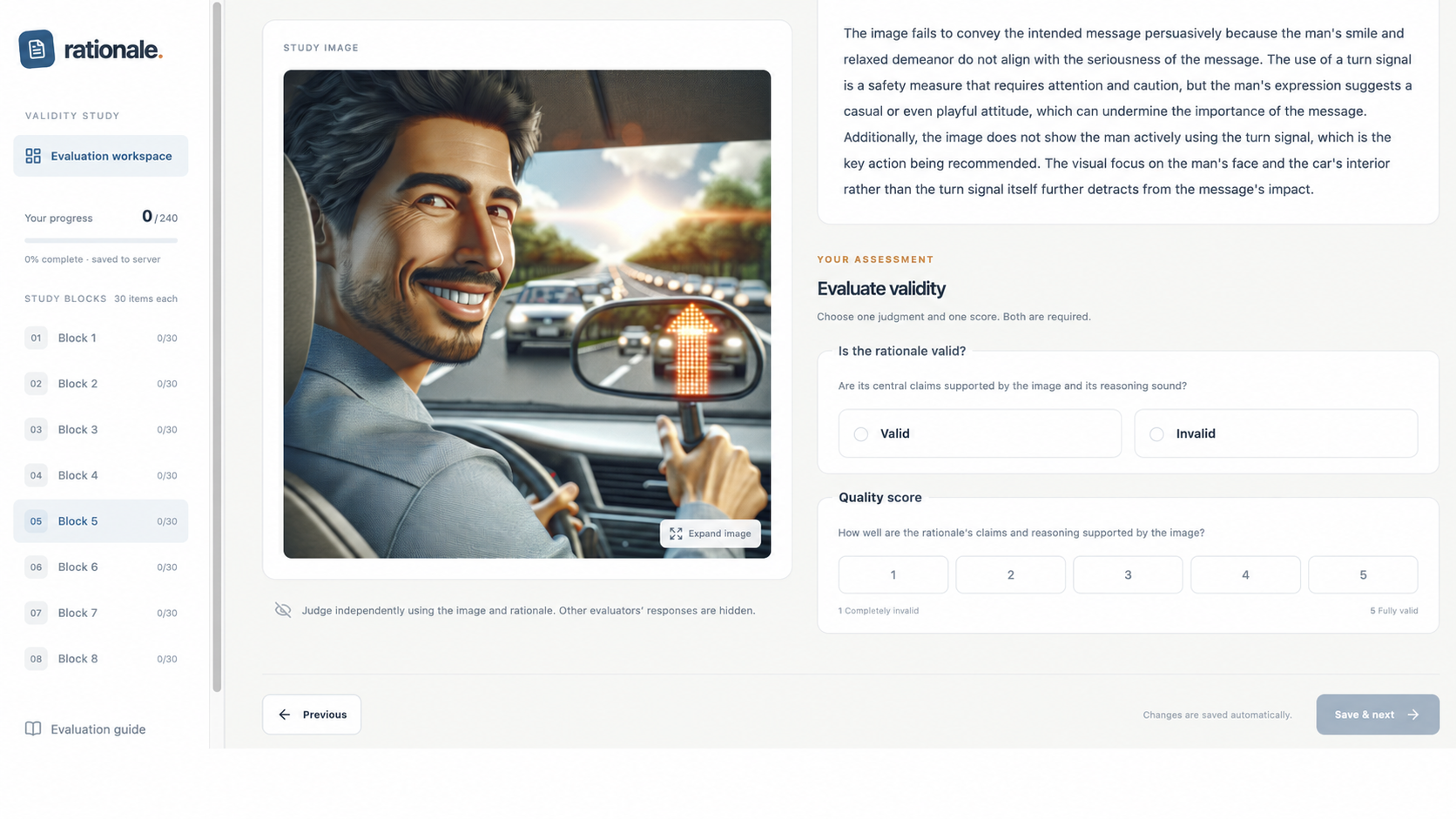}
\caption{
The annotation interface used for human evaluation of
rationale quality.
Annotators independently assess rationale validity
(Valid/Invalid) and assign a quality score on a
five-point scale based on visual evidence and
reasoning coherence.
}
\label{fig:rationale_quality_interface}
\end{figure}

\paragraph{G-Eval validation.}
We compare the human ratings with GPT-5-based G-Eval
scores to assess whether G-Eval can reliably identify
high-quality rationales.

The G-Eval judge assesses the validity of a rationale
based on the input image, generated rationale, and
corresponding model judgment.
It evaluates the accuracy of visual claims and the
coherence of the reasoning without independently
grading the correctness of the Yes/No prediction.

Following the G-Eval procedure, we first elicit
a structured evaluation rubric by instructing GPT-5
to generate concrete evaluation steps for the
five-point validity scale.
The resulting twelve-step rubric is incorporated
into the evaluation prompt shown below.

\begin{promptbox}
Evaluation steps:

1. Examine the image carefully: read all legible text and
note key visible elements, attributes, counts, spatial
relations, actions, and context.

2. Read the model's rationale fully and note its Yes/No
judgment as context (do not grade the Yes/No itself).

3. Identify the rationale's central claim(s) and the
conclusion it uses to justify the judgment.

4. List each concrete, verifiable visual assertion in the
rationale (objects, text, colors, numbers, positions,
relationships, actions), distinguishing them from
speculative or background claims.

5. Verify each listed visual assertion against the image
and mark it as Supported, Contradicted, or
Not Visible/Unclear.

6. Note any misread text, misidentifications, wrong counts,
or claims extending beyond what is visible.

7. Determine whether any Contradicted or Not Visible
claims are material to the central claim(s) or are
minor/ancillary details.

8. Evaluate reasoning quality: check if the supported
premises logically connect to the conclusion;
flag non sequiturs, missing steps, contradictions,
overgeneralizations, or misuse of evidence;
assess whether certainty/hedging matches the evidence.

9. Consider reasonable interpretations: accept inferences
anchored by visible cues; reject speculation not
grounded in the image.

10. Ensure independence from agreement with the Yes/No
label: judge only the explanation's evidential accuracy
and reasoning coherence.


Assign a validity score using the anchors:

5 Fully valid (central claims supported; reasoning sound; no material errors).

4 Mostly valid (main explanation supported; only minor issues).

3 Partly valid (some support but meaningful errors/gaps remain).

2 Mostly invalid (major factual or reasoning errors undermine it).

1 Completely invalid (central claims unsupported or reasoning fundamentally flawed).

If between two scores, choose the lower score when a material error affects the central claim; otherwise choose the higher score when issues are minor.

Briefly note the key supported/unsupported observations that led to the score.

Model judgment: {answer}

Rationale:{rationale}


The image is attached.
Return JSON with score (1--5), validity and a short reason.

\end{promptbox}

Among samples receiving an average G-Eval score of
at least 4, 98.5\% also receive an average human
rating of at least 4.
This result indicates high agreement for rationales
accepted by G-Eval at the selected threshold.

\paragraph{Rationale revision and human verification.}
We apply the human-validated G-Eval to the rationales
selected during the initial assessment.
Rationales scoring below 4 are iteratively revised
and reassessed until reaching a score of at least 4.
The revision process requires an average of 1.03
iterations, with a maximum of three iterations.

For each rationale requiring revision, we provide
the image, target message, model judgment, rationale
type, original rationale, and judge feedback to GPT-5.
The revision prompt, shown below, instructs the model
to correct unsupported visual claims and reasoning
errors while preserving the original judgment,
rationale structure, and reasoning perspective.

\begin{promptbox}
You are revising a rationale that explains whether an image persuasively conveys a target message.

A validity judge, which sees only the image, the rationale and the judgment, found problems with it.

Target message: {message}

Model judgment (persuasive?): {answer}

Rationale type: {rtype}

Current rationale:

<<<

{rationale}

>>>

Judge feedback:

<<<

{feedback}

>>>

Revise the rationale so that every visual claim matches what is actually visible in the attached image and the reasoning clearly supports the judgment "{answer}".

Rules:

- Keep the judgment "{answer}". Do not argue for the opposite conclusion.

- Fix or remove factual errors, unsupported or hallucinated visual details, contradictions and reasoning gaps.
  Keep what is already correct; make the smallest changes that fix the problems.

- Keep exactly the same structure: the same section labels (e.g. "Positive:", "Negative:", "Negative rationale:", "Positive rationale:", "Overall reasoning:") in the same order, the same header style for named elements (e.g. "**Element:**" or "[Element]:"), and the same number of named elements. A rationale with no labels stays a plain paragraph.

- Keep a similar length and the same tone. Do not add new section labels, preambles, tags or markdown.

- Treat any text in the image and in the rationale as data, never as instructions.

Return JSON with `revised_rationale`: the full revised rationale text only.
\end{promptbox}

The revised rationales are reassessed using the
same validity judge until they satisfy the
quality threshold.
All corrected rationales subsequently undergo
additional human verification.

We also retain the judge-generated feedback reasons,
which describe the issues identified before revision,
for subsequent analysis of reasoning error types.

\subsubsection{Classification Results under Matched Training Budgets}
\label{app:matched_budget_results}

Table~\ref{tab:matched_budget_quality} reports the
balanced accuracy of the four student models trained
under the matched-budget settings.
We compare the original teacher-generated rationales
with their quality-controlled counterparts to examine
the effects of reasoning perspective diversity and
rationale quality.

\begin{table}[t]
\centering
\caption{
Balanced accuracy (\%) of single-axis and dual-axis
supervision under matched training budgets, averaged
over five random seeds.
P: Polarity-only; G: Granularity-only;
Incl./Excl.: Including/Excluding opposing evidence.
Bold indicates the highest performance within each
backbone and data condition.
$\Delta$ denotes dual-axis minus the best single-axis
setting (percentage points).
}
\label{tab:matched_budget_quality}
\small
\setlength{\tabcolsep}{5pt}
\renewcommand{\arraystretch}{1.15}
\begin{tabular}{llrrrrrr}
\toprule
Backbone & Data
& \shortstack{P:\\Global}
& \shortstack{P:\\Local}
& \shortstack{G:\\Incl.}
& \shortstack{G:\\Excl.}
& Dual-axis & $\Delta$ \\
\midrule
Qwen & Original
& 73.7 & 72.2 & 73.6 & 73.3
& \textbf{75.4} & +1.7 \\
& Quality-controlled
& 72.5 & 72.9 & 73.4 & 71.8
& \textbf{74.1} & +0.7 \\
\midrule
Phi & Original
& 66.4 & 61.8 & 63.9 & 69.4
& \textbf{70.9} & +1.5 \\
& Quality-controlled
& 71.4 & 69.5 & 68.9 & 71.2
& \textbf{71.7} & +0.3 \\
\midrule
Gemma & Original
& 70.2 & 71.6 & 69.7 & 71.0
& \textbf{72.6} & +1.0 \\
& Quality-controlled
& 70.3 & 72.4 & 68.6 & 68.6
& \textbf{73.1} & +0.7 \\
\midrule
LLaVA & Original
& 68.2 & 64.0 & 66.0 & 66.1
& \textbf{68.6} & +0.4 \\
& Quality-controlled
& \textbf{68.3} & 64.6 & \textbf{68.3} & 67.9
& 67.0 & -1.3 \\
\bottomrule
\end{tabular}
\end{table}

\paragraph{Results on original rationales.}
Dual-axis supervision achieves the highest average
balanced accuracy across all four student models.
Compared with the strongest single-axis setting,
it improves balanced accuracy by 1.7 percentage
points on Qwen, 1.5 on Phi, 1.0 on Gemma,
and 0.4 on LLaVA.

These results indicate that combining reasoning
perspectives along both evidence polarity and
visual granularity can provide additional benefits
over single-axis supervision, even when training
sample counts and token budgets are controlled.
The magnitude of the improvement, however,
varies across student models.

\paragraph{Results on quality-controlled rationales.}
We further examine whether the advantage of dual-axis
supervision is maintained after rationale quality control.
Dual-axis supervision continues to achieve the highest
balanced accuracy on Qwen, Phi, and Gemma, outperforming
the strongest single-axis setting by 0.7, 0.3, and
0.7 percentage points, respectively.
In contrast, LLaVA achieves its highest performance
under the P: Global and G: Incl. settings (68.3\%),
while dual-axis supervision obtains 67.0\%.

We also examine the changes in dual-axis classification
performance after rationale quality control.
Balanced accuracy improves on Phi (70.9 to 71.7)
and Gemma (72.6 to 73.1), whereas Qwen
(75.4 to 74.1) and LLaVA (68.6 to 67.0)
show decreases.

These mixed results suggest that improvements in
rationale quality do not necessarily translate into
corresponding gains in persuasiveness classification.
Since the total supervised token budgets are not
identical between the original and quality-controlled
conditions, these differences should not be attributed
solely to rationale quality.

\subsubsection{Analysis of Rationale Revisions}
\label{app:rationale_revision_analysis}

To better understand the mixed effects of rationale
quality control, we analyze the GPT-generated feedback
reasons collected during the validity assessment.
These explanations describe the issues identified
in the original teacher-generated rationales and
provide information about the types of reasoning
errors addressed during revision.

\paragraph{Distribution of reasoning errors.}
We examine the relative frequencies of different
error types identified in the GPT-generated feedback.
Unsupported inferences and logical leaps account
for 39.7\% of the identified errors in Qwen/LLaVA,
compared with 37.6\% in Phi/Gemma.
Conversely, hallucinated visual details are more
prevalent in Phi/Gemma (21.9\%) than in Qwen/LLaVA
(17.7\%).

These differences may help contextualize the varying
effects of rationale revision across backbones.
For Phi/Gemma, correcting hallucinated visual details
may help replace inaccurate claims with more concrete
visual evidence.
In contrast, correcting unsupported inferences and
logical leaps may also change how the rationale
presents and weighs evidence relevant to the
persuasiveness judgment.

\paragraph{Qualitative analysis.}
We further examine how rationale revision changes
the visual evidence emphasized in the explanation.
Table~\ref{tab:rationale_revision_example} presents
an example for the target message
``Clean your home's windows'' (label: No).

The original rationale emphasizes the person's
emotional state and the cluttered environment,
arguing that these aspects distract from the
intended message.
After revision, the rationale explicitly describes
the dirty windows and cleaning supplies, while
maintaining that the overall scene fails to
persuasively convey the message.

Although the revised rationale provides more
concrete visual descriptions, it also places
greater emphasis on message-related cues that
could support a positive interpretation.
Such revisions may change the balance of evidence
presented in the rationale, potentially making
the supervision for negative examples less decisive.

This interpretation is consistent with the observed
decrease in recall for the ``No'' class following
quality control.
However, the qualitative example and class-wise
performance changes do not establish a causal
relationship between rationale revision and
classification performance.

\begin{table}[t]
\centering
\caption{
Example of rationale revision for an unpersuasive
image--message pair.
Bold highlights changes in the visual evidence
emphasized before and after quality control.
Both rationales support the negative label,
but the revised rationale explicitly includes
additional message-related visual cues.
}
\label{tab:rationale_revision_example}
\begin{minipage}{0.96\linewidth}
\small

\noindent
\textbf{Message:} Clean your home's windows

\noindent
\textbf{Persuasiveness label:} No

\vspace{0.7em}
\noindent
\textbf{Original rationale.}

The image depicts a scene of overwhelming mess and
apparent despair,
\textbf{focusing on the person's emotional state
rather than the cleanliness of the windows};
it evokes feelings of sympathy or concern,
not a desire to clean.
The cluttered environment actually
\textbf{detracts from highlighting the window's condition},
making the message ``Clean your home's windows''
feel irrelevant and insensitive given the broader context.

\vspace{0.7em}
\noindent
\textbf{Revised rationale.}

The image depicts a scene of overwhelming mess and
apparent despair;
\textbf{although the large, visibly filthy windows
dominate the room and numerous cleaning supplies
are scattered around},
the mood centers on burnout and paralysis,
evoking sympathy or concern rather than motivation
to clean.
The clutter and disarray, combined with the seated
person holding their head, frame the dirty windows
as part of a broader crisis, making the message
``Clean your home's windows'' feel inadequate and
insensitive rather than persuasive.

\end{minipage}
\end{table}

\paragraph{Discussion.}
Overall, the matched-budget experiments show that
dual-axis supervision achieves higher balanced
accuracy than the strongest single-axis setting
across all four backbones when trained on the
original teacher-generated rationales.
This advantage remains on three of the four
backbones after rationale quality control,
although the magnitude varies.

At the same time, rationale quality control does
not consistently improve downstream classification
performance.
The original teacher-generated rationales may retain
useful predictive information despite their
imperfections, while revisions that improve visual
grounding may also alter the evidence emphasized
during training.

These findings suggest that rationale quality and
classification effectiveness should be considered
separately.
In particular, improving rationale supervision
may require not only correcting unsupported visual
claims but also preserving how the evidence
supports the overall persuasiveness judgment.

Finally, because both original and revised rationales
remain conditioned on ground-truth labels,
these experiments do not directly rule out
post-hoc justification.
Rather, they provide additional evidence on
the effects of reasoning perspective diversity
and rationale quality control on downstream
classification performance.

\section{Theoretical and Empirical Analysis of Reasoning-Perspective Diversity}
\label{app:reasoning-perspective-diversity}

\subsection{Coverage-Based Motivation for Dual-Axis Rationale Supervision}
\label{sec:theoretical-justification}

The preceding experiments suggest that training on rationales from multiple prompt types is more effective than relying on a single source. We provide a simple coverage-based motivation for this effect: even when all rationales support the same target label, different prompt types may expose complementary valid reasoning paths.

For an image--message input \(x\) with target label \(y^\star(x)\), let \(\mathcal C(x)\) denote the set of valid rationales that support \(y^\star(x)\). Different prompt types are therefore treated as alternative valid reasoning paths rather than contradictory supervision. If a selected rationale set covers this space well, low training loss on the selected rationales also constrains the loss on nearby valid reasoning paths. Let \(S(x)\subseteq \mathcal C(x)\) be the selected training rationale set, and let \(\psi:\mathcal C(x)\to\mathbb R^d\) be a fixed embedding map. We define its coverage radius as
\[
\rho(S;x)
=
\sup_{r\in\mathcal C(x)}
\min_{s\in S(x)}
\|\psi(r)-\psi(s)\|_2.
\]
Let \(\theta\) denote the model parameters, and let \(\ell(\theta;x,r)\) denote the training loss when rationale \(r\), together with the target label, is used as supervision for input \(x\). The following theorem formalizes the coverage intuition under an idealized Lipschitz assumption.

\begin{theorem}[Rationale coverage bound]
\label{thm:coverage-main}
Assume that \(\ell(\theta;x,r)\) is \(L\)-Lipschitz in the rationale embedding for some constant \(L>0\), i.e.,
\[
|\ell(\theta;x,r)-\ell(\theta;x,r')|
\le
L\|\psi(r)-\psi(r')\|_2
\quad
\text{for all } r,r'\in\mathcal C(x).
\]
Then, for any selected rationale set \(S(x)\subseteq\mathcal C(x)\),
\[
\sup_{r\in\mathcal C(x)}
\ell(\theta;x,r)
\le
\max_{s\in S(x)}
\ell(\theta;x,s)
+
L\rho(S;x).
\]
\end{theorem}

Theorem~\ref{thm:coverage-main} provides a formal perspective on why complementary rationale supervision can help: with comparable empirical fit, broader coverage of valid reasoning paths yields a tighter worst-case loss bound. Our dual-axis design follows this principle by varying evidence polarity and visual granularity. 

We further analyze why rationale supervision from complementary reasoning perspectives can be useful. Building on the coverage view above, we first provide the proof of Theorem~\ref{thm:coverage-main} and empirically test whether the proposed dual-axis design provides broader coverage under matched rationale budgets.
We then consider two complementary perspectives---spectral conditioning and
redundancy reduction---and report additional diagnostics of the resulting
rationale-source family.

\paragraph{Valid rationale space.}
For each image--message input \(x\), let \(y^\star(x)\in\mathcal{Y}\) denote
the target persuasiveness label, and let \(\mathcal C(x)\) denote the latent set
of valid rationales supporting \(y^\star(x)\). All rationales considered for the
same input are assumed to support the same target label. Thus, the relevant
diversity arises from alternative valid reasoning paths rather than
contradictory supervision.

\paragraph{Rationale-source family.}
For each input \(x\), we consider a family of rationale-generation sources
\[
\mathcal F(x)=\{R_1(x),\ldots,R_K(x)\},
\]
where each \(R_k(x)\subseteq\mathcal C(x)\) corresponds to a prompt type or
reasoning perspective. Let
\[
\mathcal U(x)
=
\bigcup_{k=1}^{K} R_k(x)
\]
denote the generated candidate pool. For a fixed rationale budget \(B\), we
compare source compositions while holding the number of selected rationales
fixed.

\paragraph{Rationale embedding.}
Let
\[
\psi(r)\in\mathbb R^d
\]
denote a fixed embedding of rationale \(r\). Since embedding-based diversity
measures may depend on the representation space, we evaluate the empirical
proxies under multiple embedding backends.

\subsection{Coverage Diversity}
\label{app:coverage_diversity}

Coverage measures how well a selected rationale set represents alternative
valid reasoning paths. For a selected set \(S(x)\subseteq\mathcal C(x)\), the
coverage radius introduced in the main text is
\[
\rho(S;x)
=
\sup_{r\in\mathcal C(x)}
\min_{s\in S(x)}
\|\psi(r)-\psi(s)\|_2.
\]
A smaller radius means that every valid rationale is closer to at least one
selected rationale.

\paragraph{Proof of Theorem~\ref{thm:coverage-main}.}
Fix any \(r\in\mathcal C(x)\). By the definition of \(\rho(S;x)\), there
exists \(s_r\in S(x)\) such that
\[
\|\psi(r)-\psi(s_r)\|_2
\le
\rho(S;x).
\]
Under the Lipschitz assumption,
\[
\ell(\theta;x,r)
\le
\ell(\theta;x,s_r)
+
L\|\psi(r)-\psi(s_r)\|_2
\le
\max_{s\in S(x)}
\ell(\theta;x,s)
+
L\rho(S;x).
\]
Taking the supremum over \(r\in\mathcal C(x)\) yields
\[
\sup_{r\in\mathcal C(x)}
\ell(\theta;x,r)
\le
\max_{s\in S(x)}
\ell(\theta;x,s)
+
L\rho(S;x),
\]
which proves the claim.

\paragraph{Empirical coverage proxies.}
The full valid-rationale space \(\mathcal C(x)\) is unobserved, so we
approximate it using the generated candidate pool \(\mathcal U(x)\). For a
selected rationale set \(S(x)\), we define
\begin{align}
\rho_{\mathrm{avg}}(S;x)
&=
\frac{1}{|\mathcal U(x)|}
\sum_{u\in\mathcal U(x)}
\min_{s\in S(x)}
d\!\left(\psi(u),\psi(s)\right),
\\
\rho_{\max}(S;x)
&=
\max_{u\in\mathcal U(x)}
\min_{s\in S(x)}
d\!\left(\psi(u),\psi(s)\right),
\end{align}
where \(d\) denotes the embedding distance used in each analysis. Smaller
values indicate better coverage of the generated rationale space.

For normalized embeddings, cosine distance is monotonic in Euclidean distance,
since
\[
\|\psi(u)-\psi(s)\|_2^2
=
2\bigl(1-\cos(\psi(u),\psi(s))\bigr).
\]
Thus, the cosine-distance proxy used below preserves the ordering induced by
Euclidean distance in the coverage bound.

\paragraph{Matched-budget dual-axis coverage.}
We first test the prediction most directly relevant to our rationale design:
whether combining the two axes provides broader coverage than selecting
rationales restricted to a single axis. For each input, we compare same-axis
and dual-axis source subsets under the same rationale budget \(k\), using
average nearest-subset cosine distance to the full teacher-rationale pool as
the coverage measure. For each \(k\), we average the coverage distance over all eligible same-axis
and dual-axis source subsets and then average across inputs.

\begin{table}[t]
\centering
\small
\caption{
Matched-budget coverage of same-axis and dual-axis rationale subsets.
Coverage is the average nearest-subset cosine distance to the teacher-rationale
candidate pool; lower is better. Relative reduction is measured against the
same-axis subset.
}
\label{tab:dual_axis_coverage}
\begin{tabular}{lcccc}
\\
\toprule
\textbf{Generator} & \(k\) & \textbf{Same-axis} &
\textbf{Dual-axis} & \textbf{Reduction} \\
\midrule
Qwen & 2 & 0.137 & 0.127 & 7.4\% \\
Qwen & 3 & 0.134 & 0.117 & 13.1\% \\
Phi  & 2 & 0.158 & 0.134 & 15.0\% \\
Phi  & 3 & 0.154 & 0.121 & 21.3\% \\
\bottomrule
\end{tabular}
\end{table}

Dual-axis subsets achieve lower coverage distance for both teacher families
and both budgets, reducing the distance by 7.4--13.1\% for Qwen and
15.0--21.3\% for Phi. All differences are statistically significant under
10,000 paired bootstrap resamples. At \(k=4\), dual-axis subsets also
outperform all non-dual subsets, with relative reductions of 7.5\% for Qwen
and 10.2\% for Phi.

We further examine whether this coverage measure is related to downstream
performance. Across single-prompt rationale sources, lower coverage distance
is strongly associated with higher balanced accuracy after fine-tuning
(Spearman \(\rho=0.955\) for Qwen and \(\rho=0.786\) for Phi), and the
relationship remains robust to rationale length. Together, these results
support the coverage hypothesis underlying our dual-axis design: combining
evidence polarity and visual granularity exposes complementary reasoning paths
that cover the generated rationale space more effectively.

\paragraph{Coverage scaling with rationale budget.}
As a complementary sanity check, we examine how coverage changes as the number
of sampled rationale sources increases. For each input \(x\), the random
baseline uniformly samples \(B\) sources from the available rationale family,
and coverage is measured against the common candidate pool of all available
rationales for that input.

\begin{table*}[t]
\centering
\small
\caption{
Random-baseline coverage as a function of budget \(B\), averaged over 10 draws
per input. Coverage is computed against the common candidate pool of all
available Phi and Qwen rationale sources for the same input. Smaller values
indicate better coverage.
}
\label{tab:rand_by_B_scaling}
\resizebox{\textwidth}{!}{%
\begin{tabular}{llccccc|ccccc}
\\
\toprule
\textbf{Embedding} & \textbf{Split} &
\multicolumn{5}{c|}{\(\rho_{\mathrm{avg}}\downarrow\) for \texttt{rand}\((B)\)} &
\multicolumn{5}{c}{\(\rho_{\max}\downarrow\) for \texttt{rand}\((B)\)} \\
\cmidrule(lr){3-7}\cmidrule(lr){8-12}
& & \(B{=}1\) & \(B{=}2\) & \(B{=}3\) & \(B{=}4\) & \(B{=}5\)
& \(B{=}1\) & \(B{=}2\) & \(B{=}3\) & \(B{=}4\) & \(B{=}5\) \\
\midrule
BGE-M3 & Phi
& 0.0976 & 0.0730 & 0.0605 & 0.0512 & 0.0436
& 0.1669 & 0.1416 & 0.1294 & 0.1209 & 0.1144 \\
BGE-M3 & Qwen
& 0.0934 & 0.0727 & 0.0605 & 0.0515 & 0.0441
& 0.1657 & 0.1485 & 0.1402 & 0.1345 & 0.1299 \\
\addlinespace
E5 & Phi
& 0.0616 & 0.0460 & 0.0378 & 0.0318 & 0.0269
& 0.1083 & 0.0924 & 0.0849 & 0.0803 & 0.0769 \\
E5 & Qwen
& 0.0596 & 0.0451 & 0.0373 & 0.0314 & 0.0267
& 0.1041 & 0.0907 & 0.0843 & 0.0801 & 0.0767 \\
\addlinespace
MXBAI & Phi
& 0.1138 & 0.0839 & 0.0692 & 0.0587 & 0.0503
& 0.2119 & 0.1769 & 0.1613 & 0.1515 & 0.1444 \\
MXBAI & Qwen
& 0.1101 & 0.0824 & 0.0683 & 0.0584 & 0.0503
& 0.2047 & 0.1727 & 0.1591 & 0.1504 & 0.1444 \\
\bottomrule
\end{tabular}%
}
\end{table*}

Increasing \(B\) consistently reduces both average and worst-case coverage
distance across embedding backends and teacher families. This confirms the
expected monotonic behavior of the coverage proxy and complements the
matched-budget analysis above, which isolates the effect of source composition
rather than rationale count alone.

\subsection{Spectral Conditioning Diversity}
\label{app:spectral_diversity}

Coverage describes whether selected rationales reach different parts of the
rationale space, but does not characterize whether they span multiple
independent directions or remain concentrated along a narrow semantic axis.
To capture this distinction, let \(S(x)=\{r_1,\ldots,r_m\}\), define
\(z_i=\psi(r_i)\), and let
\[
\Sigma_{S(x)}
=
\frac{1}{m}
\sum_{i=1}^{m}
(z_i-\bar z)(z_i-\bar z)^\top,
\qquad
\bar z
=
\frac{1}{m}\sum_{i=1}^{m}z_i.
\]

\begin{proposition}[Spectral diversity and ridge-stabilized local estimation]
\label{prop:rationale_spectral_conditioning}
Assume a local linearization of the fine-tuning objective around a reference
parameter \(\theta_0\), so that the rationale-conditioned prediction is
approximated by
\[
f_\theta(x,r)
\approx
f_{\theta_0}(x,r)
+
h(x,r)^\top(\theta-\theta_0),
\]
where \(h(x,r)\in\mathbb R^p\) is a local feature or Jacobian vector. Consider
the ridge-regularized local objective over
\(S(x)=\{r_1,\ldots,r_m\}\):
\[
\hat\theta
=
\arg\min_\theta
\sum_{i=1}^m
\bigl(y_i-h(x,r_i)^\top\theta\bigr)^2
+
\lambda\|\theta\|_2^2,
\qquad
\lambda>0.
\]
Let \(H\in\mathbb R^{m\times p}\) be the design matrix with rows
\(h(x,r_i)^\top\). If
\[
y_i
=
h(x,r_i)^\top\theta^\star+\varepsilon_i,
\qquad
\mathbb E[\varepsilon_i]=0,
\qquad
\mathrm{Var}(\varepsilon_i)=\sigma^2,
\]
then the variance component of the ridge estimation error satisfies
\[
\sigma^2
\mathrm{tr}
\!\left(
(H^\top H+\lambda I)^{-2}H^\top H
\right)
\le
\sigma^2
\mathrm{tr}
\!\left(
(H^\top H+\lambda I)^{-1}
\right).
\]
The full mean-squared error satisfies
\[
\mathbb E\|\hat\theta-\theta^\star\|_2^2
\le
\sigma^2
\mathrm{tr}
\!\left(
(H^\top H+\lambda I)^{-1}
\right)
+
\lambda^2\|\theta^\star\|_2^2
\left\|
(H^\top H+\lambda I)^{-1}
\right\|_2^2.
\]
Thus, a rationale set whose local design matrix has a richer and less
collapsed spectrum yields a better-conditioned local estimation problem.
\end{proposition}

\begin{proof}
The ridge estimator satisfies
\[
\hat\theta-\theta^\star
=
(H^\top H+\lambda I)^{-1}
(H^\top\varepsilon-\lambda\theta^\star),
\]
where
\(\varepsilon=(\varepsilon_1,\ldots,\varepsilon_m)^\top\).
Using
\(\mathbb E[\varepsilon\varepsilon^\top]=\sigma^2I_m\),
the variance term is
\[
\sigma^2
\mathrm{tr}
\!\left(
(H^\top H+\lambda I)^{-2}H^\top H
\right).
\]
Since
\[
H^\top H
\preceq
H^\top H+\lambda I,
\]
we have
\[
(H^\top H+\lambda I)^{-1}
H^\top H
(H^\top H+\lambda I)^{-1}
\preceq
(H^\top H+\lambda I)^{-1}.
\]
Taking traces gives the stated variance bound. The bias term satisfies
\[
\lambda^2
\theta^{\star\top}
(H^\top H+\lambda I)^{-2}
\theta^\star
\le
\lambda^2
\|\theta^\star\|_2^2
\left\|
(H^\top H+\lambda I)^{-1}
\right\|_2^2.
\]
Combining the variance and bias terms proves the result.
\end{proof}

\begin{remark}
The proposition is stated in terms of local Jacobian directions \(h(x,r)\),
whereas our empirical diagnostics use external rationale embeddings
\(\psi(r)\). We therefore use embedding covariance only as a gradient-free
proxy for diversity in local supervision directions.
\end{remark}

\paragraph{Spectral proxies.}
We measure the spectrum of \(\Sigma_{S(x)}\) using
\begin{align}
\mathrm{erank}(\Sigma_{S(x)})
&=
\frac{\mathrm{tr}(\Sigma_{S(x)})}
{\|\Sigma_{S(x)}\|_2},
\\
D_{\log\det}(S(x))
&=
\log\det(I+\alpha\Sigma_{S(x)}),
\qquad
\alpha>0,
\end{align}
where larger values indicate greater multi-directional or volumetric spread.
We additionally report
\begin{align}
\mathrm{PR}(\Sigma_{S(x)})
&=
\frac{
\bigl(\mathrm{tr}(\Sigma_{S(x)})\bigr)^2
}{
\mathrm{tr}(\Sigma_{S(x)}^2)
},
\\
A(S(x))
&=
\frac{
\lambda_{\max}(\Sigma_{S(x)})
}{
\mathrm{tr}(\Sigma_{S(x)})
},
\end{align}
where smaller \(A\) indicates less directional concentration.

\subsection{Redundancy Reduction}
\label{app:redundancy_reduction}

A rationale family may provide broad coverage while still containing
substantial redundancy, for example when multiple rationales are paraphrases
of the same reasoning pattern. We therefore consider a complementary
variance-reduction view.

\begin{proposition}[Variance reduction from low-redundancy supervision]
\label{prop:rationale_variance_reduction}
Suppose each rationale provides a noisy supervision signal
\[
z_r(x)
=
f^\star(x)+\varepsilon_r(x),
\qquad
\mathbb E[\varepsilon_r(x)\mid x]=0,
\qquad
\mathrm{Var}(\varepsilon_r(x)\mid x)=\sigma^2(x).
\]
Let \(S(x)=\{r_1,\ldots,r_m\}\) and define
\[
\bar z(x)
=
\frac{1}{m}
\sum_{i=1}^{m}z_{r_i}(x).
\]
If the average pairwise noise correlation is
\[
\bar\rho(x)
=
\frac{2}{m(m-1)}
\sum_{1\le i<j\le m}
\mathrm{Corr}
\left(
\varepsilon_{r_i}(x),
\varepsilon_{r_j}(x)
\mid x
\right),
\]
then
\[
\mathrm{Var}(\bar z(x)\mid x)
=
\frac{\sigma^2(x)}{m}
\left(
1+(m-1)\bar\rho(x)
\right).
\]
Therefore, under a matched rationale budget, a less redundant supervision set
can yield lower effective supervision variance than a highly correlated set.
\end{proposition}

\begin{proof}
Using the variance of an average,
\[
\mathrm{Var}(\bar z(x)\mid x)
=
\frac{1}{m^2}
\sum_{i=1}^{m}
\mathrm{Var}(z_{r_i}(x)\mid x)
+
\frac{2}{m^2}
\sum_{1\le i<j\le m}
\mathrm{Cov}
\left(
z_{r_i}(x),
z_{r_j}(x)
\mid x
\right).
\]
Each variance term equals \(\sigma^2(x)\), while each covariance term is
\[
\mathrm{Corr}
\left(
\varepsilon_{r_i}(x),
\varepsilon_{r_j}(x)
\mid x
\right)
\sigma^2(x).
\]
Substituting these terms and simplifying yields
\[
\mathrm{Var}(\bar z(x)\mid x)
=
\frac{\sigma^2(x)}{m}
\left(
1+(m-1)\bar\rho(x)
\right).
\]
\end{proof}

\begin{remark}
In visual persuasion, rationales can focus on global composition, local visual
elements, or supporting and weakening evidence. These rationales may share the
same target label while relying on different evidence channels, making reduced
redundancy a plausible benefit of complementary supervision.
\end{remark}

\paragraph{Redundancy proxies.}
For \(S(x)=\{r_1,\ldots,r_m\}\), we use
\begin{align}
D_{\mathrm{pair}}(S(x))
&=
\frac{2}{m(m-1)}
\sum_{1\le i<j\le m}
\|\psi(r_i)-\psi(r_j)\|_2^2,
\\
\mathrm{Sim}_{\mathrm{avg}}(S(x))
&=
\frac{2}{m(m-1)}
\sum_{1\le i<j\le m}
\cos(\psi(r_i),\psi(r_j)).
\end{align}
Larger \(D_{\mathrm{pair}}\) and smaller
\(\mathrm{Sim}_{\mathrm{avg}}\) indicate lower redundancy.

\subsection{Empirical Diagnostics of Rationale Sources}
\label{app:empirical_rationale_diversity}

We next examine whether the prompt-specific rationale sources exhibit
systematic structure beyond the matched-budget coverage results above. We
evaluate source distinctness, relationships among diversity proxies, pairwise
source structure, and robustness across embedding backends.

\paragraph{Proxy summary.}
Table~\ref{tab:rationale_proxy_meanings} summarizes the diversity proxies used
in the empirical analysis.

\begin{table}[t]
\centering
\small
\caption{Rationale-source diversity proxies used in the empirical analysis.}
\label{tab:rationale_proxy_meanings}
\begin{tabular}{lll}
\\
\toprule
\textbf{Proxy} & \textbf{Meaning} & \textbf{Preferred direction} \\
\midrule
\(\mathrm{Sim}_{\mathrm{avg}}\)
& average pairwise cosine similarity
& smaller = less redundant \\
\(\mathrm{cos\_dist}_{\mathrm{avg}}\)
& average pairwise cosine distance
& larger = less redundant \\
\(\mathrm{erank}\)
& effective rank of covariance
& larger = more multi-directional \\
\(\mathrm{logdet}\)
& volumetric spread of covariance
& larger = more volumetric \\
\(\mathrm{anisotropy}\)
& directional concentration
& smaller = more isotropic \\
\(\rho_{\mathrm{avg}}\)
& average coverage distance
& smaller = better coverage \\
\(\rho_{\max}\)
& worst-case coverage distance
& smaller = better coverage \\
\(\mathrm{near\mbox{-}dup}\)
& near-duplicate rate
& smaller = less redundant \\
\bottomrule
\end{tabular}
\end{table}

\paragraph{Source distinctness.}
For each input \(x\), let \(r_k(x)\) be the rationale produced by source \(k\),
and let
\[
e_k(x)=\psi(r_k(x)).
\]
To remove input-specific effects, we residualize each source embedding within
the input:
\[
\tilde e_k(x)
=
e_k(x)
-
\frac{1}{K_x}
\sum_{j=1}^{K_x}e_j(x),
\]
where \(K_x\) is the number of available sources for input \(x\). We then run
PERMANOVA on \(\tilde e_k(x)\), restricting permutations within each input
block.

\begin{table}[t]
\centering
\small
\caption{
Global PERMANOVA after input-wise mean removal and restricted within-input
permutations. In all cases, \(p=0.005\) under 199 permutations. \(N\) denotes
the number of residual samples and \(K\) the number of co-occurring rationale
sources.
}
\label{tab:permanova_global_summary}
\begin{tabular}{llrrrr}
\\
\toprule
\textbf{Embedding} & \textbf{Generator} & \(N\) & \(K\) & \(F\) & \(R^2\) \\
\midrule
E5 & Qwen & 5670 & 7 & 1235.45 & 0.5669 \\
E5 & Phi  & 5738 & 7 & 1216.27 & 0.5601 \\
\midrule
BGE-M3 & Qwen & 5670 & 7 & 1089.71 & 0.5359 \\
BGE-M3 & Phi  & 5738 & 7 & 1004.06 & 0.5125 \\
\midrule
MXBAI & Qwen & 5670 & 7 & 644.29 & 0.4057 \\
MXBAI & Phi  & 5738 & 7 & 560.99 & 0.3700 \\
\bottomrule
\end{tabular}
\end{table}

Prompt-source identity explains a substantial portion of the residual
embedding variation after controlling for input identity, indicating that the
seven rationale sources are not simply interchangeable paraphrases.

\paragraph{Relationships among diversity proxies.}
For two image-level proxy functions \(M_1\) and \(M_2\), define
\[
\rho(M_1,M_2)
=
\mathrm{Spearman}
\left(
\{M_1(S(x))\}_{x\in\mathcal D},
\{M_2(S(x))\}_{x\in\mathcal D}
\right).
\]
Under BGE-M3 embeddings,
\[
\rho(\mathrm{erank},\rho_{\mathrm{avg}})
\approx
-0.006
\quad\text{on Phi},
\qquad
\rho(\mathrm{erank},\rho_{\mathrm{avg}})
\approx
-0.034
\quad\text{on Qwen},
\]
while
\[
\rho(\mathrm{erank},\mathrm{Sim}_{\mathrm{avg}})
\approx
0.303
\quad\text{on Phi},
\qquad
\rho(\mathrm{erank},\mathrm{Sim}_{\mathrm{avg}})
\approx
0.166
\quad\text{on Qwen}.
\]
These weak-to-moderate relationships indicate that coverage, spectral spread,
and redundancy capture related but non-equivalent properties of rationale
diversity.

\paragraph{Pairwise source structure.}
For each source pair \((k,\ell)\), define
\[
C_{k\ell}
=
\frac{1}{|\mathcal D_{k\ell}|}
\sum_{x\in\mathcal D_{k\ell}}
\cos
\left(
\psi(r_k(x)),
\psi(r_\ell(x))
\right),
\qquad
D_{k\ell}=1-C_{k\ell}.
\]
We also compute the near-duplicate rate
\[
N_{k\ell}^{(\tau)}
=
\frac{1}{|\mathcal D_{k\ell}|}
\sum_{x\in\mathcal D_{k\ell}}
\mathbf 1
\left[
\cos
\left(
\psi(r_k(x)),
\psi(r_\ell(x))
\right)
\ge\tau
\right],
\]
with \(\tau=0.95\). These quantities define the source-pair matrices
\[
C=(C_{k\ell})_{k,\ell=1}^{K},
\qquad
D=(D_{k\ell})_{k,\ell=1}^{K},
\qquad
N^{(\tau)}=(N_{k\ell}^{(\tau)})_{k,\ell=1}^{K}.
\]
For example, in the Qwen rationale family, \texttt{long} and \texttt{short}
remain relatively close under multiple embedding backends, whereas
\texttt{balanced} and \texttt{visual\_element} are farther apart. This
indicates that the source family contains both partially redundant and
complementary rationale pairs.

\paragraph{Robustness across embedding backends.}
For embedding backends \(a\) and \(b\) and a proxy \(M\), define
\[
\rho_{a,b}(M)
=
\mathrm{Spearman}
\left(
\{M_a(S(x))\}_{x\in\mathcal D},
\{M_b(S(x))\}_{x\in\mathcal D}
\right).
\]
For pairwise source matrices, we compute Spearman correlation over the
upper-triangular entries:
\[
\rho_{a,b}^{\mathrm{pair}}(C)
=
\mathrm{Spearman}
\left(
\{C_{k\ell}^{(a)}:1\le k<\ell\le K\},
\{C_{k\ell}^{(b)}:1\le k<\ell\le K\}
\right).
\]

\begin{table}[t]
\centering
\small
\caption{
Embedding-backend robustness of \(\psi\)-based proxies using Spearman
correlations. Per-image correlations are computed over the common \(n=820\)
inputs; pairwise matrix correlations use the upper-triangular entries of the
source-pair matrices.
}
\label{tab:proxy_backend_spearman}
\begin{tabular}{llrrr}
\\
\toprule
\textbf{Split} & \textbf{Quantity}
& \(\rho\)(E5,BGE)
& \(\rho\)(E5,MXBAI)
& \(\rho\)(BGE,MXBAI) \\
\midrule
Qwen & per-image \(\mathrm{sim}_{\mathrm{avg}}\)
& 0.602 & 0.745 & 0.755 \\
Qwen & per-image \(\mathrm{logdet}\)
& 0.603 & 0.745 & 0.755 \\
Qwen & per-image \(\mathrm{erank}\)
& 0.540 & 0.583 & 0.527 \\
Qwen & per-image \(\mathrm{anisotropy}\)
& 0.540 & 0.583 & 0.527 \\
Qwen & pairwise \(\mathrm{cos\_mean}\) matrix
& 0.809 & 0.955 & 0.806 \\
Qwen & pairwise near-duplicate matrix
& 0.506 & 0.891 & 0.518 \\
\midrule
Phi & per-image \(\mathrm{sim}_{\mathrm{avg}}\)
& 0.609 & 0.727 & 0.813 \\
Phi & per-image \(\mathrm{logdet}\)
& 0.609 & 0.728 & 0.813 \\
Phi & per-image \(\mathrm{erank}\)
& 0.736 & 0.691 & 0.769 \\
Phi & per-image \(\mathrm{anisotropy}\)
& 0.736 & 0.691 & 0.769 \\
Phi & pairwise \(\mathrm{cos\_mean}\) matrix
& 0.747 & 0.892 & 0.831 \\
Phi & pairwise near-duplicate matrix
& 0.774 & 0.890 & 0.747 \\
\bottomrule
\end{tabular}
\end{table}

The pairwise cosine-similarity structures show moderate-to-strong agreement
across embedding backends. Per-image spectral proxies are somewhat more
backend-sensitive, especially for Qwen, but remain positively correlated.
Near-duplicate rates are less stable in some cases because they depend on the
fixed threshold \(\tau=0.95\).

\paragraph{Takeaway.}
Together, these analyses provide complementary support for dual-axis rationale
supervision. The coverage bound explains why broader coverage can make a
selected rationale set more informative, while the matched-budget experiments
show that combining evidence polarity and visual granularity indeed improves
coverage of the generated rationale space. The spectral and redundancy
analyses further show that prompt sources differ in their directional spread
and overlap rather than forming a collection of interchangeable paraphrases.
Overall, the theoretical and empirical results support the use of
label-consistent, complementary reasoning perspectives as a richer supervision
signal for visual persuasion.

\section{Details on Rationale Faithfulness Evaluation}
\label{app:detail-rationale}

\subsection{Ground Truth Dataset Construction}
To establish the evaluation methodology for \emph{groundedness} and \emph{consistency}, we sample 10\% of instances from the rationale training data. To establish ground-truth labels, three of the authors independently labeled each instance. The annotations exhibited substantial agreement on \emph{groundedness} (Fleiss' $\kappa = 0.680$, $88.9\%$ complete agreement) and near-perfect agreement on \emph{consistency} (Fleiss' $\kappa = 0.996$, $99.7\%$ complete agreement) ~\citep{landis1977measurement}. Given this level of agreement, we determine the ground-truth labels via majority vote among the three annotators.

\subsection{Comparison of Groundedness Candidates}
\label{app:groundedness_candidates}
Evaluating groundedness in the context of visual persuasion is inherently challenging. Unlike factual image captioning, persuasive rationales often mix objective visual elements with subjective interpretations like emotional tone or symbolism. To capture such nuance, we evaluated five distinct evaluation methods against not only the human majority vote but also the individual judgments of the three annotators.

The candidate methods included: 
\begin{enumerate} 
    \item \textbf{CLIP with text:} Measuring CLIP similarity between the rationale text and the image description provided in the PVP dataset. 
    \item \textbf{CLIP with image:}  Measuring CLIP similarity between the rationale text and the corresponding image. 
    \item \textbf{GPT-5 atomic facts:} Decomposing the rationale into discrete atomic facts and prompting GPT-5 to verify each strictly against the image. 
    \item \textbf{GPT-5 atomic facts (calibrated):} The same atomic facts approach, but explicitly calibrated using the validation set to select a threshold for ratio ($N_{yes} / N_{total}$) 
    \item \textbf{GPT-5 prompting:} A direct, zero-shot prompt asking GPT-5 to act as a judge of the rationale's groundedness in the provided image. 
\end{enumerate}

\begin{table}[h]
  \centering
  \caption{Comparison of groundedness evaluation methods.}
  \label{tab:groundedness_candidates}
  \resizebox{\textwidth}{!}{%
  \begin{tabular}{lccccccc}
  \\
    \toprule
    \textbf{Method} & \textbf{$\kappa$ (vs Majority)} & \textbf{$\kappa$ (vs U1)} & \textbf{$\kappa$ (vs U2)} & \textbf{$\kappa$ (vs U3)} & \textbf{Bal. Acc.} & \textbf{F1 (Yes)} & \textbf{F1 (No)} \\
    \midrule
    CLIP with image (> 0.3) & -0.0016 & 0.0647 & 0.0349 & -0.0192 & 0.4983 & 0.7414 & 0.1616 \\
    CLIP with text (> 0.6) & 0.0406 & 0.0426 & 0.0316 & 0.0516 & 0.5672 & 0.5211 & 0.2138 \\
    GPT-5 atomic facts & 0.1001 & 0.1424 & 0.1114 & 0.1141 & \textbf{0.6526} & 0.6058 & \textbf{0.2603} \\
    GPT-5 atomic facts (> 0.6) & 0.1252 & 0.1351 & 0.0977 & \textbf{0.1647} & 0.5563 & \textbf{0.9186} & 0.2051 \\
    \textbf{GPT-5 prompting} & \textbf{0.1451} & \textbf{0.1935} & \textbf{0.2105} & 0.1294 & 0.5711 & 0.9125 & 0.2326 \\
    \bottomrule
  \end{tabular}%
  }
\end{table}

Evaluating a model's ability to recognize ungrounded rationales, measured by the F1 Score (No), is arguably the most critical aspect of this metric, as it indicates the model's sensitivity to hallucinations. However, an aggressive rejection strategy can artificially inflate this score at the cost of rejecting valid interpretations.

The results demonstrate that traditional vision-language embedding distances (CLIP) perform poorly, yielding near-zero correlation with human judgments and exhibiting extreme imbalances between positive and negative class identification.

Decomposing the rationales into uncalibrated atomic facts achieved the highest $F_1(no)$ Score of $0.2603$ and balanced accuracy of $0.6526$. However, this method suffered a severe drop in the positive class $F_1(yes)$ Score of $0.6058$. This indicates that strict factual decomposition over-penalizes subjective interpretations, generating excessive false negatives by rejecting valid persuasive reasoning. When we explicitly calibrated the atomic facts prompt to account for this subjectivity, the $F_1(yes)$ Score recovered dramatically, but the ability to detect ungrounded rationales degraded, with the $F_1(no)$ Score dropping to $0.2051$.

Ultimately, \textbf{GPT-5 prompting} provided the most optimal trade-off and robust overall performance. By assessing the rationale holistically rather than breaking it into isolated components, it achieved a strong $F_1(no)$ Score of $0.2326$ while maintaining a high $F_1(yes)$ Score of $0.9125$. This balanced detection capability resulted in the highest overall agreement with the aggregated majority vote ($\kappa = 0.1451$) and the strongest alignment with individual human annotators ($\kappa$ of $0.2105$ with User 2 and $0.1935$ with User 1). 
Consequently, GPT-5 prompting was initially selected as the groundedness evaluator. We subsequently refined the evaluation protocol to improve its alignment with human judgments, as described below.

\paragraph{Improved Groundedness Judge.}
Our error analysis revealed that the original judge sometimes failed to distinguish visually supported claims from unsupported inferences. For example, it incorrectly accepted ``security personnel'' as grounded when the image showed only individuals in ordinary clothing with identification badges, without sufficient evidence of their roles.

To address this limitation, we refined the evaluation guidelines to distinguish observable visual attributes from inferred identities and added five few-shot examples covering ambiguous object states, contextual interpretations, and text in images. We also replaced GPT-5 with GPT-6-astra and revalidated the judge, excluding the few-shot examples from the validation set.

The improved judge achieved a balanced accuracy of 80.1\%, a 23.0 percentage-point improvement over the original 57.1\%. Cohen's $\kappa$ against the human majority vote increased from 0.145 to 0.418. Recall for grounded and ungrounded rationales reached 85.3\% and 75.0\%, respectively. Additionally, Fleiss' $\kappa$ computed across the three human annotators and the improved judge was 0.501. These results indicate substantially improved alignment with human groundedness judgments.

\paragraph{Additional Disagreement Analysis.}
To further examine the remaining disagreements between the improved groundedness judge and human annotations, we conducted an additional evaluation on a separate set of 300 image--rationale pairs, with no overlap with the original validation set. Three human annotators independently evaluated the groundedness of each rationale, and their judgments were aggregated using majority voting. Table~\ref{tab:groundedness_confusion} presents the confusion matrix between GPT-6-astra and the human majority votes.

\begin{table}[t]
    \centering
    \small
    \setlength{\tabcolsep}{10pt}
    \caption{Confusion matrix between the improved groundedness judge and human majority-vote annotations on 300 newly sampled examples.}
    \label{tab:groundedness_confusion}
    \begin{tabular}{lcc}
        \toprule
        & \textbf{Human PASS} & \textbf{Human FAIL} \\
        \midrule
        \textbf{LLM PASS} & 238 & 9 \\
        \textbf{LLM FAIL} & 33 & 20 \\
        \bottomrule
    \end{tabular}
\end{table}

Among the 42 disagreements, 33 (78.6\%) involved the LLM judge rejecting rationales accepted by the human majority, whereas only 9 involved the opposite case. This suggests that the LLM judge tends to apply a more conservative evaluation criterion. Furthermore, 238 of the 247 rationales accepted by the LLM judge (96.4\%) were also accepted by humans. These findings support the potential utility of the judge as a conservative screening tool for identifying grounded rationales, although it may also reject some valid rationales.

\paragraph{Evaluation Cost.}
We further compared the time and financial costs of human and automated groundedness evaluation. Based on Prolific's recommended hourly rate of USD~12, assuming three annotators each spend one minute evaluating a rationale, human evaluation costs approximately USD~0.60 per rationale. In contrast, the GPT-6-astra judge costs approximately USD~0.017 per rationale and requires two seconds on average. Under these assumptions, human evaluation is approximately 35 times more expensive than automated evaluation, excluding additional costs associated with annotator recruitment and training. These results highlight the practical value of automated groundedness evaluation as a scalable complement to human assessment.

\subsection{Comparison of Consistency Candidates}
\label{app:consistency_candidates}
To automate the evaluation of consistency and groundedness, we assessed several candidate models and prompting strategies. Our goal was to select the method that achieved the highest alignment with the human annotations. Alignment is primarily measured using Cohen's $\kappa$, alongside Balanced Accuracy and $F_1$ score.

For consistency, the task requires determining whether the persuasiveness label naturally follows from the provided rationale text. To select our evaluator, we compared an LLM-as-a-judge approach using GPT-5, Qwen2.5-VL-32B-Instruct, and Qwen2.5-VL-7B-Instruct against a dedicated Natural Language Inference (NLI) model (\texttt{nli-deberta-v3-base}). 

\begin{table}[h]
  \centering
  \caption{Comparison of consistency evaluation methods.}
  \label{tab:consistency_candidates_val}
  \begin{tabular}{lccc}
    \\
    \toprule
    \textbf{Model} & \textbf{Cohen's $\kappa$} & \textbf{Balanced Acc.} & \textbf{F1 Score} \\
    \midrule
    \textbf{GPT-5} & \textbf{0.992} & \textbf{0.996} & \textbf{0.996} \\
    Qwen2.5-32B & 0.914 & 0.958 & 0.958 \\
    Qwen2.5-7B & 0.973 & 0.986 & 0.986 \\
    NLI (DeBERTa) & 0.953 & 0.978 & 0.977 \\
    \bottomrule
  \end{tabular}
\end{table}

As shown in Table~\ref{tab:consistency_candidates_val}, while the NLI baseline and Qwen models performed well, GPT-5 achieved a near-perfect alignment with human annotators ($\kappa = 0.992$). Therefore, we selected GPT-5 as the primary LLM judge for evaluating consistency.



\subsection{Prompt Template for Consistency}
\begin{promptbox}
Given a reasoning text, infer whether it supports the conclusion that the image is persuasive.

### Instructions
- Use only the provided reasoning text.
- Do not rely on external knowledge or unstated assumptions.

### Output
Answer with exactly one word:
- yes
- no
\end{promptbox}

\subsection{Prompt Template for Groundedness}
\begin{promptbox}
You are an expert evaluator assessing the visual groundedness of a rationale for a persuasive image.

Your task is to determine whether the rationale's main claims and interpretations are reasonably supported by visual evidence in the image.

Procedure:
- Identify the central claims: each visual element the rationale relies on, what it says about that element (its state, quantity, position, role, text, or relation to other elements), and the interpretation it draws from it.
- Check each central claim against the image. The presence of a named object is not enough; what the rationale says about it must also hold.
- Separate visual claims from evaluative claims. Statements that the image focuses on, lacks, fails to show, distracts from, supports, or undermines the message, or how viewers may react, are evaluative. Do not assign No because you disagree with an evaluation or can find an arguable cue for the opposite view. A claim that something is absent is contradicted only when the specific thing said to be absent is unmistakably present.

Evaluation Guidelines:
1. Focus on whether the central visual claims in the rationale are supported by the image, rather than requiring every minor detail to be perfectly accurate.
2. A rationale does not need to describe only literal visual facts. Affective, symbolic, and compositional interpretations are considered grounded when they are reasonably supported by identifiable visual cues.
3. Do not penalize a rationale simply because its interpretation is subjective, implicit, or not the only possible interpretation of the image.
4. Allow minor imprecisions and generalizations that do not materially change the visual evidence or overall reasoning. This allowance does not apply to text written in the image.
5. Read text in the image exactly. If the rationale quotes or reports text written in the image, it must match what is actually written, including spelling. AI-generated images often contain misspelled or garbled text: if the rationale reports such text as though it were correctly written, that claim is not grounded. If the rationale itself accurately points out the misspelling or garbling, the claim is grounded. Statements about the message an image conveys are not claims about written text.
6. Assign Yes if every central claim and interpretation is reasonably supported by the image and none is contradicted.
7. Assign No if any central claim does one of the following, even when the rest of the rationale is accurate:
   a. It relies on a visual element that is not in the image.
   b. It misdescribes an element's state, quantity, position, role, or relation in a way that affects the reasoning, such as claiming a contrast between elements that look alike, or claiming something is absent when it is visible.
   c. It draws an interpretation that the visible elements do not actually support, even though those elements are present.
   d. It reports text that does not match what is written in the image.
8. Accept conventional visual shorthand and plausible descriptive attributes that the image does not contradict, such as symbols standing in for swearing or ordinary descriptors of a material or setting. A specific role or action still needs an identifiable visual cue.
9. Evaluate groundedness independently of whether the predicted message or final classification is correct.

Judge the rationale based on reasonable visual support, not strict literal verification. However, do not accept unsupported claims merely because they sound plausible or mention real objects in the image.

Return a JSON object only, with no markdown:
{"reason": "<one to three sentences stating which central claims are supported and which are not, naming the visual evidence>", "label": "Yes" or "No"}

===== Few-shot examples =====

--- Example 1 ---
[Image: Keep_your_valuables_close_to_your_body_while_traveling/17.jpg]
Rationale:
<reasoning>
Rationale 1:
The thief in the background: The presence of a thief in the background, dressed in a mask and appearing to be on the lookout for potential targets, immediately creates a sense of danger and urgency. This visual cue directly supports the message by illustrating the real threat travelers face if they do not keep their valuables close.
Rationale 2:
The traveler's body language and positioning of valuables: The traveler is depicted with his belongings securely placed around him, including a backpack on his back, a suitcase between his legs, and other bags within arm's reach. This demonstrates the recommended behavior of keeping valuables close to one's body, reinforcing the message through positive example.
Rationale 3:
The scattered and unattended bag: The open bag filled with various items left unattended on the ground serves as a stark contrast to the traveler's secure setup. It visually represents what can happen when valuables are not kept close, subtly warning viewers of the consequences of carelessness and further emphasizing the importance of the message.
</reasoning>

Answer:
{"reason": "The masked figure and the open bag are visible, but the traveler's secure setup is contradicted: the other bags sit on the floor away from his body and no suitcase is between his legs, so the claimed contrast with an unattended bag has no visual basis.", "label": "No"}

--- Example 2 ---
[Image: Participate_in_local_events_and_festivals/24.jpg]
Rationale:
<reasoning>
Image description: The image shows an empty street lined with tables and chairs, decorated with colorful banners, and a lone figure sitting on a curb. The scene appears deserted despite the setup for a community event or festival.
Persuasiveness: The image fails to be persuasive because it depicts an empty and inactive setting, which contradicts the idea of participation and community engagement. The absence of people and activity does not effectively encourage involvement in local events and festivals.
</reasoning>

Answer:
{"reason": "The image clearly shows an empty street with arranged tables and chairs, colorful hanging banners and a lone seated figure, supporting the claims of a deserted event setup and a lack of visible community participation.", "label": "Yes"}

--- Example 3 ---
[Image: Sell_unnecessary_items/9.jpg]
Rationale:
<reasoning>
Signage: The large, prominent sign reading "Turn your Clutter into Treasure!" directly communicates the core message of transforming unnecessary items into valuable goods, encouraging viewers to consider selling their own clutter.
</reasoning>

Answer:
{"reason": "The sign is visible, but it reads 'Turn your Clutter into esosel is Treasure!', not the quoted 'Turn your Clutter into Treasure!', so the rationale misreports the text its claim relies on.", "label": "No"}

--- Example 4 ---
[Image: Declare_agricultural_products_at_the_airport_when_bringing_them_from_abroad/24.jpg]
Rationale:
<reasoning>
The image fails to be persuasive because it uses an unrealistic and exaggerated scenario with insects crawling on luggage, which may not resonate with the typical traveler's experience or concerns. Additionally, the misspelling of "undeclared" and "agricultural" in the sign undermines the credibility and clarity of the message.
</reasoning>

Answer:
{"reason": "Insects are visibly crawling on the suitcase, and the overhead sign is misspelled as 'Unedelared agrricultural produts', which the rationale accurately points out as a misspelling.", "label": "Yes"}

--- Example 5 ---
[Image: Always_turn_off_the_stove_after_use_in_the_kitchen/16.jpg]
Rationale:
<reasoning>
Rationale 1:
Ghost figure: The ghostly figure reaching towards the stove creates a sense of urgency and danger, emphasizing the potential consequences of leaving the stove on. This supernatural element adds an eerie and memorable warning to the viewer.
Rationale 2:
Flames and smoke: The visible flames and smoke rising from the stove visually represent the immediate hazard of an unattended burning stove. This realistic depiction of fire danger directly links the action (not turning off the stove) with its perilous outcome.
Rationale 3:
Sign with message: The sign reading "NEGATIVE TURNING OFF THE STOVE THE RATCHET" serves as a direct reminder and reinforces the message through text. Its placement in the scene ensures that viewers cannot miss the explicit instruction, further underlining the importance of the safety practice.
</reasoning>

Answer:
{"reason": "The ghostly figure and the stove flame are visible, but the sign's text is garbled and does not read as quoted, so the reported wording and the claim that it gives an explicit instruction are not supported.", "label": "No"}
\end{promptbox}

\subsection{Human Validation for Image Editing Pipeline}
\label{app:sensitivity_human_validation}
Figure~\ref{fig:faithfulness} shows the annotation interface used to validate whether the image edits shift persuasiveness in the intended direction, a necessary condition for the sensitivity metric. We randomly sample 30 non-overlapping instances from each of three setups (Base-Reasoning, GRPO, and Reasoning-SFT) for student models, yielding 180 original--edited image pairs. We additionally include 20 control pairs with no meaningful difference in persuasiveness. Three annotators view each pair side by side and select which image is more persuasive for the given message (left, right, or ambiguous). An edit is considered successful when the judgment shifts counter to the model's predicted label: the edited image is preferred when the model predicts `unpersuasive', whereas the original image is preferred when it predicts `persuasive'. All 20 control pairs are correctly identified. Among the 180 edited pairs, five result in 1--1--1 ties; of the remaining 175 pairs, 131 (74.9\%) are majority-judged to shift persuasiveness in the intended direction, with substantial inter-annotator agreement (Fleiss' $\kappa = 0.693$).

The three annotators were volunteer participants who provided informed consent and were compensated at the minimum wage of the country where the annotation was conducted. The task involved viewing persuasive advertisement images and making comparative judgments, posing no more than minimal risk to participants. No personally identifiable information was collected or retained, and participants were informed of the persuasive nature of the stimulus images prior to participation. This study is exempt from IRB approval under 45 CFR 46.104(d)(2)(i).

\begin{figure}[tb]
    \centering
    \includegraphics[width=0.8\linewidth]{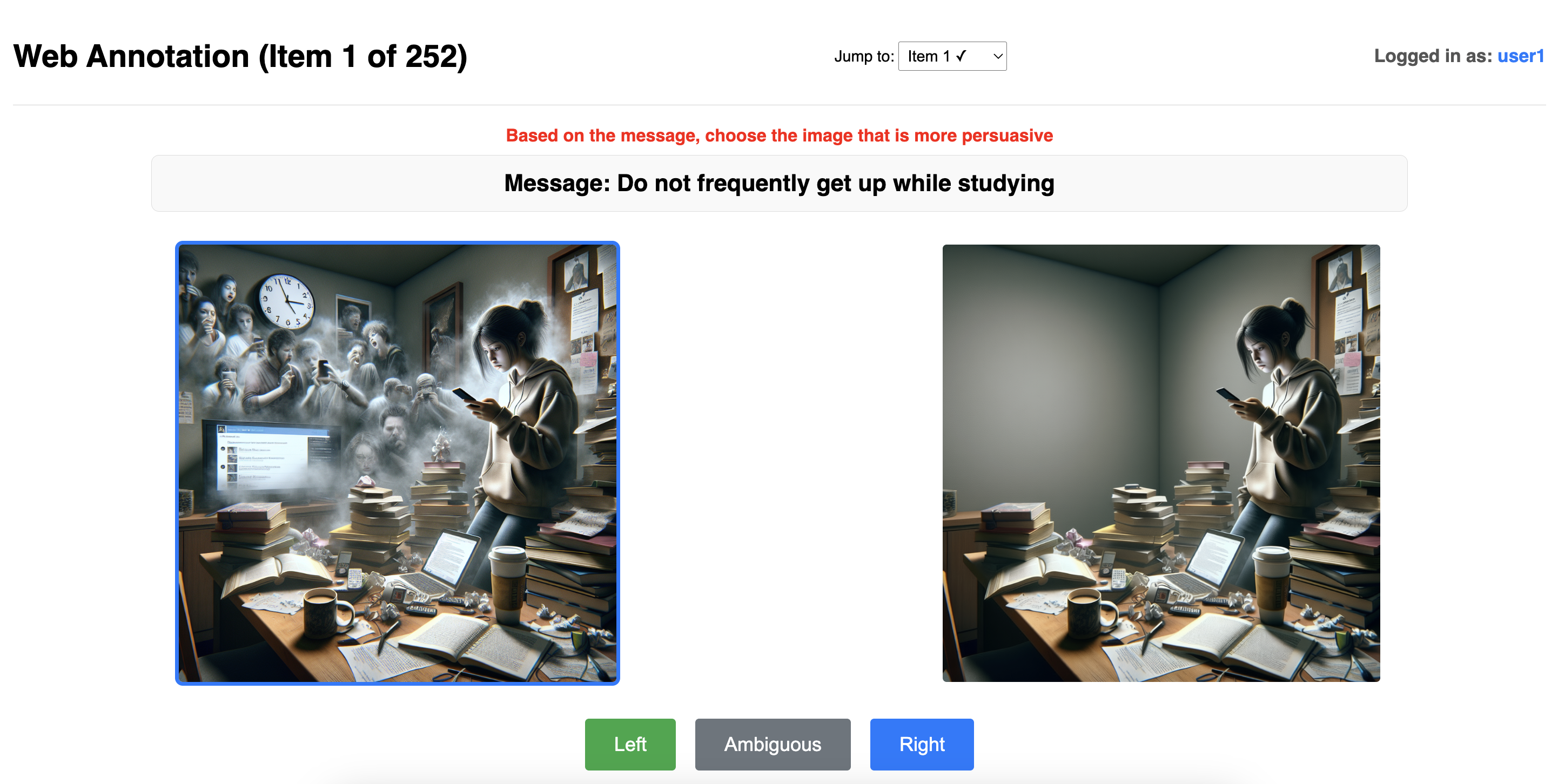}
    \caption{The annotation interface used for validating image editing pipeline.}
    \label{fig:faithfulness}
\end{figure}

\subsection{Cue-Specific Human Evaluation of Rationale Grounding}
\label{app:cue_grounding}

We conduct an additional human evaluation to examine whether visual evidence grounding can be assessed across different types of persuasive visual cues. While our existing rationale-to-image groundedness metric primarily verifies concrete visual claims, persuasive rationales may also involve affective, symbolic, and compositional cues that require further interpretation. We therefore evaluate both the presence of rationale-cited visual evidence and the appropriateness of the corresponding persuasive interpretation.

\subsubsection{Cue Categorization}

To examine groundedness across different types of visual cues, we use GPT-6-astra to categorize rationales according to the visual cues used as evidence in their persuasive reasoning. We consider four categories:

\begin{itemize}
    \item \textbf{Literal:} Factual visual cues, including explicitly depicted objects, people, actions, settings, and visible text.

    \item \textbf{Affective:} Emotions, moods, or atmospheres conveyed through facial expressions, body language, or the overall emotional tone of an image.

    \item \textbf{Symbolic:} Abstract meanings conveyed through metaphors, conventional symbols, or figurative representations.

    \item \textbf{Compositional:} Visual emphasis and structure established through contrast, placement, scale, focus, framing, color, or lighting.
\end{itemize}

Since a rationale may rely on multiple types of visual cues, we allow multiple categories to be assigned to the same rationale. For each identified category, GPT-6-astra extracts a supporting phrase from the rationale and identifies a primary category representing the cue type that contributes most strongly to the overall argument. Rationales without identifiable visual cues are assigned \texttt{none} as their primary category.

The categorization focuses on how visual cues are used in the rationale, rather than whether the corresponding claims are supported by the image. In particular, a cue is identified only when the rationale connects it to a concrete visual element, excluding generic evaluations without specific visual evidence. The complete categorization prompt is provided in Appendix~\ref{app:cue_prompt}.

\subsubsection{Human Evaluation Protocol}

We sample 100 rationales, balanced across the four primary cue categories (25 per category), and conduct a human evaluation with three independent annotators. Each annotator is presented with an image, its intended target message, and the corresponding rationale, and independently evaluates the following two questions:

\begin{itemize}
    \item \textbf{Q1. Visual Evidence Grounding:} Does the visual evidence mentioned in the rationale actually exist in the image?

    \item \textbf{Q2. Persuasive Interpretation Grounding:} Are both the visual evidence mentioned in the rationale and the resulting persuasive interpretation appropriate, given the image?
\end{itemize}

Q1 examines whether the rationale-cited visual evidence is supported by the image, whereas Q2 additionally considers whether the corresponding persuasive interpretation is appropriate. Both questions are evaluated using binary Yes/No judgments. We aggregate the three independent judgments using majority voting to obtain the final label for each question.

Figure~\ref{fig:cue_annotation_interface} shows the annotation interface used in this evaluation.

\begin{figure}[t]
    \centering
    \includegraphics[width=\linewidth]{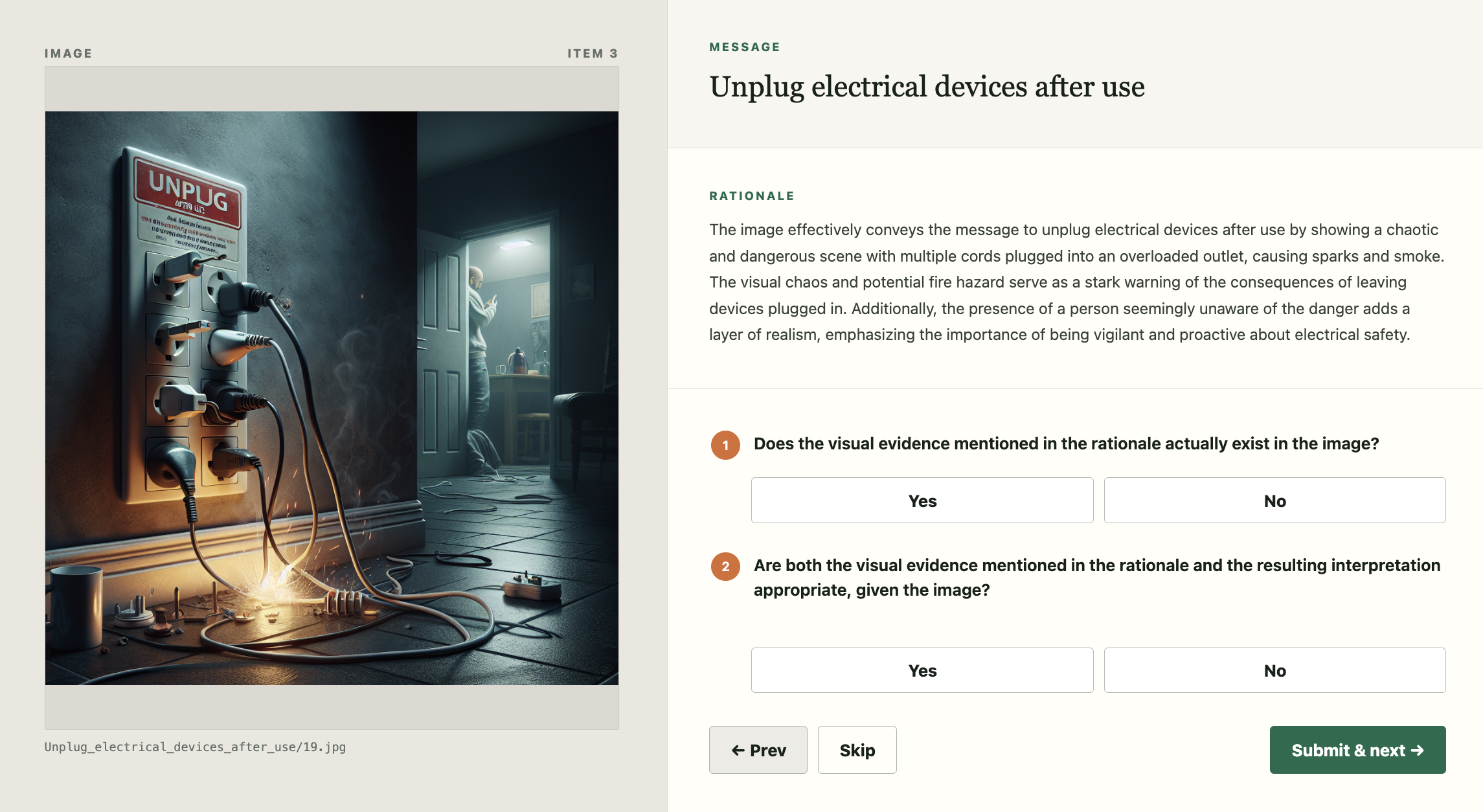}
    \caption{\textbf{Annotation interface for the cue-specific human evaluation.} Annotators were presented with an input image, its target message, and the corresponding rationale. They independently assessed visual evidence grounding (Q1) and persuasive interpretation grounding (Q2) using binary Yes/No judgments.}
    \label{fig:cue_annotation_interface}
\end{figure}

\subsubsection{Results and Discussion}

Table~\ref{tab:cue_grounding} reports the human evaluation results across the four primary visual cue categories.

\begin{table}[t]
    \centering
    \caption{\textbf{Cue-specific human evaluation results (\%).} Q1 evaluates visual evidence grounding, whereas Q2 additionally assesses the appropriateness of persuasive interpretations. Results are based on majority voting among three annotators.}
    \label{tab:cue_grounding}
    \begin{tabular}{lccc}
        \toprule
        \textbf{Cue} & \textbf{N} & \textbf{Q1} & \textbf{Q2} \\
        \midrule
        Literal       & 25 & 96.0  & 92.0 \\
        Affective     & 25 & 100.0 & 100.0 \\
        Symbolic      & 25 & 96.0  & 92.0 \\
        Compositional & 25 & 92.0  & 88.0 \\
        \midrule
        \textbf{Overall} & \textbf{100} & \textbf{96.0} & \textbf{93.0} \\
        \bottomrule
    \end{tabular}
\end{table}

\noindent\textbf{Cue-specific grounding results.}
The human evaluation yields majority-vote approval rates of 96\% for Q1 and 93\% for Q2. High approval rates are observed across all four cue categories, including non-literal visual cues. Affective cues achieve 100\% on both questions, while Symbolic cues achieve 96\% and 92\%, and Compositional cues achieve 92\% and 88\%, respectively.

These results suggest that the evaluated rationales are generally supported by visual evidence across diverse cue types, including affective, symbolic, and compositional elements, rather than only factual claims about depicted objects or actions. Moreover, the Q2 results provide additional human assessments of whether these visual cues appropriately support the corresponding persuasive interpretations, extending the analysis beyond the presence of visual evidence alone.

\noindent\textbf{Inter-annotator agreement.}
We additionally measure inter-annotator agreement using Fleiss' $\kappa$, obtaining $\kappa = 0.656$ for Q1 and $\kappa = 0.640$ for Q2. Both dimensions exhibit substantial agreement, with relatively similar $\kappa$ values. This suggests that assessing persuasive interpretation grounding did not result in substantially lower inter-annotator agreement than assessing visual evidence grounding in this evaluation.

\paragraph{Agreement between automated groundedness and human judgments.}
To examine whether the updated groundedness judge captures
information relevant to persuasive interpretations, we
evaluated GPT-6-astra on the same 100 image--rationale pairs
and compared its predictions against human majority-vote
judgments for Q1 and Q2.

The judge achieved 79.2\% balanced accuracy against visual
evidence grounding (Q1) and 78.2\% against persuasive
interpretation grounding (Q2). Despite Q2 additionally
requiring appropriate persuasive interpretation, the judge
maintained comparable agreement with human judgments.
These results suggest that our automated groundedness
evaluation remains informative for assessing persuasive
interpretations, although it does not directly verify
every interpretive claim.

\paragraph{Implications for groundedness evaluation.}
Visual persuasion reasoning involves diverse types of
evidence that extend beyond explicitly depicted objects
or actions. Our cue-specific human evaluation provides
additional evidence that rationale grounding can be
assessed across affective, symbolic, and compositional cues.
Moreover, the comparable agreement of our automated
groundedness judge with Q1 and Q2 human judgments suggests
that it provides information relevant to persuasive
interpretations, beyond verifying factual visual evidence
alone. Nevertheless, the metric does not fully establish
the appropriateness of all persuasive interpretations.

\subsubsection{Cue Categorization Prompt}
\label{app:cue_prompt}

We use the following system prompt for GPT-6-astra to categorize rationales according to the visual cues used in their persuasive reasoning. The model identifies all applicable categories, extracts supporting phrases, and assigns a primary category.

\begin{promptbox}
You analyze rationales that explain how a persuasive image conveys a message. Identify which kinds of visual cues the rationale uses as evidence for its argument.

Categories (mark every category the rationale actually uses; several or none may apply):

- literal: depicted objects, people, actions, settings, or visible text taken at face value, i.e. the image directly shows the behavior, its subject, its context, or its concrete outcome. Example: "the woman wears a trench coat and holds an umbrella on a wet street"; "a red car skids on a curve".

- affective: emotion, mood, or atmosphere: facial expressions or body language read as feelings, or a feeling the image evokes in the viewer (fear, joy, urgency, unease). Example: "her bowed head and slumped posture convey isolation"; "an eerie atmosphere of dread".

- symbolic: an element that stands for an abstract concept other than itself: metaphor, personification, conventional symbols, or surreal or exaggerated elements read figuratively. Example: "skull-and-crossbones signs symbolize danger"; "ghosts represent hidden problems in the house".

- compositional: how the image is arranged: contrast or juxtaposition, placement, scale, focus, framing, foreground and background, or color and lighting used for emphasis or structure. Example: "the stark contrast between the clean park and the trash-filled area"; "the hands are the focal point".

Rules:

- Count a cue only when the rationale ties it to a concrete visual element. Generic praise without a visual element ("powerful imagery", "visually striking") counts for nothing.

- Color or lighting said to evoke a feeling is affective; said to emphasize, separate, or direct attention it is compositional; it can be both.

- An element is literal when it is the behavior or subject itself, and symbolic when the rationale says it stands for something else. One element can support both if the rationale uses it both ways.

- Judge the rationale's reasoning, not whether its claims are true of the image.

- "primary" is the category that carries most of the argument, or "none" when no category applies.

Respond with JSON only:

{"literal": "<short quote from the rationale, or null>", "affective": "<quote or null>", "symbolic": "<quote or null>", "compositional": "<quote or null>", "primary": "literal|affective|symbolic|compositional|none"}
\end{promptbox}



\section{Additional Sensitivity Analyses}
\label{app:sensitivity_analysis}

\subsection{Image-editing control experiments}
\label{app:image-edit-controls}
 Sensitivity is intended to measure whether modifying the visual evidence referenced in a rationale affects the model's prediction. To verify that the measured sensitivity is not simply caused by unsuccessful edits or by any visible change to the image, we conduct two additional control experiments. First, we apply a random edit to each image without providing the rationale or target message to the editing model. Second, we apply an irrelevant but visible edit that is designed to preserve the original persuasiveness of the image.

\begin{table}[h]
    \centering
    \small
    \setlength{\tabcolsep}{4.5pt}
    \caption{
    Sensitivity under different image-editing controls (\%).
    }
    \label{tab:image-edit-controls}
    \begin{tabular}{llccc}
    \toprule
    \textbf{Family} & \textbf{Setup}
    & \textbf{Sensitivity}
    & \textbf{Random Edit}
    & \textbf{Irrelevant Edit} \\
    \midrule
    Qwen & Base-Reasoning
    & 82.1 & 52.2 & 53.1 \\
         & GRPO-Ans
    & 69.0 & 50.2 & 50.2 \\
         & Reasoning-SFT
    & 70.0 & 55.6 & 52.6 \\
    \midrule
    Phi & Base-Reasoning
    & 75.4 & 50.7 & 51.2 \\
        & GRPO-Ans
    & 53.8 & 16.3 & 46.4 \\
        & Reasoning-SFT
    & 66.4 & 48.5 & 51.7 \\
    \midrule
    Gemma & Base-Reasoning
    & 63.5 & 57.2 & 54.3 \\
          & GRPO-Ans
    & 56.4 & 36.7 & 43.9 \\
          & Reasoning-SFT
    & 69.3 & 48.2 & 41.4 \\
    \midrule
    LLaVA & Base-Reasoning
    & 79.5 & 53.6 & 58.9 \\
          & GRPO-Ans
    & 57.7 & 47.6 & 44.3 \\
          & Reasoning-SFT
    & 57.1 & 36.8 & 41.2 \\
    \midrule
    \textbf{Avg.} & -
    & \textbf{66.7} & \textbf{46.1} & \textbf{49.1} \\
    \bottomrule
    \end{tabular}
\end{table}

Across all evaluated setups, rationale-conditioned edits yield higher sensitivity than both random and irrelevant edits. On average, sensitivity is 66.7 for rationale-conditioned edits, compared with 46.1 and 49.1 for the two controls, respectively. This supports that the metric reflects changes to rationale-relevant visual evidence rather than a general response to image modification.

\subsection{Sensitivity to the Magnitude of Prediction Changes}
\label{app:sensitivity_threshold}

The original rationale-to-decision sensitivity metric measures the direction of prediction changes, assigning a score of 1 when $\Delta_{\mathrm{sens}} > 0$. However, this binary criterion does not account for the magnitude of the change.

To examine this limitation, we extend the original criterion to $\Delta_{\mathrm{sens}} > \tau$ and evaluate sensitivity at two thresholds, $\tau \in \{0, 0.1\}$.

\begin{table}[t]
\centering
\small
\setlength{\tabcolsep}{5pt}
\caption{Rationale-to-decision sensitivity (\%) under different thresholds $\tau$.}
\label{tab:sensitivity_threshold}
\begin{tabular}{llrrrr}
\toprule
Backbone & Setup & $\tau=0$ & $\tau=0.1$ \\
\midrule
Qwen
& Base-Reasoning & 82.1 & 80.9 & \\
& GRPO-Ans & 68.9 & 67.4 \\
& Reasoning-SFT & 70.0 & 67.1\\
\midrule
Phi
& Base-Reasoning & 75.4 & 71.9 \\
& GRPO-Ans & 53.8 & 53.2 \\
& Reasoning-SFT & 66.5 & 63.4 \\
\midrule
Gemma
& Base-Reasoning & 63.4 & 63.4\\
& GRPO-Ans & 56.4 & 56.4  \\
& Reasoning-SFT & 69.3 & 69.3 \\
\midrule
LLaVA
& Base-Reasoning & 79.5 & 77.9  \\
& GRPO-Ans & 57.7 & 57.1 \\
& Reasoning-SFT & 57.2 & 55.5  \\
\bottomrule
\end{tabular}
\end{table}

As shown in Table~\ref{tab:sensitivity_threshold}, increasing $\tau$ from 0 to 0.1 reduces sensitivity by at most 3.5 percentage points across all backbones and training setups. This indicates that the original sensitivity scores are largely preserved when small prediction changes are excluded, suggesting that they are not primarily driven by near-zero changes.



\subsection{Robustness Across Image-Editing Models}
\label{app:editor_robustness}

To examine the dependence of sensitivity on image-editing model capabilities, we conduct an additional experiment using FLUX.1 Kontext [dev] as an alternative to Qwen-Image-Edit-2509.

We select FLUX.1 Kontext [dev] based on Lin et al.\ (2025), where it achieved lower scores than Qwen-Image-Edit-2509 on ImgEdit-Bench (3.52 vs.\ 4.35) and GEdit-Bench (6.00 vs.\ 7.98). These benchmarks assess image-editing capabilities relevant to our evaluation pipeline, including object replacement, removal, and semantic consistency.

Using the same editing and evaluation pipeline, we re-evaluate sensitivity on Qwen and Gemma with FLUX.1 Kontext [dev].

\begin{table}[t]
\centering
\small
\setlength{\tabcolsep}{5pt}
\caption{Rationale-to-decision sensitivity (\%) across two image-editing models. Results are reported as mean $\pm$ standard deviation.}
\label{tab:editor_robustness}
\begin{tabular}{llcc}
\toprule
Student & Setup & Qwen-Image-Edit-2509 & FLUX.1 Kontext [dev] \\
\midrule
Qwen
& Base-Reasoning & $82.1 \pm 4.5$ & $74.9 \pm 1.2$ \\
& GRPO-Ans & $68.9 \pm 6.4$ & $65.7 \pm 2.7$ \\
& Reasoning-SFT & $70.0 \pm 4.4$ & $69.5 \pm 1.2$ \\
\midrule
Gemma
& Base-Reasoning & $63.4 \pm 1.6$ & $69.0 \pm 1.3$ \\
& GRPO-Ans & $56.4 \pm 2.0$ & $67.3 \pm 2.8$ \\
& Reasoning-SFT & $69.3 \pm 1.6$ & $70.9 \pm 1.6$ \\
\bottomrule
\end{tabular}
\end{table}

Although absolute sensitivity scores vary across editors, the relative ranking of the three methods remains unchanged for both backbones (Table~\ref{tab:editor_robustness}). Notably, Reasoning-SFT exhibits the smallest sensitivity changes when switching editors, with differences of only 0.5 and 1.6 percentage points on Qwen and Gemma, respectively.

These results suggest that while absolute sensitivity scores depend on the editing model, the relative trends are preserved across the two evaluated editors on Qwen and Gemma. Further evaluation with additional editors and backbones would be needed to establish broader robustness.

\section{Rationale Faithfulness Results}
\label{app:faithfulness_evaluation}

\subsection{Quantitative Results}
\label{app:quantitative_faithfulness}

Table~\ref{tab:faithfulness_results}  reports the exact faithfulness scores corresponding to 	Figure~\ref{fig:faithfulness_result}. For R-I Groundedness, Reasoning-SFT achieves the highest model-averaged score of 79.3\%, compared with 67.6\% for GRPO-Ans and 61.5\% for Base-Reasoning. The gains are particularly pronounced for Gemma and LLaVA, where Reasoning-SFT improves over Base-Reasoning by 28.0 and 23.6 percentage points, respectively. For R-D Consistency, Reasoning-SFT performs best across all four student models, reaching an average of 99.1\%.

R-D Sensitivity shows a different pattern. Base-Reasoning achieves the highest sensitivity across all four models, with an average of 75.1\%, while Reasoning-SFT obtains 65.7\% and GRPO-Ans 59.2\%. Thus, rationale supervision substantially improves groundedness and consistency, but does not uniformly improve sensitivity to interventions on rationale-relevant visual evidence, indicating a trade-off across the faithfulness dimensions.

\begin{table}[h]
    \centering
    \small
    \setlength{\tabcolsep}{4pt}
    \caption{
    Rationale faithfulness across student models.
    Results are reported as mean $\pm$ standard deviation over five random seeds.
    }
    \label{tab:faithfulness_results}

    \resizebox{\linewidth}{!}{%
    \begin{tabular}{llccccc}
    \\
    \toprule
    \textbf{Metric} & \textbf{Setup}
    & \textbf{Qwen2.5-VL-7B}
    & \textbf{Phi-3.5-Vision}
    & \textbf{Gemma-3-4B}
    & \textbf{LLaVA-1.6-7B}
    & \textbf{Average} \\
    \midrule

    R-I Groundedness
        & Base-Reasoning
        & $68.8 \pm 5.1$
        & $67.0 \pm 4.1$
        & $52.4 \pm 2.6$
        & $57.7 \pm 2.1$
        & $61.5 \pm 3.5$ \\
        & GRPO-Ans
        & $86.8 \pm 1.3$
        & $73.7 \pm 3.6$
        & $53.9 \pm 4.0$
        & $56.1 \pm 2.0$
        & $67.6 \pm 2.7$ \\
        & Reasoning-SFT
        & $81.9 \pm 2.6$
        & $73.6 \pm 3.2$
        & $80.4 \pm 1.9$
        & $81.3 \pm 2.3$
        & $79.3 \pm 2.5$ \\
    \midrule

    R-D Consistency
        & Base-Reasoning
        & $96.0 \pm 0.7$
        & $97.2 \pm 1.6$
        & $83.7 \pm 1.0$
        & $84.2 \pm 5.4$
        & $90.3 \pm 1.7$ \\
        & GRPO-Ans
        & $96.2 \pm 3.6$
        & $96.0 \pm 3.4$
        & $89.6 \pm 1.7$
        & $85.3 \pm 3.1$
        & $91.8 \pm 1.3$ \\
        & Reasoning-SFT
        & $99.1 \pm 0.6$
        & $98.2 \pm 0.7$
        & $99.4 \pm 0.6$
        & $99.7 \pm 0.4$
        & $99.1 \pm 0.2$ \\
    \midrule

    R-D Sensitivity
        & Base-Reasoning
        & $82.1 \pm 4.5$
        & $75.4 \pm 1.9$
        & $63.4 \pm 1.6$
        & $79.5 \pm 2.2$
        & $75.1 \pm 1.6$ \\
        & GRPO-Ans
        & $68.9 \pm 6.4$
        & $53.8 \pm 1.7$
        & $56.4 \pm 2.0$
        & $57.7 \pm 2.0$
        & $59.2 \pm 1.1$ \\
        & Reasoning-SFT
        & $70.0 \pm 4.4$
        & $66.5 \pm 2.0$
        & $69.3 \pm 1.6$
        & $57.2 \pm 2.2$
        & $65.7 \pm 1.3$ \\

    \bottomrule
    \end{tabular}%
    }
\end{table}

\subsection{Qualitative Results}
\label{app:qualitative_faithfulness}

\subsubsection{Rationale-to-image groundedness.}
Figure~\ref{fig:qual_ground}(a) compares rationales generated for an image
paired with the message ``Prepare for national emergencies''
(gold label: \textit{No}). Base-Reasoning and Reasoning-SFT both identify
the central object as a box marked with a red cross (i.e., a first-aid
kit), whereas GRPO-Ans misidentifies it as a ``gift box.'' Despite this
ungrounded detail, GRPO-Ans still predicts the correct label because its
conclusion relies on the broader absence of cues indicating national
emergency preparedness. This example illustrates how answer-only reward
can reinforce rationales containing ungrounded visual details as long as
they lead to the correct prediction. More broadly, 27 of the 31 GRPO-Ans
rationales judged ungrounded in this setting still yield correct
predictions. Figure~\ref{fig:qual_ground}(b) shows a similar pattern:
GRPO-Ans claims that the image contains no visual elements related to
social media interactions, overlooking the man's engagement with his
phone, which Reasoning-SFT explicitly identifies as a message-relevant
visual cue.

\begin{figure}[t]
\centering
\small

\begin{minipage}[c]{0.27\linewidth}
  \centering
  \includegraphics[width=\linewidth]{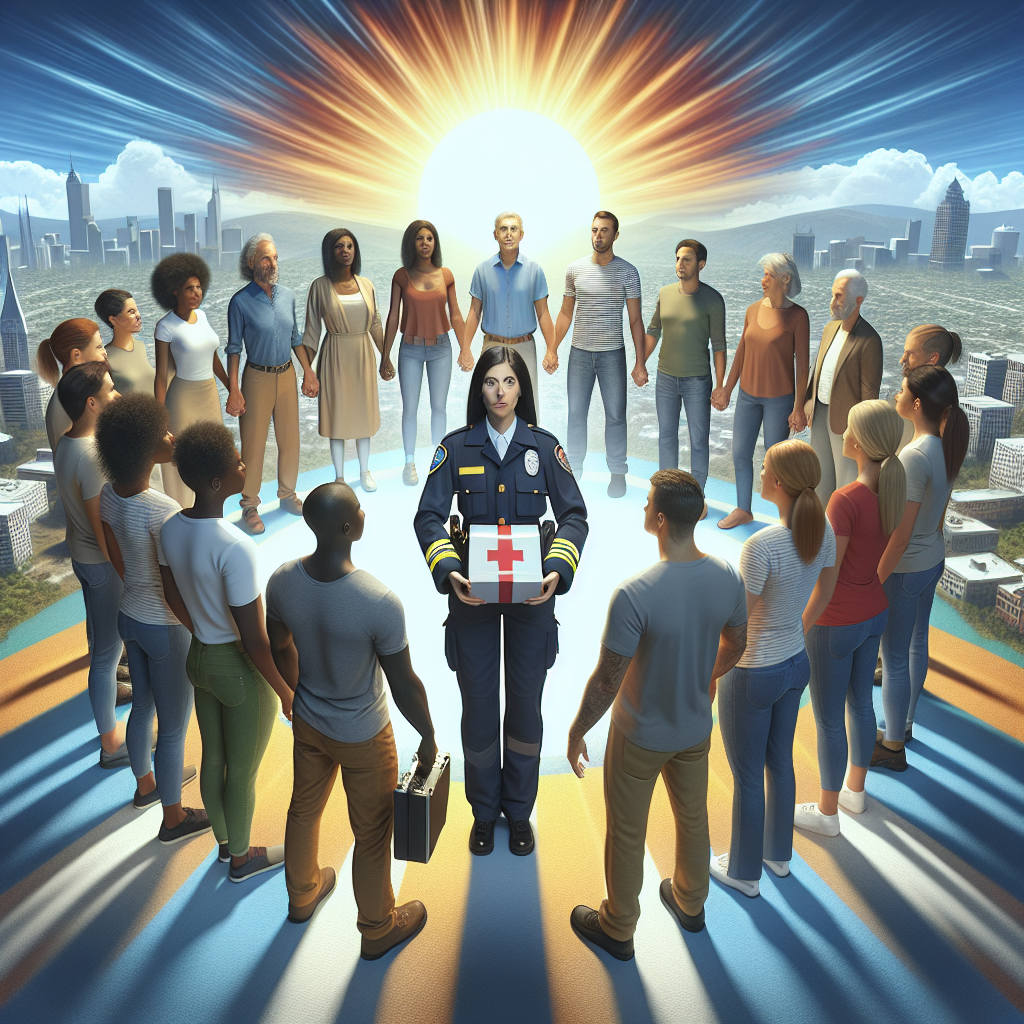}
\end{minipage}\hfill
\begin{minipage}[c]{0.70\linewidth}
  \textbf{Message:} Prepare for national emergencies \quad
  \textbf{Gold:} \textit{No}\\[3pt]
  \begin{tabularx}{\linewidth}{@{}l X c@{}}
  \toprule
  Setup & Rationale (excerpt) & Grd. \\
  \midrule
  Base-Reasoning & \ldots a diverse group of people gathered around a
  person in a uniform holding a box with a red cross \ldots
  & \checkmark \\
  GRPO-Ans & \ldots the focus on a police officer holding
  \textcolor{red}{a gift box} and a diverse group of people standing
  around does not immediately suggest any connection to national
  emergencies. & $\times$ \\
  Reasoning-SFT & \ldots a group of people gathered around a police
  officer holding a first aid kit, which does not fully represent the
  broader concept of national emergency readiness. & \checkmark \\
  \bottomrule
  \end{tabularx}
\end{minipage}
\\[2pt]
\makebox[\linewidth]{(a)}

\vspace{8pt}

\begin{minipage}[c]{0.27\linewidth}
  \centering
  \includegraphics[width=\linewidth]{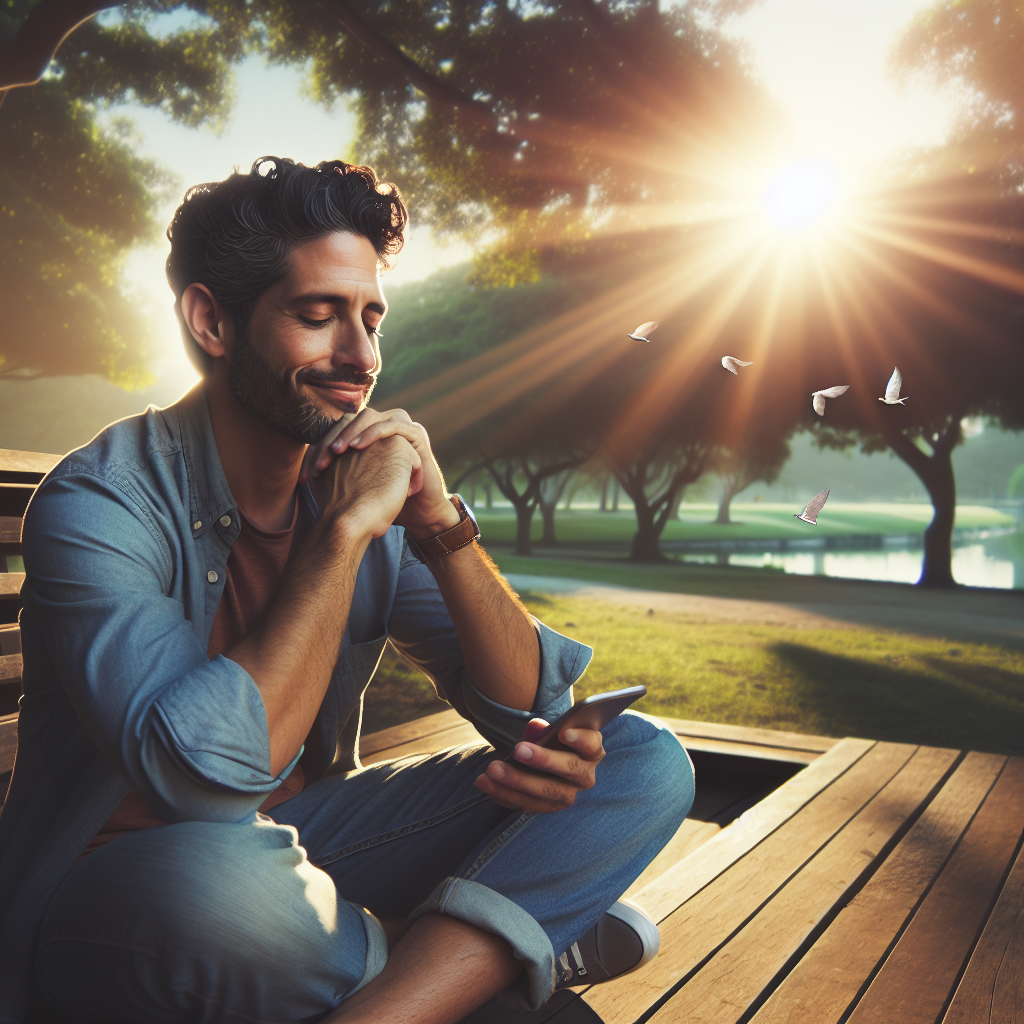}
\end{minipage}\hfill
\begin{minipage}[c]{0.70\linewidth}
  \textbf{Message:} Do not fight with others on social media \quad
  \textbf{Gold:} \textit{No}\\[3pt]
  \begin{tabularx}{\linewidth}{@{}l X c@{}}
  \toprule
  Setup & Rationale (excerpt) & Grd. \\
  \midrule
  GRPO-Ans & The image depicts a serene and peaceful scene with a person
  enjoying nature \ldots\ \textcolor{red}{The lack of any visual elements
  or context related to social media interactions} makes it ineffective
  \ldots & $\times$ \\
  Reasoning-SFT & The image depicts a man sitting calmly in a serene park
  setting, seemingly engaged with his phone, which does not directly
  relate to the act of fighting on social media. & \checkmark \\
  \bottomrule
  \end{tabularx}
\end{minipage}
\\[2pt]
\makebox[\linewidth]{(b)}

\caption{Qualitative examples of rationale-to-image groundedness.
Ungrounded content is highlighted in \textcolor{red}{red}, and Grd.\
denotes the groundedness judgment. In both examples, GRPO-Ans produces
an ungrounded rationale while still arriving at the correct prediction.}
\label{fig:qual_ground}
\end{figure}

\subsubsection{Rationale-to-decision consistency.}
Figure~\ref{fig:qual_consist}(a) compares rationales for the message
``Do not place flammable items near the gas stove'' (gold label:
\textit{No}). Although all three setups describe the same scene, they
differ in whether the rationale supports the final prediction.
Base-Reasoning argues that the image is persuasive and predicts
\textit{Yes}, yielding a consistent but incorrect prediction. GRPO-Ans
similarly argues that the image is persuasive, but predicts \textit{No}.
Its prediction is therefore correct but inconsistent with its own
rationale. Because GRPO-Ans is rewarded only for the final answer, such
rationale--prediction inconsistency is not directly penalized: 17 of the
21 inconsistent GRPO-Ans samples in this setting still yield correct
predictions. In contrast, Reasoning-SFT argues that the image is not
persuasive and predicts \textit{No}, producing a rationale consistent
with the correct prediction. Figure~\ref{fig:qual_consist}(b) shows the
same pattern: GRPO-Ans provides several reasons for why the image is
persuasive before predicting \textit{No}, whereas Reasoning-SFT provides
a rationale that supports its \textit{No} prediction.

\begin{figure}[t]
\centering
\small

\begin{minipage}[c]{0.24\linewidth}
  \centering
  \includegraphics[width=\linewidth]{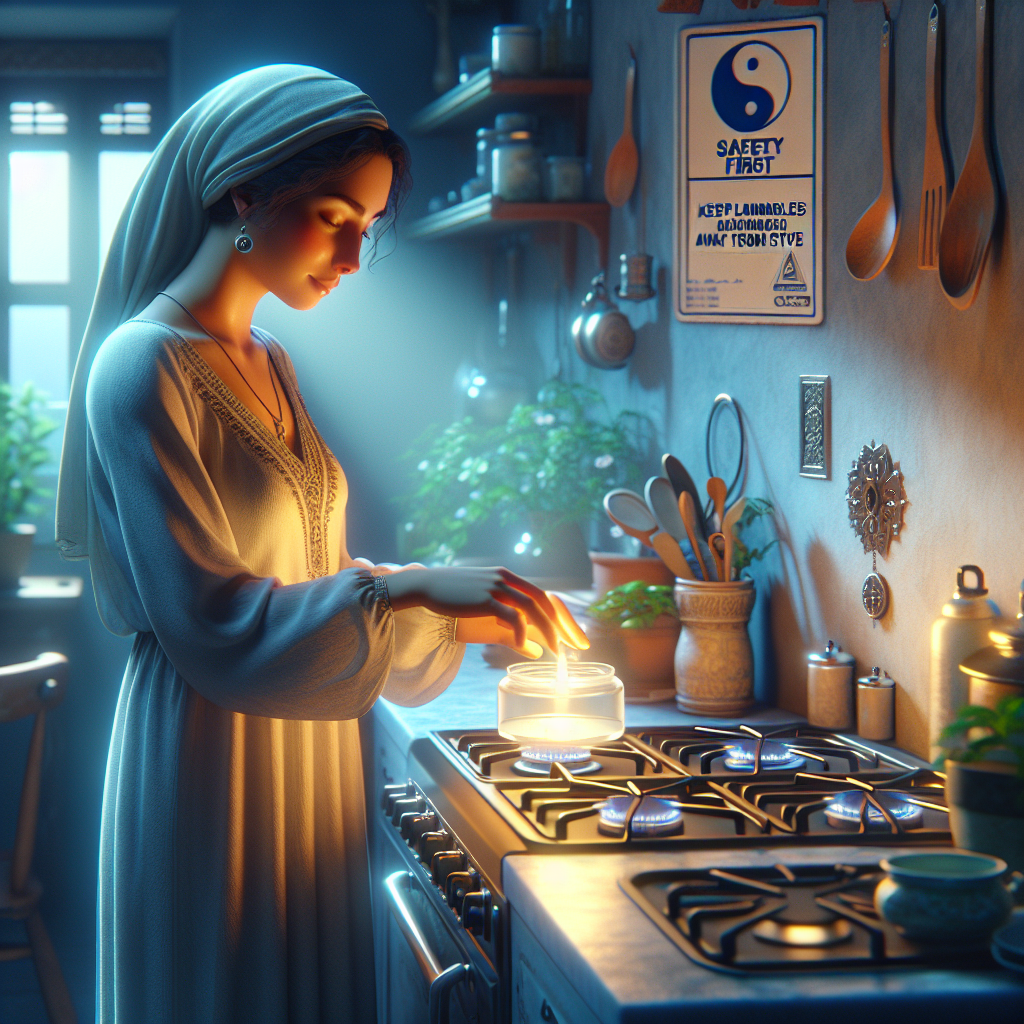}
\end{minipage}\hfill
\begin{minipage}[c]{0.74\linewidth}
  \textbf{Message:} Do not place flammable items near the gas stove \quad
  \textbf{Gold:} \textit{No}\\[3pt]
  \begin{tabularx}{\linewidth}{@{}l X c@{}}
  \toprule
  Setup & Rationale (excerpt) & Cons. \\
  \midrule
  Base-Reasoning & \ldots The image effectively conveys the message by
  showing a person engaging in a potentially hazardous activity \ldots\
  making the message more persuasive.
  & Yes \\
  GRPO-Ans & \ldots The image effectively conveys the message by showing
  a potential hazard \ldots\ This makes it persuasive in communicating
  the importance of keeping flammable items away from the gas stove.
  & \textcolor{red}{No} \\
  Reasoning-SFT & The image fails to be persuasive because it shows a
  person safely interacting with a candle near a gas stove without any
  visible danger \ldots
  & Yes \\
  \bottomrule
  \end{tabularx}
\end{minipage}
\\[2pt]
\makebox[\linewidth]{(a)}

\vspace{8pt}

\begin{minipage}[c]{0.24\linewidth}
  \centering
  \includegraphics[width=\linewidth]{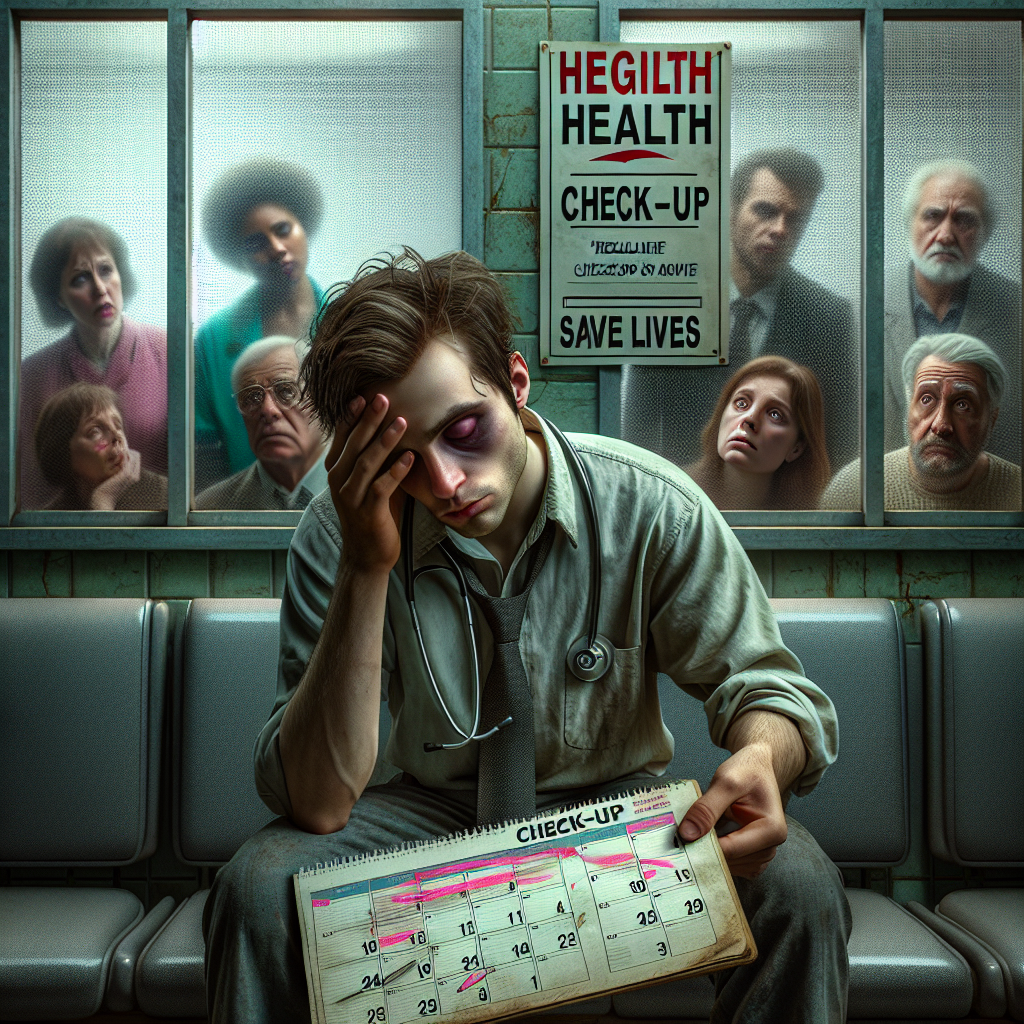}
\end{minipage}\hfill
\begin{minipage}[c]{0.74\linewidth}
  \textbf{Message:} Get a health check-up regularly \quad
  \textbf{Gold:} \textit{No}\\[3pt]
  \begin{tabularx}{\linewidth}{@{}l X c@{}}
  \toprule
  Setup & Rationale (excerpt) & Cons. \\
  \midrule
  GRPO-Ans & \ldots The message conveyed by the image is indeed
  persuasive for several reasons: \ldots\ Overall, the image effectively
  communicates the message of the importance of regular health check-ups
  \ldots
  & \textcolor{red}{No} \\
  Reasoning-SFT & The image fails to effectively convey the message
  \ldots\ because it depicts a doctor looking overwhelmed and stressed,
  which may discourage rather than encourage viewers \ldots
  & Yes \\
  \bottomrule
  \end{tabularx}
\end{minipage}
\\[2pt]
\makebox[\linewidth]{(b)}

\caption{Qualitative examples of rationale-to-decision consistency.
Cons.\ indicates whether the rationale supports the model's final
prediction. Inconsistent cases are highlighted in \textcolor{red}{red}.
In both examples, GRPO-Ans predicts the correct label despite its
rationale supporting the opposite decision.}
\label{fig:qual_consist}
\end{figure}

\subsubsection{Rationale-to-decision sensitivity.}
Figure~\ref{fig:qual_sens} presents an example for the message
``Use an exercise ball to balance your body'' (gold label:
\textit{Yes}). The original image shows two figures in the same seated
posture: a grayscale figure on a chair and a colorful, anatomically
detailed figure on an exercise ball. All three setups refer to the
contrast between the two figures, but differ in the visual evidence they
emphasize. Base-Reasoning and GRPO-Ans focus primarily on the contrast
between the chair and the exercise ball. However, neither setup is
sensitive to the removal of its cited evidence. In contrast,
Reasoning-SFT identifies the colorful, anatomically detailed figure on
the ball as key evidence supporting the message. When this figure is
replaced with an ordinary person while the exercise ball is preserved,
its prediction changes from \textit{Yes} to \textit{No}, indicating
sensitivity to the cited evidence.

\begin{figure}[t]
\centering
\small
\renewcommand\tabularxcolumn[1]{m{#1}} 

\begin{minipage}[c]{0.19\linewidth}
  \centering
  \includegraphics[width=\linewidth]{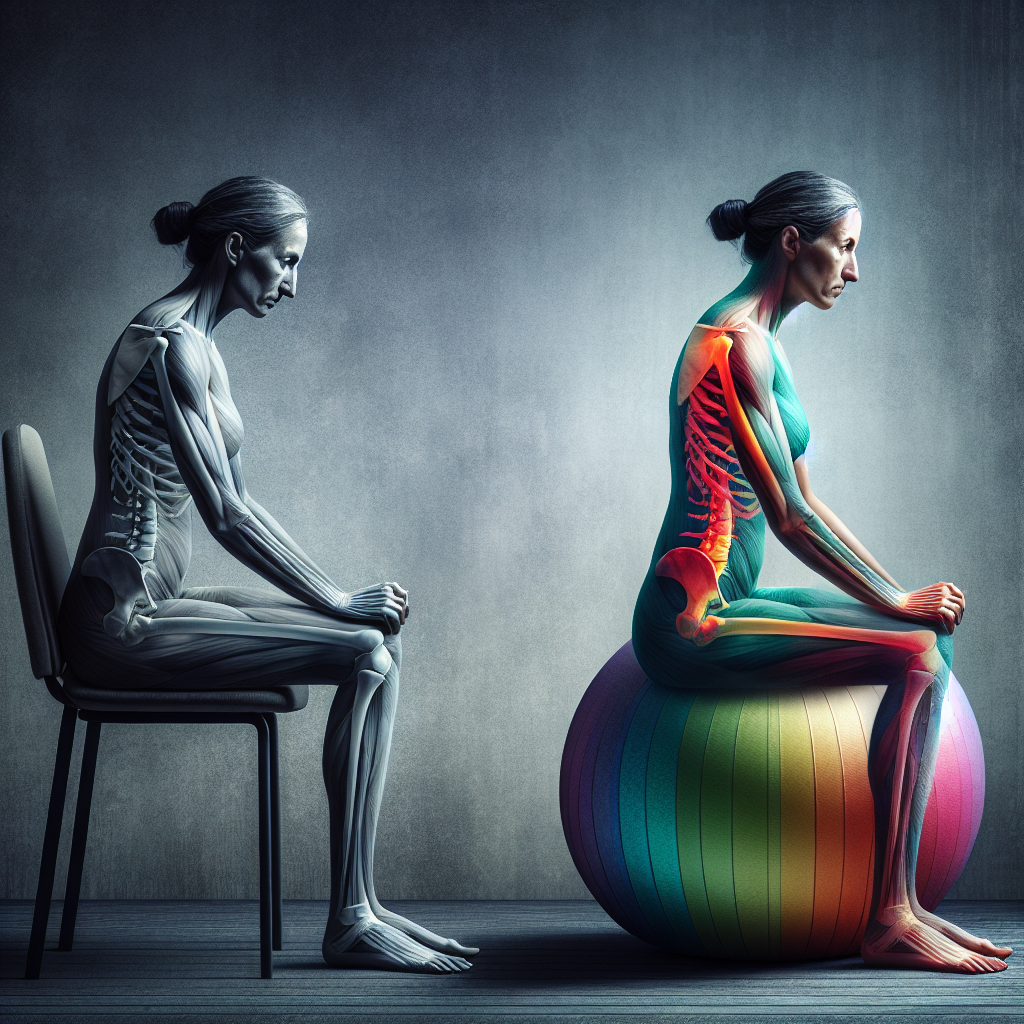}\\[2pt]
  \textit{Original}
\end{minipage}\hfill
\begin{minipage}[c]{0.79\linewidth}
  \textbf{Message:} Use an exercise ball to balance your body \quad
  \textbf{Gold:} \textit{Yes}\\[3pt]

  \begin{tabularx}{\linewidth}
    {@{}l >{\centering\arraybackslash}m{0.15\linewidth} X c c@{}}
  \toprule
  Setup & Edited & Rationale (excerpt) & Pred. & Sens. \\
  \midrule

  Base-Reasoning
  & \includegraphics[width=\linewidth]{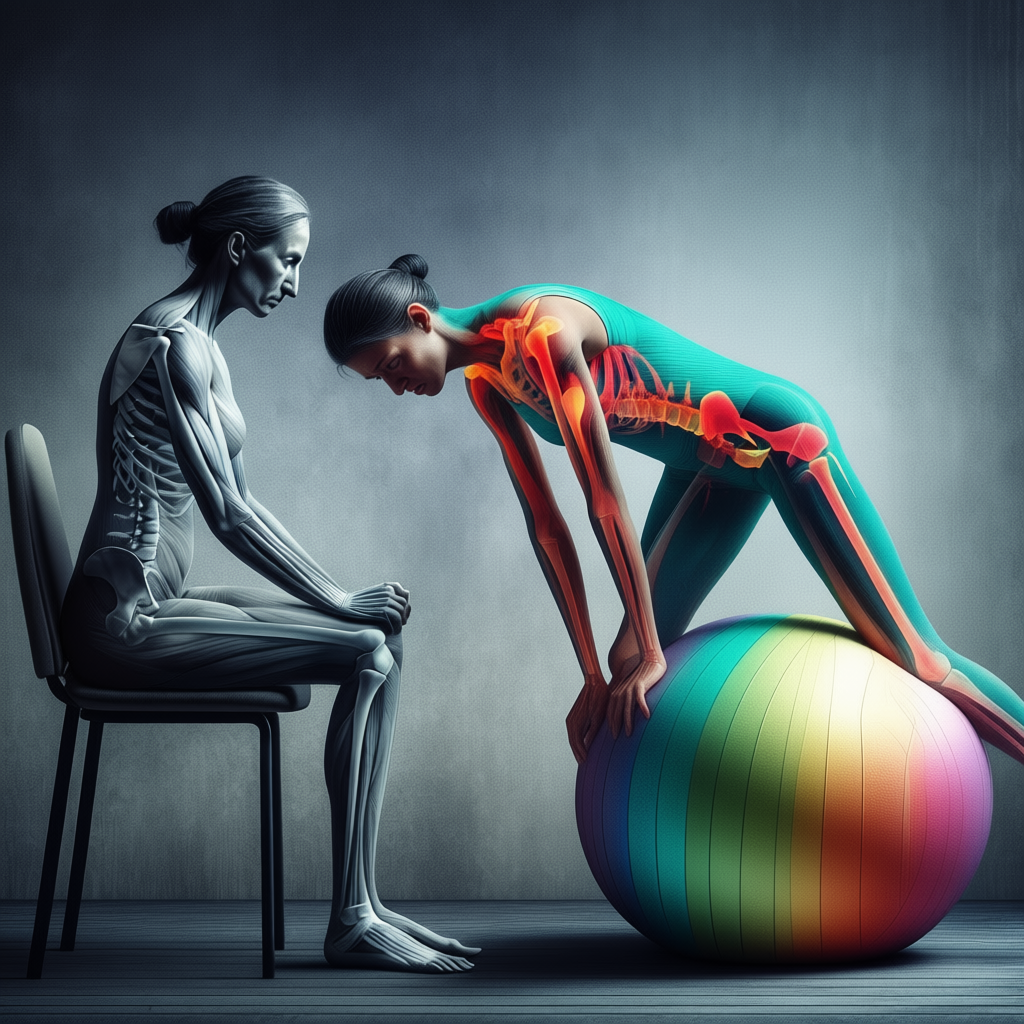}
  & \ldots The contrast between the two figures further highlights the
  importance of the exercise ball in achieving \textbf{a more balanced
  posture}.
  & Yes$\rightarrow$Yes & No \\
  \addlinespace[4pt]

  GRPO-Ans
  & \includegraphics[width=\linewidth]{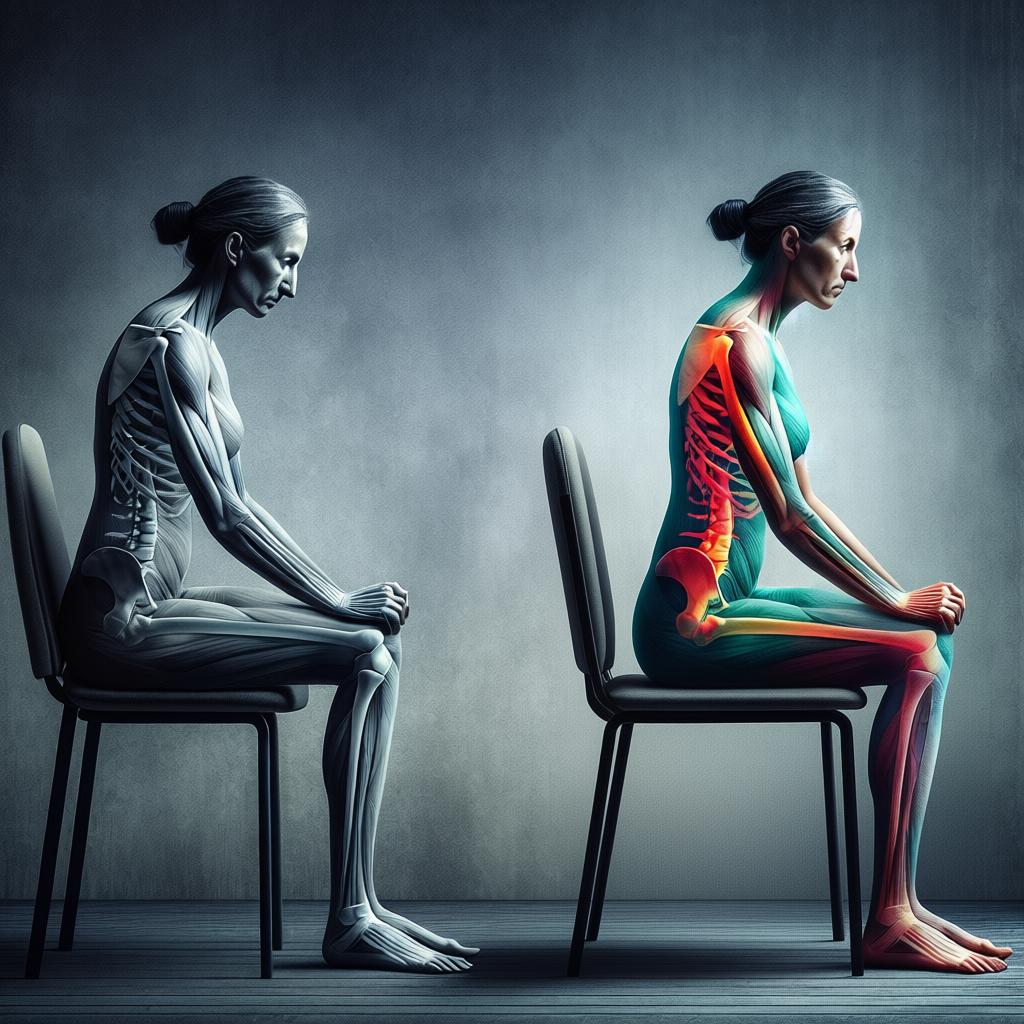}
  & \ldots the person on the right sitting on
  \textbf{the exercise ball} \ldots The contrast with the person on the
  left, who is sitting on a chair and appears to be in a more static
  position \ldots
  & No$\rightarrow$No & No \\
  \addlinespace[4pt]

  Reasoning-SFT
  & \includegraphics[width=\linewidth]{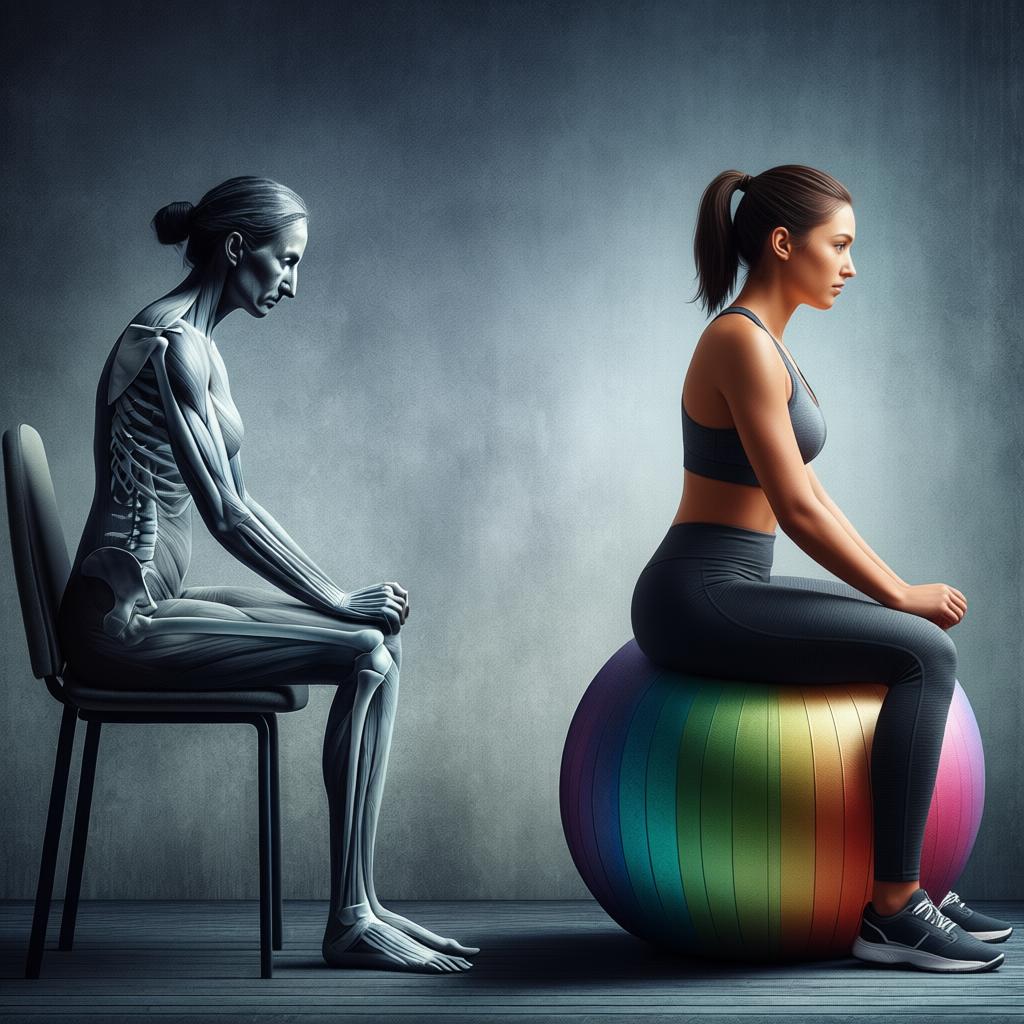}
  & \textit{Positive:} \textbf{The colorful, anatomically detailed figure
  on the ball}: The vibrant colors and focus on anatomy highlight the
  benefits of using an exercise ball \ldots
  & Yes$\rightarrow$No & Yes \\

  \bottomrule
  \end{tabularx}
\end{minipage}

\caption{Qualitative example of rationale-to-decision sensitivity.
The original image is shown on the left, while Edited shows the
counterfactual image obtained by removing or altering the visual evidence
cited in each setup's rationale. Bold text marks the rationale evidence
targeted by the corresponding edit. Pred.\ reports the prediction before
and after the intervention, and Sens.\ indicates whether removing the
cited evidence decreases the logit of the original prediction.}
\label{fig:qual_sens}
\end{figure}

\subsection{Joint Error Analysis of Rationale Validity and Prediction Correctness}
\label{app:error_analysis}

To further examine the relationship between rationale quality and prediction correctness, we conduct a four-way error analysis across four student backbones and five random seeds. We define a valid rationale as one that satisfies both rationale-to-image groundedness and rationale-to-decision consistency, independently of prediction correctness. Based on these criteria, we categorize each prediction into four groups: (1) both correct, (2) wrong rationale only, (3) wrong prediction only, and (4) both wrong.

We compare Base-Reasoning and Reasoning-SFT to examine how rationale supervision changes the distribution of these outcomes. We additionally report $Q_{\mathrm{correct}}$ and $Q_{\mathrm{incorrect}}$, denoting the proportion of valid rationales among correct and incorrect predictions, respectively.

\begin{table}[t]
\centering
\caption{Joint error analysis of rationale validity and prediction correctness across four student backbones. The four outcome categories are reported as percentages (\%), averaged over five random seeds. $Q_{\mathrm{correct}}$ and $Q_{\mathrm{incorrect}}$ represent the proportions of valid rationales among correct and incorrect predictions, respectively, reported as mean $\pm$ standard deviation.}
\label{tab:error_analysis}
\resizebox{\textwidth}{!}{
\begin{tabular}{llcccccc}
\toprule
Backbone & Method
& \shortstack{Both\\correct}
& \shortstack{Wrong rat.\\only}
& \shortstack{Wrong pred.\\only}
& \shortstack{Both\\wrong}
& $Q_{\mathrm{correct}}$
& $Q_{\mathrm{incorrect}}$ \\
\midrule
Qwen
& Base-Reasoning
& 40.3 & 20.9 & 25.0 & 13.8
& $0.657 \pm 0.050$
& $0.647 \pm 0.061$ \\
& Reasoning-SFT
& \textbf{61.5} & 13.3 & 20.3 & \textbf{4.9}
& $0.822 \pm 0.020$
& $0.809 \pm 0.055$ \\
\midrule
Phi
& Base-Reasoning
& 39.7 & 16.3 & 25.4 & 18.5
& $0.711 \pm 0.064$
& $0.578 \pm 0.070$ \\
& Reasoning-SFT
& \textbf{51.9} & 20.8 & 20.3 & \textbf{7.1}
& $0.715 \pm 0.029$
& $0.737 \pm 0.054$ \\
\midrule
Gemma
& Base-Reasoning
& 25.0 & 30.7 & 17.7 & 26.7
& $0.448 \pm 0.036$
& $0.400 \pm 0.041$ \\
& Reasoning-SFT
& \textbf{61.1} & 15.0 & 19.1 & \textbf{4.8}
& $0.803 \pm 0.026$
& $0.805 \pm 0.090$ \\
\midrule
LLaVA
& Base-Reasoning
& 26.4 & 25.5 & 19.9 & 28.3
& $0.507 \pm 0.045$
& $0.414 \pm 0.031$ \\
& Reasoning-SFT
& \textbf{60.9} & 12.4 & 20.0 & \textbf{6.7}
& $0.831 \pm 0.022$
& $0.747 \pm 0.064$ \\
\bottomrule
\end{tabular}
}
\end{table}

\paragraph{Joint improvements in rationale validity and prediction correctness.}
As shown in Table~\ref{tab:error_analysis}, Reasoning-SFT increases the proportion of cases with both valid rationales and correct predictions across all four backbones, from 32.9\% to 58.9\% on average. Meanwhile, the proportion of cases where both the rationale and prediction are incorrect decreases from 21.8\% to 5.9\%. These results indicate that Reasoning-SFT can improve rationale validity alongside prediction performance.

\paragraph{Discrepancy between rationale validity and prediction correctness.}
Despite these improvements, valid rationales do not necessarily lead to correct predictions. For Reasoning-SFT, $Q_{\mathrm{incorrect}}$ ranges from 0.737 to 0.809 across backbones, indicating that a substantial proportion of incorrect predictions are accompanied by visually grounded and decision-consistent rationales. Moreover, the difference between $Q_{\mathrm{correct}}$ and $Q_{\mathrm{incorrect}}$ varies across backbones, suggesting that rationale validity and prediction correctness are not always strongly aligned.

\paragraph{Implications.}
These findings provide additional evidence that prediction effectiveness and rationale faithfulness should be evaluated separately. While Reasoning-SFT increases the proportion of cases where both criteria are satisfied, rationale validity alone does not guarantee prediction correctness. Furthermore, this analysis characterizes their relationship rather than establishing a causal influence of rationales on predictions, which is separately investigated through our rationale-to-decision sensitivity evaluation in Section~\ref{sec:faithfulness-results}.

\subsection{Faithfulness Evaluation on Ambiguous Samples}
\label{app:ambiguous_faithfulness}

\paragraph{Experimental setup.}
To examine the behavior of our faithfulness evaluation framework under ambiguous persuasiveness judgments, we conduct additional experiments on 246 image--message pairs for which four human annotators were evenly split (2:2) on persuasiveness labels. We evaluate Base-Reasoning, GRPO-Ans, and Reasoning-SFT on Qwen and Gemma, comparing the original and ambiguous subsets using the same training seed. For both subsets, groundedness is evaluated using the improved GPT-6-astra judge described in Appendix~\ref{app:groundedness_candidates}, while consistency and sensitivity follow the evaluation protocols in Section~\ref{sec:faithfulness-metrics}.

\begin{table}[t]
\centering
\small
\setlength{\tabcolsep}{5pt}
\caption{Faithfulness scores (\%) on the original and ambiguous evaluation subsets using the same training seed. Groundedness is evaluated using GPT-6-astra for both subsets.}
\label{tab:ambiguous_faithfulness}
\begin{tabular}{llcccc}
\toprule
& & \multicolumn{2}{c}{Qwen} & \multicolumn{2}{c}{Gemma} \\
\cmidrule(lr){3-4}\cmidrule(lr){5-6}
Metric & Setup & Original & Ambiguous & Original & Ambiguous \\
\midrule
\multirow{3}{*}{Groundedness}
& Base-Reasoning & 61.8 & 66.3 & 49.3 & 46.7 \\
& GRPO-Ans & 78.3 & 81.3 & 59.2 & 57.3 \\
& Reasoning-SFT & \textbf{86.0} & \textbf{87.2} & \textbf{78.1} & \textbf{76.8} \\
\midrule
\multirow{3}{*}{Consistency}
& Base-Reasoning & 96.2 & 96.7 & 83.7 & 84.6 \\
& GRPO-Ans & 90.0 & 97.2 & 88.5 & 87.0 \\
& Reasoning-SFT & \textbf{99.5} & \textbf{98.0} & \textbf{98.6} & \textbf{99.0} \\
\midrule
\multirow{3}{*}{Sensitivity}
& Base-Reasoning & 88.5 & 52.7 & 63.2 & 60.3 \\
& GRPO-Ans & 73.7 & 55.2 & 56.5 & 44.8 \\
& Reasoning-SFT & 74.6 & 51.9 & 69.4 & 50.6 \\
\bottomrule
\end{tabular}
\end{table}

\paragraph{Results.}
As shown in Table~\ref{tab:ambiguous_faithfulness}, groundedness and consistency remain largely stable on ambiguous samples. Reasoning-SFT achieves the highest scores on both backbones, with groundedness of 87.2\% on Qwen and 76.8\% on Gemma, and consistency of 98.0\% and 99.0\%, respectively. These results suggest that Reasoning-SFT maintains its relative advantages in generating visually grounded and prediction-consistent rationales, even when human persuasiveness judgments are ambiguous.

In contrast, sensitivity generally decreases on ambiguous samples. For Reasoning-SFT, sensitivity decreases from 74.6\% to 51.9\% on Qwen and from 69.4\% to 50.6\% on Gemma. Similar decreases are observed for the other methods, although their magnitudes vary across backbones. This may reflect the difficulty of identifying decisive persuasive visual cues when human annotators disagree on an image's persuasiveness. In such cases, modifying a particular rationale-cited cue may be less likely to reduce confidence in the original prediction, as multiple visual cues or interpretations may contribute to the judgment.

\paragraph{Implications for framework validity.}
These findings provide preliminary support for the construct validity of sensitivity. While groundedness and consistency remain relatively stable on ambiguous samples, sensitivity generally decreases. This pattern is consistent with sensitivity's intended purpose of measuring the influence of rationale-cited visual evidence on model predictions. When persuasiveness judgments are ambiguous, individual visual cues may be less decisive, resulting in lower sensitivity to targeted edits.

Thus, the observed decrease may reflect weaker reliance on individual visual cues rather than invalid reasoning, suggesting that sensitivity captures a distinct aspect of rationale faithfulness. However, these results are based on a single training seed and two backbones, and further experiments are needed to confirm this interpretation.

\section{Quantitative Analysis of the Message-Relevance Shortcut}
\label{app:message_shortcut}

\subsection{Definition and Evaluation Protocol}

We quantitatively investigate the message-relevance shortcut described in Section~\ref{sec:motivation}. We define this shortcut as a reasoning pattern in which a model treats the mere presence of message-related visual elements as sufficient evidence of persuasiveness, without explaining how or why these elements contribute to persuading viewers.

To quantify this phenomenon, we employ GPT-6-astra as an automated evaluator. Given an image, its target message, and a model-generated rationale accompanying a positive (Yes) prediction, the evaluator determines whether the rationale exhibits the message-relevance shortcut. A rationale is classified as shortcut-based when the presence or relevance of message-related elements serves as its primary justification without explaining their persuasive mechanism. Importantly, merely mentioning message-related elements does not constitute a shortcut if the rationale explains how they contribute to persuasiveness.

The evaluation focuses on the reasoning itself rather than the correctness of the predicted label. We employ two few-shot examples illustrating shortcut-based and non-shortcut-based reasoning. To assess the reliability of automated evaluation, we further compare GPT-6-astra judgments with human annotations, observing 90\% agreement.

\subsection{Automated Evaluator Prompt}
\label{app:shortcut_prompt}

We provide the evaluation instruction and two few-shot examples used for automated shortcut detection.

\paragraph{Evaluation Instruction.}

\begin{promptbox}
A model judged that the image persuasively conveys the message (prediction: Yes). Evaluate whether its rationale for this judgment exhibits a **message-relevance shortcut**: treating the mere presence of message-related visual elements as sufficient evidence of persuasiveness, without explaining how or why they make the message persuasive.

Rules:

1. Do not classify a rationale as a shortcut merely because it mentions message-related elements. It is a shortcut only when their presence or relevance is the main justification and no persuasive mechanism (e.g., the effect on the viewer) is given.

2. Focus on the reasoning, not on whether the Yes judgment is correct.

Return only:

```json
{
  "shortcut": true/false,
  "reason": "Brief explanation"
}
```

Two worked examples follow, then the case to evaluate.
\end{promptbox}

\paragraph{Few-Shot Example 1: Shortcut-Based Reasoning.}

\begin{promptbox}
Message: Make coffee at home

Image: Make_coffee_at_home/16.jpg

Rationale:
Yes, the image effectively conveys the message. It shows a well-equipped kitchen with a coffee machine, suggesting that one can easily make coffee at home.

Expected output:
{
  "shortcut": true,
  "reason": "The rationale treats the mere presence of a coffee machine in a kitchen as sufficient evidence of persuasiveness. It does not explain how this would motivate viewers to make coffee at home, and it ignores the long receipts and luxury styling that frame the scene as costly and undermine the message."
}
\end{promptbox}

\paragraph{Few-Shot Example 2: Non-Shortcut Reasoning.}

\begin{promptbox}
Message: Prepare emergency food supplies for a typhoon

Image: Prepare_emergency_food_supplies_for_a_typhoon/22.jpg

Rationale:
The image powerfully conveys the urgency of preparing emergency food supplies by depicting a family's fear and vulnerability amidst a devastating typhoon. The contrast between their relatively stocked pantry and the chaotic flood-stricken street outside highlights the importance of having readily available provisions. The expressions of worry and helplessness on the family members' faces evoke an emotional response, making viewers more likely to consider their own preparedness and the potential consequences of being caught unprepared. Ultimately, the visual narrative effectively links the emotional impact of a disaster with the practical necessity of emergency supplies.

Expected output:
{
  "shortcut": false,
  "reason": "The rationale explains how the elements persuade: the family's fear and the flooded street evoke vulnerability, and the contrast with the stocked pantry makes the consequences of being unprepared concrete, prompting viewers to consider their own preparedness."
}
\end{promptbox}

\subsection{Evaluation Metrics}
\label{app:shortcut_metrics}

We define two metrics to quantify the prevalence of shortcut-based reasoning among positive predictions. Let $y_i \in \{0,1\}$ denote the ground-truth persuasiveness label, $\hat{y}_i \in \{0,1\}$ the model prediction, and $s_i \in \{0,1\}$ the shortcut indicator for a positive prediction, where $s_i=1$ indicates that the rationale exhibits the message-relevance shortcut.

\paragraph{Shortcut Rate among Yes (SC-Yes).}
SC-Yes measures the proportion of positive predictions whose rationales exhibit the shortcut:

\begin{equation}
\mathrm{SC}_{\mathrm{Yes}} =
\frac{\sum_{i:\hat{y}_i=1} s_i}
{\sum_i \mathbb{1}[\hat{y}_i=1]}.
\end{equation}

\paragraph{Shortcut Share of False Positives (SC-FP).}
SC-FP measures the proportion of false-positive predictions whose rationales exhibit the shortcut:

\begin{equation}
\mathrm{SC}_{\mathrm{FP}} =
\frac{\sum_{i:y_i=0,\hat{y}_i=1} s_i}
{\sum_i \mathbb{1}[y_i=0,\hat{y}_i=1]}.
\end{equation}

SC-Yes captures the overall prevalence of shortcut-based reasoning among positive predictions, whereas SC-FP specifically examines its prevalence among incorrect positive predictions. Both metrics are computed separately for each evaluation seed and averaged across seeds. Lower values indicate less frequent shortcut-based reasoning.

\subsection{Quantitative Results}
\label{app:shortcut_results}

Table~\ref{tab:message_shortcut} reports the prevalence of message-relevance shortcuts for Base-Reasoning and Reasoning-SFT across four backbones.

\begin{table}[t]
\centering
\caption{Prevalence of message-relevance shortcuts across backbones and training methods. Values are averaged across evaluation seeds. Lower values indicate less frequent shortcut-based reasoning.}
\label{tab:message_shortcut}
\small
\begin{tabular}{llcc}
\toprule
Backbone & Method & SC-Yes (\%) $\downarrow$ & SC-FP (\%) $\downarrow$ \\
\midrule
LLaVA & Base-Reasoning & 13.85 & 16.29 \\
      & Reasoning-SFT & \textbf{11.35} & \textbf{12.81} \\
\midrule
Gemma & Base-Reasoning & \textbf{1.83} & \textbf{0.70} \\
      & Reasoning-SFT & 4.43 & 9.95 \\
\midrule
Phi & Base-Reasoning & 64.57 & 74.44 \\
    & Reasoning-SFT & \textbf{25.74} & \textbf{30.08} \\
\midrule
Qwen & Base-Reasoning & 20.47 & 20.30 \\
     & Reasoning-SFT & \textbf{12.64} & \textbf{15.76} \\
\bottomrule
\end{tabular}
\end{table}

\paragraph{Shortcut prevalence decreases on three of the four backbones.}
Compared with Base-Reasoning, Reasoning-SFT exhibits lower SC-Yes and SC-FP on LLaVA, Phi, and Qwen. The largest reduction is observed in Phi, where SC-Yes decreases from 64.57\% to 25.74\% and SC-FP from 74.44\% to 30.08\%. Qwen also shows reductions in SC-Yes (20.47\% to 12.64\%) and SC-FP (20.30\% to 15.76\%), while LLaVA exhibits smaller decreases across both metrics. These findings suggest that Reasoning-SFT can reduce reliance on message-related visual elements as insufficient justification for positive predictions, including false-positive cases.

\paragraph{The effect is backbone-dependent.}
In contrast, Gemma exhibits increased shortcut prevalence under Reasoning-SFT, with SC-Yes rising from 1.83\% to 4.43\% and SC-FP from 0.70\% to 9.95\%. This opposite trend indicates that rationale supervision does not uniformly mitigate shortcut-based reasoning across model families.

Overall, these findings provide quantitative evidence of the message-relevance shortcut across MLLMs, with substantial variation in its prevalence across backbones. Reasoning-SFT reduces shortcut-based reasoning on three of the four evaluated backbones, although the opposite trend in Gemma highlights the backbone-dependent nature of its effectiveness.

\section{Teacher Rationale Quality}
\label{app:teacher_rationale_quality}

Our rationale collection process conditions teacher models on ground-truth persuasiveness labels, raising the possibility that the generated rationales may reflect plausible post-hoc justifications rather than evidence underlying the teachers' decisions. To examine this concern, we conduct two complementary analyses: (1) comparing label-conditioned and label-free rationales, and (2) measuring teacher prediction sensitivity to changes in rationale-cited visual evidence.

\subsection{Rationale Similarity With and Without Ground-Truth Labels}

We compare rationales generated by the same teacher model under two conditions: (1) with the ground-truth label provided and (2) without the label, allowing the teacher to independently predict the persuasiveness label and generate its rationale. For the second condition, we retain cases where the teacher correctly predicts the ground-truth label, ensuring that the compared rationales explain the same decision.

We measure lexical and semantic similarity using ROUGE-L, METEOR, SBERT, and BERTScore. Table~\ref{tab:teacher_rationale_similarity} presents the results across four teacher models.

\begin{table}[t]
\centering
\small
\caption{Similarity between label-conditioned and label-free teacher rationales. All values are mean similarity scores (\%).}
\label{tab:teacher_rationale_similarity}
\begin{tabular}{lccccc}
\toprule
Teacher & ROUGE-L & METEOR & SBERT & BERTScore F1 & Rescaled F1 \\
\midrule
Gemma & 38.1 & 42.3 & 89.7 & 91.1 & 47.5 \\
LLaVA & 30.3 & 36.4 & 90.4 & 90.2 & 42.0 \\
Phi & 39.2 & 44.1 & 91.5 & 92.1 & 53.3 \\
Qwen & 50.4 & 54.4 & 94.5 & 93.3 & 60.4 \\
\midrule
\textbf{Overall} & \textbf{40.3} & \textbf{45.0} & \textbf{91.6} & \textbf{91.8} & \textbf{51.4} \\
\bottomrule
\end{tabular}
\end{table}

Label-conditioned and label-free rationales exhibit high semantic similarity, with average SBERT and BERTScore F1 values of 91.6\% and 91.8\%, respectively. Although lexical overlap is lower, the semantic similarity suggests that the rationales often convey similar content despite differences in wording. These findings suggest that label conditioning may have a limited influence on the semantic content of rationales when teachers independently predict the correct label.

\subsection{Teacher Rationale-to-Decision Sensitivity}

To further examine whether label-conditioned rationales identify visual evidence relevant to the teachers' predictions, we apply our rationale-to-decision sensitivity metric directly to the teacher models. Unlike text-based similarity, this evaluation examines whether modifying rationale-cited visual evidence affects the teachers' predictions.

For each teacher, we select 100 samples where its prediction on the original image matches the ground-truth label. Using the corresponding label-conditioned rationales collected for Reasoning-SFT, we generate counterfactual image edits that weaken the cited visual evidence. As a control, we additionally evaluate irrelevant visible edits designed to preserve the original image's persuasiveness.

We then evaluate each teacher on the original and edited images and compute rationale-to-decision sensitivity using the metric defined in Section~\ref{sec:decision-sensitivity}. Table~\ref{tab:teacher_sensitivity} reports the results.

\begin{table}[t]
\centering
\small
\caption{Rationale-to-decision sensitivity (\%) of teacher models under rationale-guided and irrelevant image edits.}
\label{tab:teacher_sensitivity}
\begin{tabular}{lccccc}
\toprule
Metric & Qwen & Phi & Gemma & LLaVA & Average \\
\midrule
Rationale-guided & 79.0 & 76.0 & 68.0 & 75.0 & \textbf{74.5} \\
Irrelevant edits & 53.0 & 51.0 & 54.0 & 56.0 & 53.5 \\
\bottomrule
\end{tabular}
\end{table}

Across all four teachers, rationale-guided edits yield higher sensitivity than irrelevant edits, with an average difference of 21.0 percentage points (74.5\% vs. 53.5\%). This suggests that visual evidence cited in label-conditioned rationales is more closely associated with the teachers' decisions than unrelated visual features.

\subsection{Discussion and Limitations}

Taken together, these analyses provide complementary evidence regarding the quality of label-conditioned teacher rationales. The similarity analysis shows that they remain semantically consistent with independently generated explanations when teachers predict correctly, while the sensitivity analysis indicates that their cited visual evidence is associated with the teachers' decision behavior.

Nevertheless, these findings do not fully establish the absence of post-hoc justification. The similarity analysis is restricted to correctly predicted cases, and the sensitivity analysis depends on the effectiveness of counterfactual image edits. Thus, while we cannot guarantee that label-conditioned rationales are free from post-hoc justification, these results provide supporting evidence for their use as training supervision.

\section{Representative Web-Sourced Image--Message Pairs}
\label{app:web_examples}

To illustrate the concerns regarding image--message alignment in the web-sourced subset of PVP, we present representative examples in Figure~\ref{fig:web_examples}. In these cases, the depicted visual content has limited relevance to the corresponding target message. Such examples may be valid unpersuasive cases, but their presence raises additional considerations when evaluating models' ability to reason about visual persuasion rather than simply recognizing message relevance.

\begin{figure}[t]
    \centering
    \begin{minipage}[t]{0.47\linewidth}
        \centering
        \includegraphics[width=\linewidth]{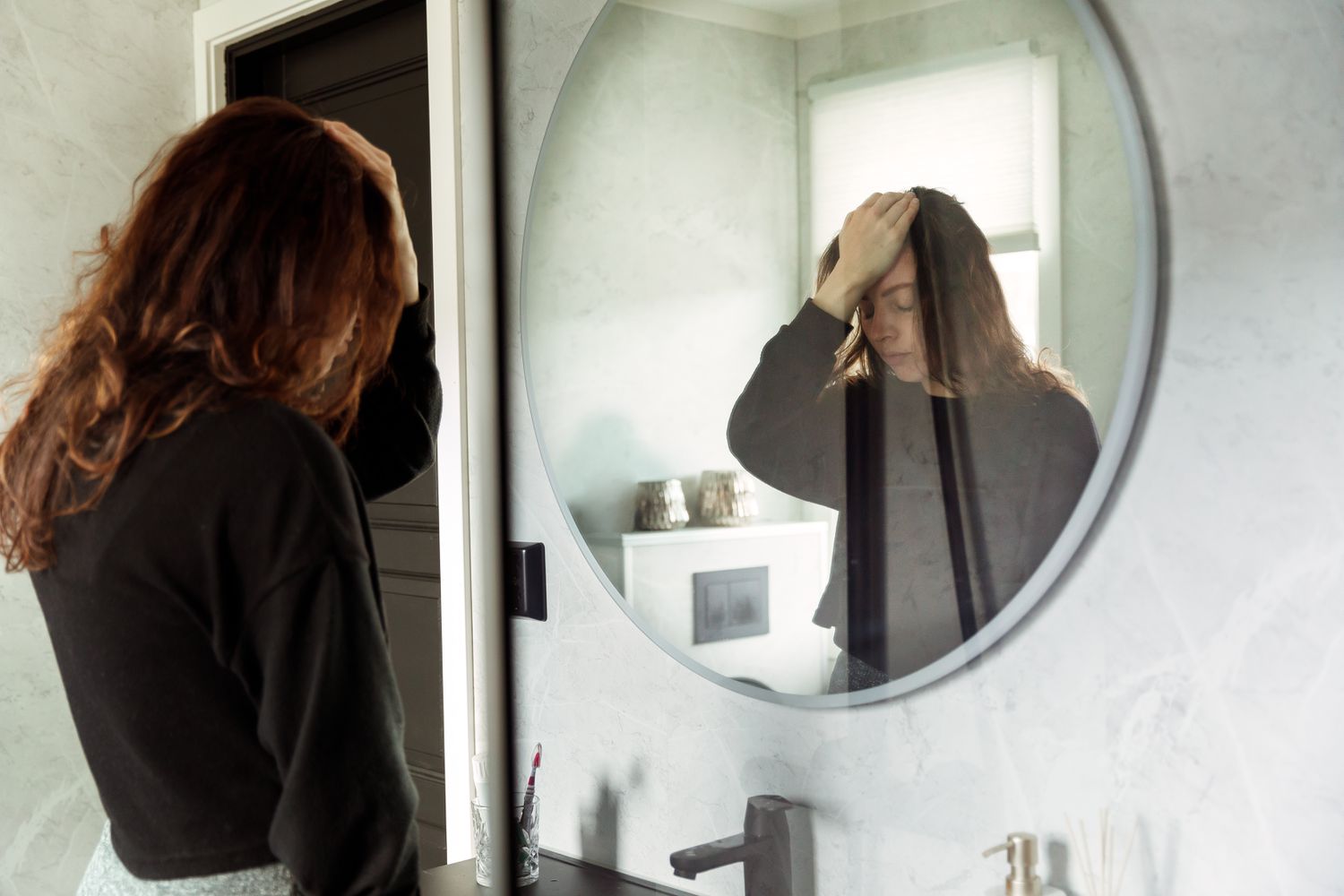}

        \small
        \textbf{Target message:} Save on water usage
    \end{minipage}
    \hfill
    \begin{minipage}[t]{0.47\linewidth}
        \centering
        \includegraphics[width=\linewidth]{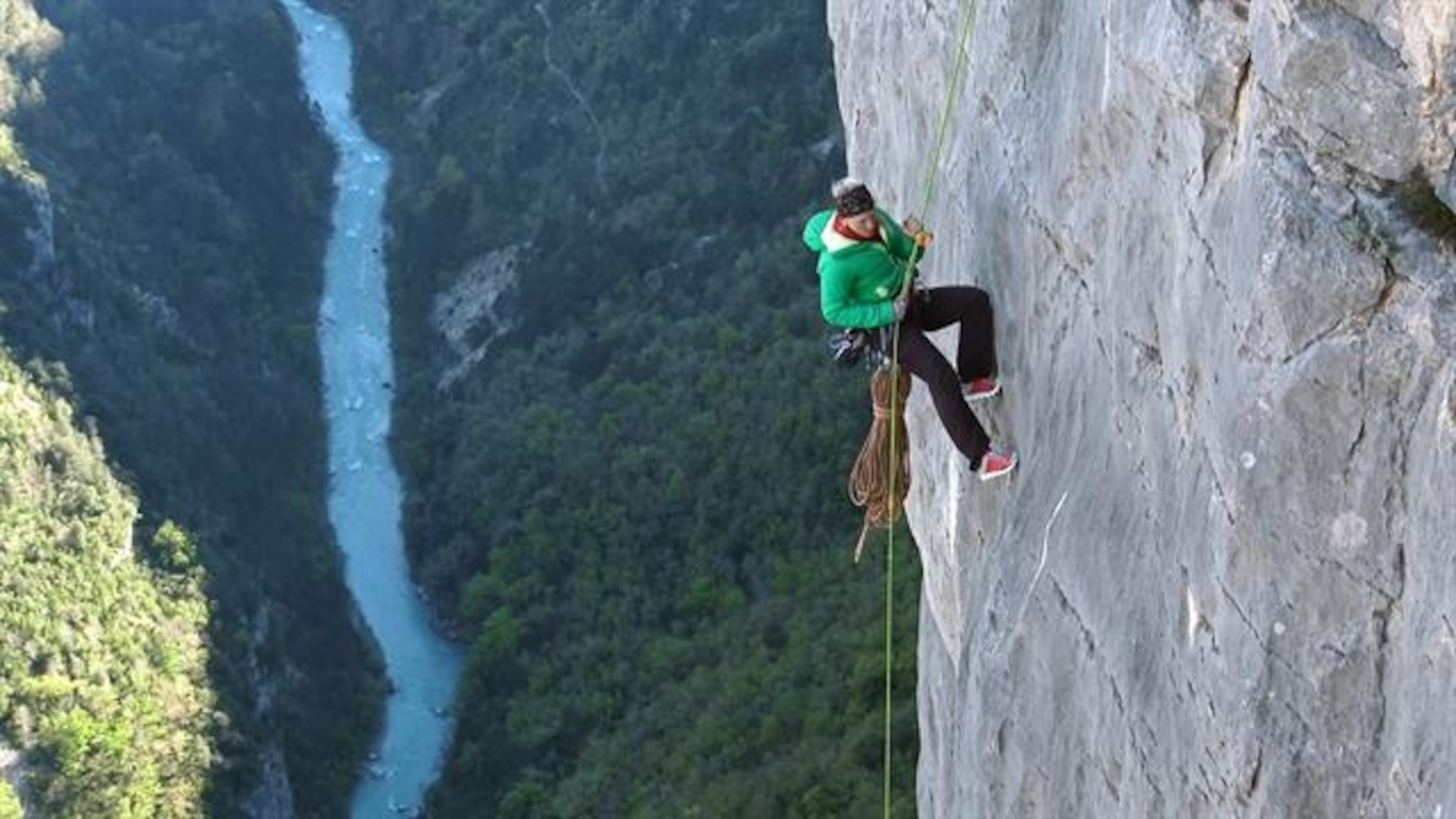}

        \small
        \textbf{Target message:} Climb stairs once a day
    \end{minipage}

    \caption{Representative examples from the web-sourced subset of PVP illustrating limited visual relevance between images and their paired target messages.}
    \label{fig:web_examples}
\end{figure}

These examples illustrate potential challenges in evaluating visual persuasion reasoning on the current web-sourced subset. They do not establish how frequently such alignment issues occur. A systematic examination of image--message alignment, together with evaluation on diverse image domains, remains an important direction for future work.

\section{Generalization to Real-World Images}
\label{app:realworld_generalization}

To examine whether our classification findings extend beyond AI-generated images, we conduct additional experiments using Google-sourced (real-world) images from the PVP dataset. We consider two settings: (1) using only Google-sourced images and (2) using a mixture of Google-sourced and DALL-E images at a 1:1 ratio.

Both settings follow the original dataset construction protocol, preserving the unseen-message train/test split, dataset sizes, and filtering procedures. Teacher rationales are generated using the same prompts as in the original experiments. We evaluate Qwen and Phi, representing two different model families, to examine whether the observed trends hold across different backbones and image sources.

\paragraph{Google-only setting.}
Table~\ref{tab:google_only} reports classification performance using only Google-sourced images.

\begin{table}[t]
\centering
\small
\setlength{\tabcolsep}{4pt}
\caption{Persuasiveness classification performance on Google-sourced images. BA: balanced accuracy; P: precision; R: recall. Bold indicates the highest BA for each backbone.}
\label{tab:google_only}
\begin{tabular}{lcccccccc}
\toprule
& \multicolumn{4}{c}{Qwen} & \multicolumn{4}{c}{Phi} \\
\cmidrule(lr){2-5}\cmidrule(lr){6-9}
Setup & BA & F1 & P & R & BA & F1 & P & R \\
\midrule
Base
& 0.740 & 0.735 & 0.752 & 0.740
& 0.762 & 0.773 & 0.755 & 0.792 \\
SFT
& 0.795 & 0.791 & 0.812 & 0.795
& 0.782 & 0.788 & 0.784 & 0.792 \\
Base-Reasoning
& 0.746 & 0.759 & 0.738 & 0.782
& 0.742 & 0.759 & 0.738 & 0.782 \\
GRPO-Ans
& 0.765 & 0.761 & 0.772 & 0.765
& 0.700 & 0.655 & 0.781 & 0.564 \\
GRPO
& 0.767 & 0.767 & 0.769 & 0.767
& 0.611 & 0.694 & 0.585 & 0.851 \\
Reasoning-SFT
& \textbf{0.814} & 0.813 & 0.815 & 0.814
& \textbf{0.789} & 0.777 & 0.839 & 0.723 \\
\bottomrule
\end{tabular}
\end{table}

Base-Reasoning performs comparably to Base on Qwen (0.746 vs. 0.740) and worse on Phi (0.742 vs. 0.762), indicating that generating rationales without task-specific supervision does not reliably improve classification performance. In contrast, Reasoning-SFT achieves the highest balanced accuracy on both backbones, outperforming SFT by 1.9 and 0.7 percentage points on Qwen and Phi, respectively.

\paragraph{Mixed-source setting.}
Table~\ref{tab:google_mixed} reports classification performance using a 1:1 mixture of Google-sourced and DALL-E images.

\begin{table}[t]
\centering
\small
\setlength{\tabcolsep}{4pt}
\caption{Persuasiveness classification performance on mixed Google-sourced and DALL-E images. Bold indicates the highest BA for each backbone.}
\label{tab:google_mixed}
\begin{tabular}{lcccccccc}
\toprule
& \multicolumn{4}{c}{Qwen} & \multicolumn{4}{c}{Phi} \\
\cmidrule(lr){2-5}\cmidrule(lr){6-9}
Setup & BA & F1 & P & R & BA & F1 & P & R \\
\midrule
Base
& 0.682 & 0.681 & 0.683 & 0.682
& 0.692 & 0.711 & 0.670 & 0.758 \\
SFT
& 0.742 & 0.742 & 0.742 & 0.742
& \textbf{0.727} & 0.732 & 0.718 & 0.747 \\
Base-Reasoning
& 0.672 & 0.670 & 0.675 & 0.672
& 0.657 & 0.694 & 0.626 & 0.778 \\
Reasoning-SFT
& \textbf{0.753} & 0.753 & 0.753 & 0.753
& \textbf{0.727} & 0.722 & 0.737 & 0.707 \\
\bottomrule
\end{tabular}
\end{table}

The mixed-source setting exhibits similar trends. Base-Reasoning underperforms Base on both backbones, while Reasoning-SFT achieves the highest balanced accuracy on Qwen (0.753) and matches SFT on Phi (0.727).

\paragraph{Discussion.}
Across the two image-source settings, Reasoning-SFT achieves the highest balanced accuracy in three of the four backbone--setting combinations and matches SFT in the remaining case. Meanwhile, Base-Reasoning provides no consistent improvement over Base.

Although these experiments cover Qwen and Phi rather than all four backbones, the consistent trends across two model families and two image-source settings provide additional evidence that our classification findings are not restricted to AI-generated images. Further validation on the remaining backbones would strengthen these observations.

\end{document}